\documentclass[journal]{IEEEtran}

\usepackage{cite}
\usepackage{amssymb,amsfonts,bm}

\usepackage{graphicx}
\usepackage{textcomp}
\usepackage{amsmath,tikz}
\usepackage{color}
\usepackage{url}
\usepackage{mathtools}
\usepackage[hidelinks]{hyperref} 
\usepackage{rotating}
\usepackage{blkarray}
\usepackage{cleveref}
\usepackage{physics}
\usepackage{tabularx}
\usetikzlibrary{arrows}
\let\proof\relax
\let\endproof\relax
\usepackage{amsthm}
\newtheorem{thm}{Theorem}[section]

\newtheorem{cor}[thm]{Corollary}

\newtheorem{rem}[thm]{Remark}

\newtheorem{ass}[thm]{Assumption}

\usepackage{caption}
\usepackage{subcaption}
\usepackage{etoolbox}
\let\classAND\AND
\let\AND\relax
\usepackage{algorithm}
\usepackage{algorithmicx,algpseudocode}
\let\AND\classAND
\AtBeginEnvironment{algorithmic}{\let\AND\algoAND}

\graphicspath{{eps/}{png/}}
\DeclareSymbolFont{symbolsC}{U}{pxsyc}{m}{n}
\DeclareMathSymbol{\coloneqq}{\mathrel}{symbolsC}{"42}

\newcommand{\vertiii}[1]{{\left\vert\kern-0.25ex\left\vert\kern-0.25ex\left\vert
		#1 
		\right\vert\kern-0.25ex\right\vert\kern-0.25ex\right\vert}}
\usepackage{multirow}
\usepackage{makecell}
\usepackage{adjustbox}
\usepackage{float}

\DeclareMathOperator{\Rank}{Rank}
\DeclareMathOperator{\Diag}{\mathsf{D}}

\DeclareMathOperator{\proj}{proj}
\newcommand{\vproj}[2][]{\proj_{#1}{#2}}
\DeclareMathOperator{\vex}{vec}
\DeclareMathOperator{\dvex}{dvec}
\DeclareMathOperator{\pvex}{\partial vec}
\DeclareMathOperator{\col}{col}

\newcommand*{\vecbf}[1]{\mathbf{#1}} % vector boldface
\newcommand*{\matbf}[1]{\mathbf{#1}} % matrix boldface
\newcommand*{\gbf}[1]{\bm{#1}} % greek boldface
\newcommand*{\myset}[1]{\mathcal{#1}} % greek boldface
\usepackage{pifont}% http://ctan.org/pkg/pifont
\usepackage{soul}

\newcommand*{\stat}{\vecbf{x}} % state
\newcommand*{\ctrl}{\vecbf{u}} % control
\newcommand*{\dyn}{\vecbf{f}} % state
\newcommand*{\inc}{\vecbf{g}} % inequality constraint
\newcommand*{\eqc}{\vecbf{h}} % equality constraint
\newcommand*{\mpccost}{W} % equality constraint

\newcommand*{\lmIn}{\gbf{\lambda}} % lagrangian multiplier for inequality constraint
\newcommand*{\lmEq}{\gbf{\gamma}} % lagrangian multiplier for equality constraint
\newcommand*{\dlmIn}{\delta\gbf{\lambda}} % delta lagrangian multiplier for inequality constraint
\newcommand*{\dlmEq}{\delta\gbf{\gamma}} % delta lagrangian multiplier for equality constraint
\newcommand*{\Resi}{\vecbf{r}} % residual for inequality constraints
\newcommand*{\resi}{r} % scalar residual
\newcommand*{\pert}{\mu} % perturbation
\newcommand*{\rescl}{\gbf{\epsilon}} % residual for close-loop loss

\newcommand*{\size}[1]{ m_{#1} }  % size or dimension
\newcommand*{\noise}{ \vecbf{n} } % noise
\newcommand*{\pathcost}{\ell}  % stage cost
\newcommand*{\termcost}{\wp}  % terminal cost
\newcommand*{\dstat}{\delta\stat}  % delta state
\newcommand*{\dctrl}{\delta\ctrl}  % delta state
\newcommand*{\ctrlset}{\myset{U}} % control set
\newcommand*{\trajset}{\myset{Z}} % traj set
\newcommand*{\sampleset}{\myset{S}} % sampling set

\newcommand*{\qfunc}{Q} % Q function

\newcommand*{\dqfunc}{\delta\qfunc}

\newcommand*{\vfunc}{V} % V function
\newcommand*{\dvfunc}{\delta\vfunc}
\newcommand*{\gaink}{\vecbf{k}}
\newcommand*{\gainK}{\vecbf{K}}
\newcommand*{\para}{\gbf{\theta}}
\newcommand*{\paraex}{\theta} % para example
\newcommand*{\dpara}{\delta\para}
\newcommand*{\parastat}{\vecbf{y}}
\newcommand*{\dparastat}{\delta\parastat}
\newcommand*{\acset}{^{\diamond}}

\newcommand*{\cmMat}{\matbf{C}}
\newcommand*{\examfunc}{c} % function in example

\newcommand{\hhat}[1]{\mathring{#1}} % hessian hat for notation
\newcommand{\dvexd}[2]{\hhat{\nabla}_{#2}#1}    %{\frac{\dvex(#1)}{\dd{#2}}} % dvex / d
\newcommand{\pvexd}[2]{\hhat{\partial}_{#2}#1}    %{\frac{\pvex(#1)}{\dd{#2}}} % pvex / d

\newcommand*{\recMat}{\matbf{J}}

\begin{document}

	\title{A Differential Dynamic Programming Framework for Inverse Reinforcement Learning}

	\author{Kun~Cao, Xinhang~Xu, Wanxin~Jin, Karl~H.~Johansson, Lihua~Xie
	\thanks{K. Cao, X.~Xu, and L.~Xie (corresponding author) are with School of Electrical and Electronic Engineering, Nanyang Technological University, 50 
			Nanyang Avenue, Singapore 639798 {\tt \footnotesize kun001@e.ntu.edu.sg; xu0021ng@e.ntu.edu.sg; elhxie@ntu.edu.sg}. W~Jin is with School for Engineering of Matter, Transport, and Energy, Arizona State University. {\tt \footnotesize wanxinjin@gmail.com}.  K. H. Johansson is with Division of
Decision and Control Systems, School of Electrical Engineering and Computer
Science, KTH Royal Institute of Technology, and also with Digital Futures,
SE-10044 Stockholm, Sweden. {\tt \footnotesize kallej@kth.se}. \\
This research was supported by Wallenberg-NTU Presidential Postdoctoral Fellowship and Ministry of Education, Singapore, under AcRF TIER 1 Grant RG64/23.
  }
			
   } 	
   
\markboth{IEEE, 2023}
{CAO \MakeLowercase{\textit{et al.}}: DDP} 	
\maketitle       

\begin{abstract}
A differential dynamic programming (DDP)-based framework for inverse reinforcement learning (IRL) is introduced to recover the parameters in the cost function, system dynamics, and constraints from demonstrations. 
Different from existing work, where DDP was used for the inner forward problem with inequality constraints, our proposed framework uses it for efficient computation of the gradient required in the outer inverse problem with equality and inequality constraints. 
The equivalence between the proposed method and existing methods based on Pontryagin's Maximum Principle (PMP) is established.  
More importantly, using this DDP-based IRL with an open-loop loss function, a closed-loop IRL framework is presented.
In this framework, a loss function is proposed to capture the closed-loop nature of demonstrations. 
It is shown to be better than the commonly used open-loop loss function.
We show that the closed-loop IRL framework reduces to a constrained inverse optimal control problem under certain assumptions.
Under these assumptions and a rank condition, it is proven that the learning parameters can be recovered from the demonstration data. 
The proposed framework is extensively evaluated through four numerical robot examples and one real-world quadrotor system.
The experiments validate the theoretical results and illustrate the practical relevance of the approach.
\end{abstract}


\IEEEpeerreviewmaketitle

\section{Introduction}
\label{sec:intro}

Recent years have witnessed a significant surge in advancements
within the field of Reinforcement Learning
(RL), which iteratively learns an optimal policy that maximizes
a human-designed accumulative reward by repeatedly interacting
with the environment, has demonstrated a remarkable
capability in dealing with challenging tasks such as game
playing \cite{silver2016mastering}, motion planning \cite{kober2013reinforcement}, portfolio optimization \cite{hambly2023recent}, and energy system operation \cite{zhang2023review}. 
Despite these achievements, one of the principal challenges in RL remains the design of an appropriate cost function that reliably induces desired behaviors, especially for high-dimensional and complex tasks \cite{amodei2016concrete}. 
Typically, the design process of a cost function involves an iterative process of trial and error, requiring substantial manual effort and even strong prior knowledge and expertise.

To address this, the inverse RL (IRL) problem has been proposed to automate the critical task of designing cost functions by learning from the observed behaviors of (possibly non-) experts.
Over the past decades, many formulations of IRL have been proposed, with different approaches emphasizing different learning criteria. Representative works include apprenticeship learning \cite{abbeel2004apprenticeship}, which matches the feature vectors of demonstration and predicted trajectory, MaxEnt \cite{ziebart2008maximum}, which maximizes the entropy of the trajectory distribution subject to a reward expectation constraint, and Max-margin \cite{ratliff2006maximum}, which maximizes the margin between the objectives of demonstration and predicted trajectory.

Despite different cost update criteria, existing approaches share a common bi-level algorithmic design: the cost function is updated in the outer loop and the corresponding RL is optimized in the inner loop. For the inner level, optimizing an RL agent is primarily driven by agent sampling via interacting with the environment. This sampling-based optimization process may take a large number of training epochs to converge, which ultimately leads to the inefficiency of the entire IRL framework. To alleviate this, the authors in \cite{jin2020pontryagin} proposed the Pontryagin Differential Programming (PDP) framework, where the inner level uses a parameterized optimal control (OC) problem and can be efficiently solved by a model-based solver. Furthermore, the proposed analytical gradient by differentiating the equilibrium condition (i.e., PMP) of the inner problem makes the end-to-end update of the cost function possible. A similar framework has also been proposed in \cite{jin2021safe, amos2018differentiable} to consider the IRL where there are stage-wise state and/or control constraints in the inner RL agent. 

While the above IRL frameworks building upon a differentiable inner loop achieve computational efficiency, a limitation, which we have empirically observed but have been largely overlooked in \cite{jin2020pontryagin,jin2021safe, amos2018differentiable}, is their imitation-based loss function in the outer level. Specifically, \cite{jin2020pontryagin,jin2021safe,jin2022learning,amos2018differentiable} proposes minimizing a mean square outer-level loss, which is a discrepancy between the reproduced trajectory and the demonstrations; thus, the IRL formulation can be viewed as a nonlinear least square problem. The use of imitation loss implies that the expert demonstrated trajectory is a result of open loop control, and that the trajectory data has been polluted by temporally independent noise. However, this assumption may be not valid (we have later shown this analytically and numerically) for observation data generated by the expert with closed-loop policies. In fact, data collection from a closed-loop policy agent is often the case, for better stability and robustness. Due to the nature of the closed-loop policy, the noise along an observed trajectory is not temporally independent. Thus, the choice of imitation loss would lead to bias of the cost function.

In this paper, we rethink the problem of IRL via the differentiation of the inner layer. But different from existing work, we consider both the inner-level optimal control agent and the outer-level learning loss from a closed-loop perspective. Specifically, 
\begin{itemize}
    \item We formulate a closed-loop optimal control problem in the inner level via the process of DDP, upon which we propose a new way of carrying out the differentiation via the corresponding Bellman optimality equation.
    \item We propose a new loss function that directly captures the feedback nature of the expert data generation, which leads to unbiased learning of cost function, compared to the open-loop definition of loss function.
\end{itemize}

\subsection{Related Work}

Bi-level optimization was first realized in the field of game theory in the seminal work \cite{von2010market} to solve a hierarchical decision-making problem, where the inner-level problem can be defined by different programs \cite{bracken1973mathematical}, such as linear programs, nonlinear programs, games, and multi-stage programs. Generally, there are two classes of approaches for solving this problem. The first one is to reduce it to a single-level problem by replacing the inner-level problem with its optimality conditions as constraints. However, this approach may lead to constrained problems with large problem sizes or complementarity constraints, which are combinatorial in nature and cannot be handled efficiently \cite{hansen1992new}. The second approach maintains the bi-level structure, where the inner-level problem can be solved by existing solvers and the gradient required by the outer level is obtained by differentiating the inner-level equilibrium conditions \cite{jin2020pontryagin, amos2018differentiable}. In the spirit of the second approach, this work focuses on developing a new way of efficient differentiation for general constrained optimal control problems.

The dynamics of inner-level multi-stage programs can be either modeled by a Markov Decision Process (MDP) or a state-space equation. In this paper, we only focus on the deterministic optimal control problem, where the dynamics are modeled by the state-space equation. We categorize existing deterministic optimal control techniques into open-loop methods, which directly solve a trajectory as a function of time, and closed-loop methods, which seek a mapping from current observation to an optimal control action. The first category is based on the PMP \cite{pontryagin2018mathematical}, which is derived from the calculus of variations. Popular methods include shooting methods \cite{bock1984multiple}
and collocation methods \cite{patyerson2014matlab}. However, these methods optimize based on the initial conditions and hence are susceptible to model errors or disturbances during deployment. 

Another category of methods is based on dynamic programming and specifically the Bellman optimality equation \cite{bellman1954theory}, which characterizes the mathematical condition that a control input in each step should satisfy w.r.t. the current state, hence it leads to a closed-loop policy. 
Differential dynamical programming \cite{mayne1966second} is a numerical algorithm that aims to find the solution to this equation by iteratively linearizing and quadraticizing the cost function and dynamic equation. 
It enjoys the linear computational complexity (w.r.t. horizon) and local quadratic convergence \cite{de1988differential}. Subsequently, this algorithm has also been generalized to the case with inequality constraints via three major methods: 
1) converting the constrained problems to unconstrained ones via penalty methods \cite{plancher2017constrained}; 
2) identifying the active inequality constraints and then solving the equality-constrained OC problem \cite{xie2017differential}; 
3) introducing a constrained version of Bellman’s principle of optimality \cite{aoyama2020constrained,pavlov2021interior}, which augments the control input with dual variables and hence avoids the combinatorial problem regarding the active constraints. 
However, these works are limited to the case with only inequality constraints, and more importantly, all of them are used in solving an optimal trajectory, which is the inner loop of the IRL problem and has not been exploited for the update in the outer loop.
In the spirit of the third method, this work will propose a DDP-based algorithm to solve general constrained OC problems and develop a new way of differentiation over DDP to tackle the IRL problem.

The inverse optimal control (IOC) problem, which is highly related to IRL while assuming that the system dynamics is known or being identified beforehand by system identification techniques, has been considered in control community.
A popular and efficient approach to solving IOC is residual minimization, which finds a set of parameters such that the violation of optimality conditions (e.g. Karush-Kuhn-Tucker conditions \cite{keshavarz2011imputing,englert2017inverse} and PMP equations \cite{johnson2013inverse,molloy2018finite, jin2021inverse}) is minimized when evaluated along with collected demonstrations. 
By exploiting the special structure of the cost function, it can be shown that the optimality conditions are linear in the parameter and the latter can be decoupled from the collected demonstrations. 
Therefore, some rank equality conditions only on demonstrations can be derived as a sufficient condition for recovering the parameter. 
Moreover, owing to the linearity, these methods only need to solve a quadratic programming problem, which avoids solving optimal control problems in an inner loop as in the bi-level optimization, and hence are generally more efficient. 
However, these methods did not take into consideration stage-wise constraints, which often appear in real applications. 
The authors in \cite{molloy2020online} extended their work \cite{molloy2018finite} to the case with only control constraints, where an additional index set was introduced to remove these constraints and convert the problem back into an unconstrained problem. 
However, the presented method is limited to the control constraints and is difficult to be extended to the case with more general constraints. This paper will establish the recoverability condition for the general constrained IRL problem and include the above-mentioned condition as a special case.

\subsection{Contributions}

In this work, we propose a new DDP-based IRL framework, where it is shown that the terms required to update the outer loop can be computed by using DDP algorithms. 
In particular, by observing that the intermediate matrices that appear in DDP recursions are exactly the terms which we require for obtaining the analytical gradient, we introduce an augmented system with the learning parameter being an additional state and show that the gradient can be generated by performing a one-step DDP recursion on that augmented system. 
Moreover, in order to incorporate the closed-loop nature of data collection, we propose a new type of loss function based on the above-mentioned intermediate matrices, where the main idea is that one should try to match the reproduced and demonstrated feedback policies instead of matching the reproduced and demonstrated trajectories.
Furthermore, thanks to the general form of this new loss function, it naturally leads to a generalized set of recoverability conditions for the constrained IOC problem under some assumptions. 

The contributions of this paper lie in five-folds:
\begin{itemize}\setlength{\parindent}{0pt}
\item We propose a unified DDP-based IRL framework to learn the parameters in the cost function, system dynamics, and general constraints;
\item We show that the required gradient term for updating the learning parameter can be obtained efficiently via performing one-step DDP recursion on an augmented system and establish the equivalence between DDP-based methods and PDP-based methods;
\item We propose a new type of loss function which by definition outperforms the traditionally adopted imitation loss on the closed-loop demonstrations and develop an efficient algorithm alongside;
\item We establish the recoverability conditions for the general constrained IRL problem, whose specialization under some assumptions is also a generalization of the traditional unconstrained IOC recoverability condition;
\item We apply the proposed theoretical results to simulation examples and real-world experiments.
\end{itemize}

The rest of this paper is structured as follows. 
Section \ref{sec:prob} formally formulates the problem to be studied. 
Section \ref{sec:ol_irl} presents our proposed DDP-based IRL framework with a commonly used open-loop loss.
Section \ref{sec:cl_irl} details the DDP-based IRL framework with the proposed closed-loop loss.
Numerical simulations and real-world experiments are provided in Section \ref{sec:simu} and Section \ref{sec:exp}. 
Finally, Section \ref{sec:conclu} concludes this paper.

\emph{Notations:} In this paper, $\|\vecbf{x}\|$ denotes the $2$-norm of $\vecbf{x} \in \mathbb{R}^{n}$ and $\|\vecbf{x}\|_{\matbf{A}}^{2} = \vecbf{x}^{\top}\matbf{A}\vecbf{x}$. 
Denote by $\matbf{A}^{\top}$ and $\matbf{A}^{-1}$ the transpose and inverse of $\matbf{A} \in \mathbb{R} ^{n \times n}$, respectively. 
Let $\matbf{I}_{n} \in \mathbb{R}^{n \times n}$ be the $n$-dimensional identity matrix and $\vecbf{1}_{n}$ be the $n$-dimensional column vector with all entries of $1$. 
Denote the vectorization operation by $\vex(\cdot)$, i.e., $\vex([\vecbf{a}, \vecbf{b}]) = [\vecbf{a}^{\top}, \vecbf{b}^{\top}]^{\top}$. 
Let $\col(\{\matbf{A}, \matbf{B}\}) = [\matbf{A}^{\top}, \matbf{B}^{\top}]^{\top}$.
Let $\Diag(\cdot)$ denote the transformation from a vector to a diagonal matrix or the extraction of the diagonal elements from a square matrix to a vector.
Let $\otimes$, $\odot$, and $\oplus$ denote the Kronecker product, the tensor contraction, and the quaternion product operation, respectively. Let $\myset{I}_{n} = \{0,\dots,n-1\}$. 
Let $(\cdot)_{\vecbf{a}} := \pdv{(\cdot)}{\vecbf{a}}$ and $(\cdot)_{\vecbf{ab}} := \pdv{(\cdot)}{\vecbf{b}}{\vecbf{a}}$, and define $\frac{\dvex(\matbf{A})}{\dd{x}}$ and $\frac{\pvex(\matbf{A})}{\partial{\vecbf{x}}}$ by $\dvexd{\matbf{A}}{x}$ and $\pvexd{\matbf{A}}{x}$, respectively. 
Let $[\matbf{A}]_{i}$ denote the $i$-th slice of tensor $\matbf{A}$ and $[\cdot]_{\times}$ denote the cross product operation. 
Let $\matbf{C}^{n, m}$ denotes the commutation matrix which satisfies $\matbf{C}^{n, m} \vex(\matbf{A}) = \vex(\matbf{A}^{\top})$, where $\matbf{A} \in \mathbb{R}^{n \times m}$.

\section{Problem Formulation}
\label{sec:prob}
Consider the following general nonlinear constrained optimal control problem
\begin{equation}
\label{eq:prob_ineq_eq}
\begin{aligned}
\min_{\ctrlset} &~~ \mpccost(\trajset; \para) := \sum_{k \in \myset{I}_{N}} \pathcost(\stat_{k}, \ctrl_{k}; \para) + \termcost(\stat_{N};\para) \\
\mathrm{s.t.} &~~ \stat_{k+1} = \dyn(\stat_{k}, \ctrl_{k};\para), \stat_{0}~\mathrm{is~given}, \\
&~~ \inc(\stat_{k}, \ctrl_{k};\para) \leq \vecbf{0}, \\
&~~ \eqc(\stat_{k}, \ctrl_{k};\para) = \vecbf{0},
\end{aligned}
\end{equation}
where $\stat_{k} \in \mathbb{R}^{\size{\stat}}$ and $\ctrl_{k} \in \mathbb{R}^{\size{\ctrl}}$ denote the state and control input at time instant $k$, respectively;
$\ctrlset := \{\ctrl_{k}\}_{k \in \myset{I}_{N}}$ is the collection of control inputs and $N$ is the control horizon; $\trajset := \{\stat_{k}\}_{k \in \myset{I}_{N+1}} \cup \ctrlset$ denotes the entire system trajectory; $\para \in \mathbb{R}^{\size{\para}}$ denotes the variable parameterizing the following functions: 
\begin{itemize}
    \item stage cost $\pathcost: \mathbb{R}^{\size{\stat}} \times \mathbb{R}^{\size{\ctrl}} \times \mathbb{R}^{\size{\para}} \to \mathbb{R}$;
    \item terminal cost $\termcost: \mathbb{R}^{\size{\stat}} \times \mathbb{R}^{\size{\para}} \to \mathbb{R}$;
    \item system dynamics $\dyn: \mathbb{R}^{\size{\stat}} \times \mathbb{R}^{\size{\ctrl}} \times \mathbb{R}^{\size{\para}} \to \mathbb{R}^{\size{\stat}}$;
    \item inequality constraint $\inc: \mathbb{R}^{\size{\stat}} \times \mathbb{R}^{\size{\ctrl}} \times \mathbb{R}^{\size{\para}} \to \mathbb{R}^{\size{\mathrm{in}}}$;
    \item equality constraint $\eqc: \mathbb{R}^{\size{\stat}} \times \mathbb{R}^{\size{\ctrl}} \times \mathbb{R}^{\size{\para}} \to \mathbb{R}^{\size{\mathrm{eq}}}$.
\end{itemize}
We assume that the above functions are twice-differentiable. 
Note that for the sake of clarity, $\pathcost, \dyn, \inc$ and $\eqc$ presented here do not explicitly depend on the time instant $k$, however, our analysis in the sequel can be easily extended to the case where $\pathcost, \dyn, \inc$ and $\eqc$ are time-dependent.

Denote a sampled trajectory of the entire system trajectory $\trajset$ as 
$\trajset_{\sampleset} := \{\stat_{k}\}_{k \in \sampleset} \cup \{\ctrl_{k}\}_{k \in \sampleset} $, where $\sampleset \subseteq \myset{I}_{N+1}$ denotes the set of sampling time instants. 
Given a specific value of $\para$, one can use a nonlinear programming solver to obtain a system trajectory $\trajset(\para)$. 
We assume that the mapping from $\para$ to $\trajset(\para)$ always exists and is unique for the local set $\Theta$, where the required regularity conditions can be found in \cite[Lemma 1]{jin2021safe}. 

The problem of interest is that given a set of $|\myset{D}|$ expert demonstrations $\myset{D} = \{\trajset_{\sampleset_{i}}(\para^{\ast})\}_{i = 1, \dots, |\myset{D}|}$ generated from some unknown parameter $\para^{\ast}$, find a $\para$ which matches these expert demonstrations most, i.e.,
\begin{equation}
\label{eq:bilevel_prob}
\begin{aligned}
\min_{\para \in \Theta} &~~ L(\myset{D}, \trajset, \para) \\
\mathrm{s.t.} &~~ \trajset~\mathrm{with}~\ctrlset~\mathrm{being~solved~from}~\eqref{eq:prob_ineq_eq}.
\end{aligned}
\end{equation}
In the above, $L$ denotes the loss function which characterizes the closeness between the demonstration $\trajset_{\sampleset_{i}}(\para^{\ast})$ and the solved trajectory $\trajset$. A commonly used loss function in the literature 
\cite{jin2020pontryagin, jin2021safe} is the mean-square-error loss\footnote{We omit the demonstration index $i$ in the sequel and assume a single demonstration for the sake of simplicity, while the subsequent analysis can be easily extended to the multiple demonstrations case.}
\begin{equation}
\label{eq:loss_ol}
L^{\mathrm{ol}} := \|\trajset_{\sampleset}(\para^{\ast}) - \trajset\|_{2}^{2},
\end{equation}
where an additional regularization term $\|\para\|_{2}^{2}$ for $\para$ can be added when required. In the sequel, we denote this loss as the open-loop loss, as it views the state and control input in the demonstration independently, which usually is not the case in the demonstration generation process. Nevertheless, in the subsequent section, we develop efficient algorithms for optimizing this loss. In Section \ref{sec:cl_irl}, to explicitly take into consideration the feedback nature of demonstrations, we propose a new so-called closed-loop loss, which will be demonstrated to be superior to the open-loop loss.   

\section{Open-loop IRL}
\label{sec:ol_irl}
In this section, we shall develop a new IRL algorithm to solve the optimization problem \eqref{eq:bilevel_prob} with the open-loop loss 
$L^{\mathrm{ol}}$ by exploiting the vanilla DDP algorithm and its variants.  
It can be found that problem $\eqref{eq:bilevel_prob}$ is of the form of the bi-level optimization, where the low-level inner optimization solves the constrained multi-stage optimal control problem \eqref{eq:prob_ineq_eq} and the higher-level outer optimization optimizes the loss function $L^{\mathrm{ol}}$. Hence, the commonly used gradient descent method can be adopted to solve this bi-level optimization problem, i.e., 
\begin{equation}
\label{eq:para_upd}
    \para_{t+1} = \vproj[\Theta]{[\para_{t} - \eta_{t} (\pdv{L^{\mathrm{ol}}}{\trajset} \dv{\trajset}{\para} + \pdv{L^{\mathrm{ol}}}{\para})])},
\end{equation}
where $\para_{t}$ is the current estimate of learning parameter $\para$ in iteration $t$ with its initial value being $\para_{0}$; $\eta_{t}$ is the learning rate; $\pdv{L^{\mathrm{ol}}}{\trajset}$ and $\pdv{L^{\mathrm{ol}}}{\para}$ denote the partial derivatives of loss function w.r.t. the solved trajectory and the learning parameter; $\dv{\trajset}{\para}$ denotes the derivative of the solved trajectory w.r.t. the learning parameter. 
At iteration $t$, one can solve the trajectory $\trajset(\para_{t})$ from \eqref{eq:prob_ineq_eq} with a specific value $\para_{t}$. This can be done by either an external solver or the method developed in Section \ref{subsec:forward}, and we call this the trajectory solver in the sequel. 
Then, with a known loss function, one can evaluate $\pdv{L^{\mathrm{ol}}}{\trajset}$ and $\pdv{L^{\mathrm{ol}}}{\para}$ easily with analytic differentiation or auto-differentiation in existing machine learning frameworks (e.g. Pytorch \cite{paszke2017automatic}). However, for $\dv{\trajset}{\para}$, $\trajset$ is the optimal solution of an optimization program, which is obtained from at least tens of iterations, instead of an explicit functional mapping from inputs (initial state and parameter). Auto-differentiation on this optimization procedure unrolls computational graphs in each iteration and hence results in prohibitive memory and computational complexity.
On the other hand, notice that in order to be the optimal solution, $\trajset$ should satisfy some equilibrium conditions, which implicitly characterize the relationship between inputs and the optimal solution. In Section \ref{subsec:backward}, these conditions are resorted to design an efficient gradient solver for obtaining $\dv{\trajset}{\para}$.

In what follows, we shall first introduce a new DDP-based trajectory solver. 
Although the trajectory can also be solved by any other external solvers, we present this algorithm as a generalization of the previous IPDDP algorithm \cite{pavlov2021interior} to the case with equality constraints and also for paving the way to develop the gradient solver in Sec. \ref{subsec:backward}.

\subsection{DDP-based trajectory solver}
\label{subsec:forward}
In this subsection, we shall present a new DDP-based algorithm to solve the optimal control problem with both inequality and equality constraints. 
The proposed algorithm inherits the general structure of the traditional DDP algorithm, i.e., solving a Bellman equation via iterative backward and forward recursions. 
The backward recursions compute control inputs to minimize a quadratic approximation of the cost-to-go in the vicinity of the current solution, and the forward recursions update the current solution to a new one.
However, in order to deal with equality constraints which have not been considered in \cite{pavlov2021interior}, both the Bellman equation and the iterative process should be redesigned, which are detailed as follows.

Let us first denote the cost-to-go and optimal cost-to-go at time instant $k$ as $\qfunc_{k}$ and $\vfunc_{k}$, respectively.
In the following, for the sake of clarity, we shall omit the subscript $(\cdot)_{k}$ if no confusion is caused and let $(\cdot)^{+}$ denote $(\cdot)_{k+1}$.
Applying the Bellman principle of optimality, one has:
\begin{equation*}
\begin{aligned}
    \vfunc & = \min_{\ctrl}~\pathcost + \vfunc^{+}(\stat^{+}) \quad \mathrm{s.t.}~\inc \leq \vecbf{0}, \eqc = \vecbf{0},
\end{aligned}
\end{equation*}
which is a general nonlinear programming problem. One possible solution is using nested optimization, i.e., in the inner loop, each $\ctrl_{k}$ is solved by calling a general nonlinear programming solver, and in the outer loop, the trajectory $\trajset$ is updated.
However, it can be found that in this solution, each outer loop calls a solver $N$ times, which results in high computational complexity. 
Observe that the above process works in a similar manner to the barrier method \cite[Chapter 19.6]{nocedal1999numerical}, where a full optimization problem is solved in the inner loop (i.e., ``centering"), which inspires us to use the primal-dual interior-point method as an alternative way to solve \eqref{eq:prob_ineq_eq}. 
To this end, we introduce the two dual variables $\lmIn, \lmEq$ for the inequality and equality constraint, respectively. 
As a result, one has the following interior-point min-max Bellman equation:
\begin{equation}
\label{eq:minmax_bellman_ineq_eq}
    \vfunc = \min_{\ctrl} \max_{\lmIn \geq \vecbf{0}, \lmEq} \qfunc := \pathcost + \vfunc^{+} + \lmIn^{\top}\inc + \lmEq^{\top}\eqc .
\end{equation}
In contrast to the cost-to-go function that appears in the traditional DDP algorithm and is only a function of $\stat$ and $\ctrl$, $\qfunc$ in the above is also a function of newly introduced dual variables $\lmIn, \lmEq$. 

After introducing the new Bellman equation, we shall aim to solve it in an iterative manner by resorting to its local approximation.
Taking the second order variation of the above $\qfunc$ and $\vfunc$, one has 
\begin{equation*}
    \dqfunc = \frac{1}{2} \begin{bmatrix}
        1 \\ \dstat \\ \dctrl \\ \dlmIn \\ \dlmEq
    \end{bmatrix}^{\top} \begin{bmatrix}
        0 & \qfunc_{\stat}^{\top} & \qfunc_{\ctrl}^{\top} & \qfunc_{\lmIn}^{\top} & \qfunc_{\lmEq}^{\top} \\
        \qfunc_{\stat} & \qfunc_{\stat\stat} & \qfunc_{\stat\ctrl} & \qfunc_{\stat\lmIn} & \qfunc_{\stat\lmEq} \\
        \qfunc_{\ctrl} & \qfunc_{\ctrl\stat} & \qfunc_{\ctrl\ctrl} & \qfunc_{\ctrl\lmIn} & \qfunc_{\ctrl\lmEq} \\  
        \qfunc_{\lmIn} & \qfunc_{\lmIn\stat} & \qfunc_{\lmIn\ctrl} & \qfunc_{\lmIn\lmIn} & \qfunc_{\lmIn\lmEq} \\ 
        \qfunc_{\lmEq} & \qfunc_{\lmEq\stat} & \qfunc_{\lmEq\ctrl} & \qfunc_{\lmEq\lmIn} & \qfunc_{\lmEq\lmEq} \\  
    \end{bmatrix}
    \begin{bmatrix}
        1 \\ \dstat \\ \dctrl \\ \dlmIn \\ \dlmEq
    \end{bmatrix}  
\end{equation*}
and 
\begin{equation}
    \dvfunc = \frac{1}{2} \begin{bmatrix}
        1 \\ \dstat 
    \end{bmatrix}^{\top} \begin{bmatrix}
        0 & \vfunc_{\stat}^{\top} \\
        \vfunc_{\stat} & \vfunc_{\stat\stat}     
    \end{bmatrix}
    \begin{bmatrix}
        1 \\ \dstat
    \end{bmatrix}.  
\end{equation}
By definition of $\qfunc$, one has the following equations
\begin{equation}
\label{eq:vfunc_to_qfunc}
    \begin{aligned}
        \qfunc_{\stat} & = \pathcost_{\stat} + \dyn_{\stat}^{\top}\vfunc_{\stat}^{+} + \inc_{\stat}^{\top} \lmIn + \eqc_{\stat}^{\top} \lmEq , \\
        \qfunc_{\ctrl} & = \pathcost_{\ctrl} + \dyn_{\ctrl}^{\top}\vfunc_{\stat}^{+} + \inc_{\ctrl}^{\top} \lmIn + \eqc_{\ctrl}^{\top} \lmEq, \\  
        \qfunc_{\lmIn} & = \inc, \qfunc_{\lmIn\stat} = \inc_{\stat}, \qfunc_{\lmIn\ctrl} = \inc_{\ctrl}, \\
        \qfunc_{\lmEq} & = \eqc, \qfunc_{\lmEq\stat} = \eqc_{\stat}, \qfunc_{\lmEq\ctrl} = \eqc_{\ctrl}, \\  
        \qfunc_{\lmIn\lmIn} & = \matbf{0}, \qfunc_{\lmIn\lmEq} = \matbf{0}, \qfunc_{\lmEq\lmEq} = \matbf{0}, \\
        \qfunc_{\stat\stat} & = \pathcost_{\stat\stat} + \dyn_{\stat}^{\top}\vfunc_{\stat\stat}^{+}\dyn_{\stat} + \vfunc_{\stat}^{+} \odot \dyn_{\stat\stat} + \lmIn \odot \inc_{\stat\stat} + \lmEq \odot \eqc_{\stat\stat} , \\
        \qfunc_{\ctrl\stat} & = \pathcost_{\ctrl\stat} + \dyn_{\ctrl}^{\top}\vfunc_{\stat\stat}^{+}\dyn_{\stat} + \vfunc_{\stat}^{+} \odot \dyn_{\ctrl\stat} + \lmIn \odot \inc_{\ctrl\stat} + \lmEq \odot \eqc_{\ctrl\stat}, \\
        \qfunc_{\ctrl\ctrl} & = \pathcost_{\ctrl\ctrl} + \dyn_{\ctrl}^{\top}\vfunc_{\stat\stat}^{+}\dyn_{\ctrl} + \vfunc_{\stat}^{+} \odot \dyn_{\ctrl\ctrl} + \lmIn \odot \inc_{\ctrl\ctrl} + \lmEq \odot \eqc_{\ctrl\ctrl}, 
    \end{aligned}
\end{equation}
which computes the partial derivatives of the cost-to-go $\qfunc_{(\cdot)}$ from the optimal cost-to-go at the next time instant. 
In order to complete the backward recursion, an update rule from the cost-to-go $\qfunc$ to the optimal cost-to-go $\vfunc$ is required.
To this end, we consider the solution to the following problem, 
\begin{equation}
\label{eq:dq_problem}
\begin{aligned}
    \min_{\dctrl} \max_{\dlmIn, \dlmEq} \dqfunc~\mathrm{s.t.}~ \lmIn + \dlmIn \geq \vecbf{0},
\end{aligned}
\end{equation}
which is first-order variation of \eqref{eq:minmax_bellman_ineq_eq}. If $(\ctrl, \lmIn, \lmEq)$ is the stationary point of \eqref{eq:minmax_bellman_ineq_eq}, $(\dctrl, \dlmIn, \dlmEq)$ should satisfy the following conditions: 
\begin{itemize}\setlength{\parindent}{0pt}
    \item  for the minimizing variable $\dctrl$, it must satisfy the stationarity condition:
\begin{equation*}
\label{eq:dqfunc_dctrl}
   \frac{\dqfunc}{\dctrl} = \qfunc_{\ctrl} + \qfunc_{\ctrl\stat} \dstat + \qfunc_{\ctrl\ctrl} \dctrl + \qfunc_{\ctrl\lmIn} \dlmIn + \qfunc_{\ctrl\lmEq} \dlmEq = \vecbf{0}.
\end{equation*}

\item for the maximizing variable $\dlmIn$ related to the inequality constraint, it must satisfy the dual feasibility condition $\lmIn + \dlmIn \geq \vecbf{0}$ and the complementary condition:
\begin{equation*}
    \Diag(\lmIn + \dlmIn)(\qfunc_{\lmIn} + \qfunc_{\lmIn\stat} \dstat + \qfunc_{\lmIn\ctrl} \dctrl) = \vecbf{0}.
\end{equation*}
Omitting the second-order terms, adding a perturbation vector $\mu \vecbf{1}$ on the left-hand side where $\pert$ is a perturbation variable, and rearranging the above equation, one has
\begin{equation*}
    \dlmIn = - [\Diag(\inc)]^{-1}(\Resi_{\mathrm{in}} +\Diag(\lmIn)\qfunc_{\lmIn\stat} \dstat + \Diag(\lmIn)\qfunc_{\lmIn\ctrl} \dctrl ),
\end{equation*}
where $\Resi_{\mathrm{in}} := \Diag(\lmIn) \inc + \mu \vecbf{1}$.

\item for the maximizing variable $\dlmEq$ related to the equality constraint, it must satisfy the primal feasibility condition:
\begin{equation*}
    \qfunc_{\lmEq} + \qfunc_{\lmEq\stat} \dstat + \qfunc_{\lmEq\ctrl} \dctrl = \vecbf{0}.
\end{equation*}
Inspired from the perturbed complementarity equation for equality constraints, e.g. \cite[Eq. (6.22)]{forsgren2002interior}, adding a perturbation term $\pert (\lmEq + \dlmEq)$ on the right-hand side and one has:
\begin{equation*}
    \dlmEq = \pert^{-1} (\qfunc_{\lmEq\stat} \dstat + \qfunc_{\lmEq\ctrl} \dctrl) - \Resi_{\mathrm{eq}},
\end{equation*}
where $ \Resi_{\mathrm{eq}} :=  \lmEq  - \pert^{-1} \qfunc_{\lmEq} = \lmEq  - \pert^{-1} \eqc$.
\end{itemize}

Substituting $\dlmIn$ and $\dlmEq$ defined above back into the stationarity condition, one has the following feedback control law:
\begin{equation}
\label{eq:dctrl}
\dctrl = \gaink + \gainK \dstat
\end{equation}
where
\begin{equation}
\label{eq:hat_Q}
    \begin{aligned}
         \gaink & = - \hat{\qfunc}_{\ctrl\ctrl}^{-1}\hat{\qfunc}_{\ctrl}, \gainK =  - \hat{\qfunc}_{\ctrl\ctrl}^{-1} \hat{\qfunc}_{\ctrl\stat}, \\ 
        \hat{\qfunc}_{\ctrl} & = \qfunc_{\ctrl} - \qfunc_{\ctrl\lmIn} [\Diag(\inc)]^{-1} \Resi_{\mathrm{in}} - \qfunc_{\ctrl\lmEq} \Resi_{\mathrm{eq}}, \\
        \hat{\qfunc}_{\ctrl\stat} & = \qfunc_{\ctrl\stat} - \qfunc_{\ctrl\lmIn} [\Diag(\inc)]^{-1} \Diag(\lmIn) \qfunc_{\lmIn\stat} + \pert^{-1} \qfunc_{\ctrl\lmEq}\qfunc_{\lmEq\stat}, \\
        \hat{\qfunc}_{\ctrl\ctrl} & = \qfunc_{\ctrl\ctrl} - \qfunc_{\ctrl\lmIn} [\Diag(\inc)]^{-1} \Diag(\lmIn) \qfunc_{\lmIn\ctrl} + \pert^{-1} \qfunc_{\ctrl\lmEq}\qfunc_{\lmEq\ctrl}.
    \end{aligned}
\end{equation}
 Compared with \cite{pavlov2021interior}, it can be found from the above definitions of $\hat{\qfunc}_{(\cdot)}$ that an additional term (i.e., the third term) was introduced to deal with the equality constraint $\eqc$.
 Next, one also substitutes the feedback control law \eqref{eq:dctrl} into $\dlmIn$ and $\dlmEq$ to obtain their expressions in feedback form:
 \begin{equation}
 \label{eq:dlm}
    \dlmIn = \gaink_{\mathrm{in}} + \gainK_{\mathrm{in}} \dstat, \quad \dlmEq = \gaink_{\mathrm{eq}} + \gainK_{\mathrm{eq}} \dstat,
\end{equation}
where 
\begin{equation*}
\begin{aligned}
    \gaink_{\mathrm{in}} & = - [\Diag(\inc)]^{-1}(\Resi_{\mathrm{in}} + \Diag(\lmIn)\qfunc_{\lmIn\ctrl} \gaink), \\
    \gainK_{\mathrm{in}} & = - [\Diag(\inc)]^{-1}(\Diag(\lmIn)\qfunc_{\lmIn\stat} + \Diag(\lmIn)\qfunc_{\lmIn\ctrl} \gainK), \\
    \gaink_{\mathrm{eq}} & =  - \Resi_{\mathrm{eq}} + \pert^{-1} \qfunc_{\lmEq\ctrl} \gaink, \\   \gainK_{\mathrm{eq}} & = \pert^{-1} (\qfunc_{\lmEq\stat} + \qfunc_{\lmEq\ctrl} \gainK).
\end{aligned}
\end{equation*}

After finding the solution of $\dctrl$, we shall update the derivatives related to optimal cost-to-go by following the traditional DDP algorithm, i.e.,
\begin{equation}
\label{eq:qfunc_to_vfunc}
\begin{aligned}
\vfunc_{\stat} & = \hat{\qfunc}_{\stat} - \hat{\qfunc}_{\ctrl\stat}^{\top}\hat{\qfunc}_{\ctrl\ctrl}^{-1}\hat{\qfunc}_{\ctrl} = \hat{\qfunc}_{\stat} - \gainK^{\top}\hat{\qfunc}_{\ctrl\ctrl}\gaink, \\ 
 V_{\vecbf{\stat\stat}} & = \hat{\qfunc}_{\stat\stat} - \hat{\qfunc}_{\ctrl\stat}^{\top}\hat{\qfunc}_{\ctrl\ctrl}^{-1}\hat{\qfunc}_{\ctrl\stat} = \hat{\qfunc}_{\stat \stat} - \gainK^{\top} \hat{\qfunc}_{\ctrl\ctrl} \gainK,
\end{aligned}
\end{equation} 
where $\hat{\qfunc}_{\stat}$ and $\hat{\qfunc}_{\stat \stat}$ can be obtained by 
replacing the subscript $(\cdot)_\ctrl$ in \eqref{eq:hat_Q} with $(\cdot)_\stat$. 
Repeating the above alternating update of the cost-to-go $\qfunc$ and the optimal cost-to-go $\vfunc$ for $k= N-1, \dots, 0$, one can obtain a set of control gains $\{\gaink, \gainK, \gaink_{\mathrm{in}}, \gainK_{\mathrm{in}}, \gaink_{\mathrm{eq}}, \gainK_{\mathrm{eq}}\}$ and this completes the backward recursions.

In the forward recursions, we aim to obtain an updated trajectory $\trajset^{\dagger}$ by using the system dynamics and the control gains obtained above, i.e., repeating the following computation 
\begin{equation}
\label{eq:fwd}
	\begin{aligned}
		\ctrl^{\dagger} & = \ctrl + \gaink + \gainK (\stat^{\dagger} - \stat), \\
        \lmIn^{\dagger} & = \lmIn + \gaink_{\mathrm{in}} + \gainK_{\mathrm{in}} (\stat^{\dagger} - \stat), \\
        \lmEq^{\dagger} & = \lmEq + \gaink_{\mathrm{eq}} + \gainK_{\mathrm{eq}} (\stat^{\dagger} - \stat), \\
		\stat^{+\dagger} & = \dyn(\stat^{\dagger}, \ctrl^{\dagger}),
	\end{aligned}
\end{equation}
for $k=0, \dots, N-1$ with fixed the initial condition $\stat^{\dagger}_{0} = \stat_{0}$.

We summarize the proposed generalized interior-point DDP-based trajectory in Algorithm \ref{alg:ipddp_traj_solver}, where the above-mentioned backward and forward recursions can be found in lines \ref{alg.ddp.bwd1} to \ref{alg.ddp.bwd2}, and lines \ref{alg.ddp.fwd1} to \ref{alg.ddp.fwd2}, respectively.    
Note that in the practical implementation of the above algorithm, regularization terms should be added in the backward recursions to guarantee the positive-definiteness of $\hat{\qfunc}_{\ctrl\ctrl}$. 
Line-search methods should be added to the forward recursions to preserve the primal feasibility, i.e., $\inc < \vecbf{0}$ and $\eqc = \vecbf{0}$.

\begin{algorithm}
	\caption{DDP-based trajectory solver}
	\label{alg:ipddp_traj_solver}
	\begin{algorithmic}[1]
		\Require system \eqref{eq:prob_ineq_eq}, parameter $\theta$, initial state $\stat_{0}$, initial solution $\ctrlset_{0}$, initial Lagrangian multipliers $\lmIn_{0}, \lmEq_{0}$ and tolerance $\mathrm{tol}$ 
		\Ensure optimal solution $\ctrlset$				
        \While {$\mathrm{merit} > \mathrm{tol}$}
        \State set $\vfunc_{\stat, N} = \termcost_{\stat}$, $\vfunc_{\stat\stat, N} = \termcost_{\stat\stat}$
		\For{$k = N-1, \dots, 0$}  \label{alg.ddp.bwd1}
		\State evaluate $\hat{\qfunc}_{(\cdot)}$ using \eqref{eq:hat_Q}
        \State compute control gains in \eqref{eq:dctrl} and \eqref{eq:dlm}
        \State update $\vfunc_{\stat}$, $\vfunc_{\stat\stat}$ using \eqref{eq:qfunc_to_vfunc}
		\EndFor \label{alg.ddp.bwd2}
		\State set $\stat_{0}^{\dagger} = \stat_{0}$
		\For{$k = 0, \dots, N-1$} \label{alg.ddp.fwd1}
		\State Update the control variable $\ctrl_{k}^{\dagger}$, multipliers $\lmIn_{k}^{\dagger}$, $\lmEq_{k}^{\dagger}$  and next state $\stat_{k+1}^{\dagger}$ according to \eqref{eq:fwd}
		\EndFor		\label{alg.ddp.fwd2}
        \EndWhile 
	\end{algorithmic}
\end{algorithm}
\begin{rem}
    Under the assumption that $\hat{\qfunc}_{\ctrl\ctrl}$ is positive-definite for all $k \in \myset{I}_{N}$, one can establish the local quadratic convergence by following the proof of \cite[Theorem 2]{pavlov2021interior} with the vector-valued merit function and the linear operator being respectively defined by $\mathrm{merit} = [ \qfunc_{\ctrl}^{\top}, \Resi_{\mathrm{in}}^{\top}, \Resi_{\mathrm{eq}}^{\top}]^{\top}$ and 
   {\small \begin{equation*}
        \begin{bmatrix}
    \qfunc_{\ctrl\ctrl}  & \qfunc_{\ctrl\lmIn}& \qfunc_{\ctrl\lmEq} \\
    \Diag(\lmIn) \qfunc_{\lmIn\ctrl} & \Diag(\inc) & \matbf{0} \\
    \pert^{-1} \qfunc_{\lmEq\ctrl} & \matbf{0} & \matbf{0}
\end{bmatrix}^{-1}.
    \end{equation*}}
\end{rem}

\begin{rem}
\label{rem:acset}
Another algorithm called active-set-based DDP algorithm, which also adopts the backward-forward structure, can also be used to solve the general nonlinear constrained optimal control problem. In the backward recursions, it first identifies the active inequality constraint while excluding the inactive parts since it does not contribute to the optimal solution. Then it considers solving the following problem:
\begin{equation}
\label{eq:dq_problem_acset}
\begin{aligned}
    \min_{\dctrl} \dqfunc\acset~\mathrm{s.t.}~ \eqc\acset(\stat + \dstat, \ctrl + \dctrl; \para) = \vecbf{0},
\end{aligned}
\end{equation}
where $\eqc\acset$ concatenates $\eqc$ and the rows of $\inc$ which equals $\vecbf{0}$. To solve this, the following KKT condition is used:
\begin{equation}
\begin{bmatrix}
    \qfunc\acset_{\ctrl\ctrl}  & (\eqc\acset_{\ctrl})^{\top} \\
    \eqc\acset_{\ctrl} & \matbf{0}
\end{bmatrix}
    \begin{bmatrix}
     \dctrl \\ \lmEq\acset
    \end{bmatrix} = - \begin{bmatrix}
     \qfunc\acset_{\ctrl\stat} \\ \eqc\acset_{\stat}
    \end{bmatrix} \dstat - \begin{bmatrix}
    \qfunc\acset_{\ctrl} \\ \vecbf{0}
    \end{bmatrix},
\end{equation}
where $\qfunc\acset 
:= \pathcost + (\vfunc\acset)^{+}$ here is redefined with new optimal cost-to-go $\vfunc\acset$ and $\lmEq\acset$ is the Lagrangian multiplier for the equality constraint $\eqc\acset = \vecbf{0}$. We refer to \cite{xie2017differential} for a more detailed process.
\end{rem}

\subsection{DDP-based gradient solver}
\label{subsec:backward}
In the last subsection, we have presented a DDP-based trajectory solver for obtaining the optimal solution $\ctrlset$ of the constrained multi-stage optimal control problem \eqref{eq:prob_ineq_eq} given the current parameter $\para$, from which both $\pdv{L^{\mathrm{ol}}}{\trajset}$ and $\pdv{L^{\mathrm{ol}}}{\para}$ can be easily computed. 
In this subsection, we aim to present an efficient algorithm, which is referred to as DDP-based gradient solver, to obtain the remaining term $\dv{\trajset}{\para}$ in order to update $\para_{t}$ as in \eqref{eq:para_upd}.

Let us first take a detour to consider the computation of the optimal solution w.r.t. the parameter for the following single-stage unconstrained optimization problem 
\begin{equation*}
    \stat^{\ast} = \arg \min_{\stat} \examfunc(\stat;\para),
\end{equation*}
where $\examfunc: \mathbb{R}^{\size{\stat}} \times \mathbb{R}^{\size{\para}} \to \mathbb{R}$ is a scalar function parameterized by $\para$. Then $\stat^{\ast}$ should satisfy the first-order necessary equilibrium condition
\begin{equation*}
    \examfunc_{\stat}(\stat^{\ast}; \para) = \vecbf{0}.
\end{equation*} 
To obtain the gradient of $\stat^{\ast}$ w.r.t $\para$, one approach is to differentiate the above equation w.r.t. $\para$, i.e.,
\begin{equation*}
    \examfunc_{\stat\stat} \stat_{\para} + \examfunc_{\stat\para} = \vecbf{0},
\end{equation*}
from which one can obtain $\stat_{\para} = - (\examfunc_{\stat\stat})^{-1}\examfunc_{\stat\para}$. 
Alternatively, one can write $\parastat := [\para^{\top},\stat^{\ast \top}]^{\top}$ and take the variation of $\bar{\examfunc}_{\stat}(\parastat):= \examfunc_{\stat}(\stat^{\ast}; \para)$ w.r.t. $\parastat$,
\begin{equation*}
    \bar{\examfunc}_{\stat\parastat} \dparastat =  [\bar{\examfunc}_{\stat\para}~ \bar{\examfunc}_{\stat\stat}] \begin{bmatrix}
        \dpara \\ \dstat
    \end{bmatrix} =  \vecbf{0},
\end{equation*}
from which one can obtain $\stat_{\para} = \frac{\dstat}{\dpara} = - (\bar{\examfunc}_{\stat\stat})^{-1}\bar{\examfunc}_{\stat\para}$.

Indeed, observing that for the constrained multi-stage optimal control problem, we can view 
\begin{equation*}
   \bar{\hat{\qfunc}}_{\ctrl}(\parastat,\ctrl):= \hat{\qfunc}_{\ctrl}(\stat,\ctrl;\para) = \vecbf{0}
\end{equation*}
as the equilibrium condition of optimization problem \eqref{eq:dq_problem} at each time instant $k$, where the dependence on the dual variables has been removed by substitution and we have slightly abused the notation $\parastat:=[\para^{\top},\stat^{\top}]^{\top}$. Taking the variation, one has
\begin{equation*}
    [\bar{\hat{\qfunc}}_{\ctrl\parastat} ~\bar{\hat{\qfunc}}_{\ctrl\ctrl}] 
    \begin{bmatrix}
    \dparastat \\ \dctrl
    \end{bmatrix} = \vecbf{0},
\end{equation*}
from which one can obtain
\begin{equation}
\label{eq:dctrldpara}
    \frac{\dctrl}{\dpara} = - \bar{\hat{\qfunc}}_{\ctrl\ctrl}^{-1}\bar{\hat{\qfunc}}_{\ctrl\parastat} \begin{bmatrix}
    \matbf{I} \\ \frac{\dstat}{\dpara}
    \end{bmatrix}.
\end{equation}
In the above, both $\frac{\dctrl}{\dpara}$ and $\frac{\dstat}{\dpara}$ are exactly the elements in $\dv{\trajset}{\para}$ and they are connected for each time instant $k$. However, currently, we are only given $\frac{\dstat_{0}}{\dpara} = \vecbf{0}$ (since $\stat_{0}$ is fixed), which is insufficient to compute $\{\frac{\dstat_{k}}{\dpara}\}_{k=0}^{N}$ and $\{\frac{\dctrl_{k}}{\dpara}\}_{k=0}^{N-1}$. 
To address this, one should be able to compute the matrix $\bar{\hat{\qfunc}}_{\ctrl\parastat}$, and also $\{\frac{\dstat_{k}}{\dpara}\}_{k=1}^{N}$.

We shall consider the following augmented system:
\begin{equation}
\label{eq:prob_aug}
\begin{aligned}
\min_{\myset{U}} &~~ \sum_{k=0}^{N-1} \bar{\pathcost}(\parastat_{k}, \ctrl_{k}) + \bar{\termcost}(\parastat_{N}) \\
\mathrm{s.t.} &~~ \parastat_{k+1} = \bar{\dyn}(\parastat_{k}, \ctrl_{k}) = \begin{bmatrix}
    \para \\ \dyn(\stat_{k},\ctrl_{k};\para)
\end{bmatrix}, \parastat_{0}~\mathrm{is~given} \\
&~~ \bar{\inc}(\parastat_{k}, \ctrl_{k}) \leq \vecbf{0}, \\
&~~ \bar{\eqc}(\parastat_{k}, \ctrl_{k}) = \vecbf{0},
\end{aligned}
\end{equation}
where all the functions $f(\dots; \para), f \in \{\pathcost, \termcost, \dyn, \inc, \eqc\}$ parameterized by $\para$ in \eqref{eq:prob_ineq_eq} have been replaced by their counterpart $\bar{f}$ with $\parastat$ being the new state variable for the augmented system. 

%Similarly, redefine the cost-to-go function as $$\bar{\qfunc} := \bar{\pathcost} + \lmIn^{\top}\bar{\inc} + \lmEq^{\top}\bar{\eqc} + \bar{\vfunc}^{+},$$ where $\bar{\vfunc}^{+}(\parastat)$ is a redefined optimal cost-to-go.

Define the following quantities:
\begin{equation}
\label{eq:hat_Q_for_gradient}
    \begin{aligned}
        \bar{\hat{\qfunc}}_{\ctrl} 
        & = \bar{\pathcost}_{\ctrl} + \bar{\dyn}_{\ctrl}^{\top}\bar{\vfunc}_{\parastat}^{+} + \pert \bar{\inc}_{\ctrl}^{\top} [\Diag(\bar{\inc})]^{-1}\vecbf{1} + \pert^{-1} \bar{\eqc}_{\ctrl}^{\top} \bar{\eqc}, \\
        \bar{\hat{\qfunc}}_{\ctrl\parastat} 
        & = \bar{\pathcost}_{\ctrl\parastat} + \bar{\dyn}_{\ctrl}^{\top}\bar{\vfunc}_{\parastat\parastat}^{+}\bar{\dyn}_{\parastat} + \bar{\vfunc}_{\parastat}^{+} \odot \bar{\dyn}_{\ctrl\parastat} \\ 
        & \quad~~~~~+ \pert \bar{\inc}_{\ctrl}^{\top} [\Diag(\bar{\inc})]^{-2} \bar{\inc}_{\parastat} - \pert ([\Diag(\bar{\inc})]^{-1} \vecbf{1}) \odot \bar{\inc}_{\ctrl\parastat}   \\
        & \quad~~~~~+ \pert^{-1} \bar{\eqc}_{\ctrl}^{\top}\bar{\eqc}_{\parastat} + \pert^{-1}\bar{\eqc} \odot \bar{\eqc}_{\ctrl\parastat},\\ 
        \bar{\hat{\qfunc}}_{\ctrl\ctrl} 
        & = \bar{\pathcost}_{\ctrl\ctrl}  + \bar{\dyn}_{\ctrl}^{\top}\bar{\vfunc}_{\parastat\parastat}^{+}\bar{\dyn}_{\ctrl}  + \bar{\vfunc}_{\parastat}^{+} \odot \bar{\dyn}_{\ctrl\ctrl} \\
        & \quad~~~~~+ \pert \bar{\inc}_{\ctrl}^{\top} [\Diag(\bar{\inc})]^{-2} \bar{\inc}_{\ctrl} - \pert ([\Diag(\bar{\inc})]^{-1} \vecbf{1}) \odot \bar{\inc}_{\ctrl\ctrl} \\
        & \quad~~~~~+ \pert^{-1} \bar{\eqc}_{\ctrl}^{\top}\bar{\eqc}_{\ctrl} + \pert^{-1}\bar{\eqc} \odot \bar{\eqc}_{\ctrl\ctrl},\\ 
    \end{aligned}
\end{equation}
where no dual variables were involved. The following result establishes the connection between one iteration of backward-forward recursion on this augmented system and the gradient of the optimal trajectory w.r.t. learning parameter. 
\begin{thm}
\label{thm:ddp}
    Suppose $\trajset$ is the optimal solution to \eqref{eq:prob_ineq_eq} with perturbation $\pert$, and $\bar{\hat{\qfunc}}_{\ctrl\ctrl}$ is invertiable for $k=0, \dots, N-1$. 
    The derivative of solved trajectory w.r.t. the learning parameter $\dv{\trajset}{\para}$ can be obtained by iteratively updating \eqref{eq:dctrldpara} and 
    \begin{equation}
\label{eq:dstatdpara}
\begin{bmatrix}
        \frac{\dpara^{+}}{\dpara} \\ \frac{\dstat^{+}}{\dpara}
    \end{bmatrix} 
     =  \bar{\dyn}_{\parastat} \begin{bmatrix}
        \matbf{I} \\ \frac{\dstat}{\dpara}
    \end{bmatrix}  + \bar{\dyn}_{\ctrl}\frac{\dctrl}{\dpara}.
\end{equation} for $k=0, \dots, N-1$, with $\frac{\dstat_{0}}{\dpara} = \vecbf{0}$ and $\bar{\hat{\qfunc}}_{(\cdot)}$ being defined in \eqref{eq:hat_Q_for_gradient}.
\end{thm}

\begin{proof}
Suppose we are using the DDP-based trajectory solver to find the optimal solution of the augmented system \eqref{eq:prob_aug}.
Following the process in Section \ref{subsec:backward}, redefine the cost-to-go function as $$\bar{\qfunc} := \bar{\pathcost} + \bar{\lmIn}^{\top}\bar{\inc} + \bar{\lmEq}^{\top}\bar{\eqc} + \bar{\vfunc}^{+},$$ where $\bar{\vfunc}^{+}(\parastat)$ is a redefined optimal cost-to-go.
Then the backward recursion reads 
\begin{equation*}
\label{eq:dctrl_aug}
\dctrl = - \hat{\bar{\qfunc}}_{\ctrl\ctrl}^{-1}(\hat{\bar{\qfunc}}_{\ctrl} + \hat{\bar{\qfunc}}_{\ctrl\parastat} \dparastat),
\end{equation*}
where the terms $\hat{\bar{\qfunc}}_{(\cdot)}$ related to the redefined cost-to-go are expressed as follows:
\begin{equation}
\label{eq:hat_Q_for_gradient_aug}
    \begin{aligned}
        \hat{\bar{\qfunc}}_{\ctrl} & = \bar{\qfunc}_{\ctrl} - \bar{\qfunc}_{\ctrl\bar{\lmIn}} [\Diag(\bar{\inc})]^{-1} \bar{\Resi}_{\mathrm{in}} - \bar{\qfunc}_{\ctrl\bar{\lmEq}} \bar{\Resi}_{\mathrm{eq}}, \\
        \hat{\bar{\qfunc}}_{\ctrl\parastat} & = \bar{\qfunc}_{\ctrl\parastat} - \bar{\qfunc}_{\ctrl\bar{\lmIn}} [\Diag(\bar{\inc})]^{-1} \Diag(\bar{\lmIn}) \bar{\qfunc}_{\bar{\lmIn}\parastat} + \pert^{-1} \bar{\qfunc}_{\ctrl\bar{\lmEq}}\bar{\qfunc}_{\bar{\lmEq}\parastat}, \\ 
        \hat{\bar{\qfunc}}_{\ctrl\ctrl} & = \bar{\qfunc}_{\ctrl\ctrl} - \bar{\qfunc}_{\ctrl\bar{\lmIn}} [\Diag(\bar{\inc})]^{-1} \Diag(\bar{\lmIn}) \bar{\qfunc}_{\bar{\lmIn}\ctrl} + \pert^{-1} \bar{\qfunc}_{\ctrl\bar{\lmEq}}\bar{\qfunc}_{\bar{\lmEq}\ctrl}.
    \end{aligned}
\end{equation}
Notice that all the terms $\bar{(\cdot)}$ involved in the right-hand side are redefined with the new state variable $\parastat$. Letting $\bar{\lmIn} \equiv \lmIn$ and $\bar{\lmEq} \equiv \lmEq$, one can find that $\hat{\bar{f}} = \bar{\hat{f}}$ for $f \in \{\qfunc_{\ctrl}, \qfunc_{\ctrl\parastat}, \qfunc_{\ctrl\ctrl}\}$.

By assumption that $\trajset$ is the optimal solution, or equivalently, the optimization problem \eqref{eq:prob_ineq_eq} has been solved in the sense of the merit function, $\mathrm{merit} = [ \qfunc_{\ctrl}^{\top}, \Resi_{\mathrm{in}}^{\top}, \Resi_{\mathrm{eq}}^{\top}]^{\top} = \vecbf{0}$, which implies that $\lmIn = - \pert [\Diag(\inc)]^{-1} \vecbf{1}$ and $\lmEq = \pert^{-1}\eqc$. 
Substituting these to the Lagrange multipliers $\bar{\lmIn}, \bar{\lmEq}$ in \eqref{eq:hat_Q_for_gradient_aug}, one can obtain \eqref{eq:hat_Q_for_gradient}.

After obtaining the matrices $\bar{\hat{\qfunc}}_{\ctrl\parastat}$ and $\bar{\hat{\qfunc}}_{\ctrl\ctrl}$, we are now at the stage of figuring out how to compute $\{\frac{\dstat_{k}}{\dpara}\}_{k=1}^{N}$ in order to compute the rest unknown $\{\frac{\dctrl_{k}}{\dpara}\}_{k=1}^{N-1}$ by virtue of \eqref{eq:dctrldpara}. Note that $\trajset$ is the optimal solution which satisfies the dynamic equation in \eqref{eq:prob_ineq_eq} and hence the augmented trajectory $\bar{\trajset}$ satisfies the dynamic equation in \eqref{eq:prob_aug}. Taking the variation of the latter, one has
\begin{equation}
\begin{aligned}
    \dparastat^{+} := \begin{bmatrix}
        \dpara^{+} \\ \dstat^{+}
    \end{bmatrix} 
    & =     \bar{\dyn}_{\parastat}\dparastat + \bar{\dyn}_{\ctrl}\dctrl \\
    & =
    \begin{bmatrix}
        \matbf{I} & \matbf{0} \\
        \dyn_{\para} & \dyn_{\stat}
    \end{bmatrix} \begin{bmatrix}
        \dpara \\ \dstat
    \end{bmatrix} + \begin{bmatrix}
        \matbf{0} \\ \dyn_{\ctrl}
    \end{bmatrix} \dctrl,
\end{aligned}
\end{equation}
by which we can establish the relationship among $\frac{\dstat^{+}}{\dpara}$, $\frac{\dpara^{+}}{\dpara}$, $\frac{\dstat}{\dpara}$ and $\frac{\dctrl}{\dpara}$.
Therefore, repeatedly evaluating $\frac{\dstat}{\dpara}$ and $\frac{\dctrl}{\dpara}$ according to \eqref{eq:dctrldpara} and \eqref{eq:dstatdpara} for $k=0, \dots, N-1$, one can obtain $\dv{\trajset}{\para}$.
\end{proof}
Theorem \ref{thm:ddp} implies that the gradient of trajectory w.r.t. parameter can be computed by performing a single backward-forward recursion on the augmented system with the augmented optimal trajectory being the initial solution. 
Intuitively, during each iteration in the backward recursion, the solver finds the affine relationship between the variation of input $\dctrl$ and that of augmented state $\dparastat$, which also leads to the affine relationship between the gradients [see \eqref{eq:dctrldpara}]. 
Next, during each iteration in the forward recursion, an affine relationship between the gradients [see \eqref{eq:dstatdpara}] can also be established by utilizing an affine relationship among the variation of augment state at next time $\dparastat^{+}$, that of input $\dctrl$ and that of augmented state $\dparastat$.
Consequently, this DDP-based gradient solver enjoys the linear computational complexity $\myset{O}(N)$. 

\begin{rem}
\label{rem:equiv_ddp_pdp}
As mentioned in Section \ref{sec:intro}, another framework called PDP (and its variant SafePDP) has been proposed in \cite{jin2020pontryagin,jin2021safe} as a gradient solver, where the PMP conditions are differentiated to obtain the implicit relationships between the learning parameter and the optimal trajectory.
Due to the close relationship between dynamic programming and PMP on optimal control problems, it is natural to ask if DDP-based and PDP-based gradient solvers, which are their respective differentiated versions, will inherit this relationship. 
We provide an affirmative answer to this, i.e., the computation of the gradient term from DDP method as in Theorem \ref{thm:ddp} is equivalent to \cite[Theorem 2(c)]{jin2021safe}. 
This can be shown by viewing the Hamiltonian function $L$ and the dual variable $\gbf{\lambda}$ for the dynamics constraint defined in \cite{jin2021safe} as the cost-to-go $\qfunc$ defined in \eqref{eq:minmax_bellman_ineq_eq} and $\vfunc_{\stat}$ defined in \eqref{eq:qfunc_to_vfunc}, respectively. 
In addition, this equivalence will also be validated by numerical simulations in Section \ref{subsec:simu_pdp_ddp_comp}. 
Compared to PDP-based method, our derivation of the one-step DDP on the augmented system is more compact and easier to be interpreted, i.e., the affine relationship among the gradients follows from that among the variations. 
In terms of computation, it has been shown in \cite{jin2021safe} that the PDP-based gradient solver is of computational complexity $\myset{O}(N)$. 
It has been widely perceived that the DDP method is much less efficient than iterative linear quadratic regulator (iLQR) due to the introduction of 3-dimensional tensor $\dyn_{\vecbf{a}\vecbf{b}}$, $\vecbf{a}, \vecbf{b} \in \{\stat, \ctrl\}$, however, the tensor evaluation can be avoided if we view $\vecbf{c} \odot \dyn_{\vecbf{a}\vecbf{b}}$ as a matrix-valued function with $(\vecbf{c}, \vecbf{a}, \vecbf{b})$ as its arguments, which only has the same cost of evaluating other matrix-valued functions (e.g., $\pathcost_{\vecbf{a}\vecbf{b}}$). 
Consequently, as will be shown by numerical simulations in Section \ref{subsec:simu_pdp_ddp_comp}, DDP-based gradient solver consumes less computational time for systems with high dimensions, which benefits from the compact form of our derivation. 
Most importantly, as in the optimal control problem where DDP can provide a closed-loop feedback policy for subsequent control and hence provide a more robust performance than PMP, the proposed DDP-based gradient solver also provides some intermediate matrices as byproducts, which can be further used to construct a new closed-loop loss function. 
As will be seen from Section \ref{sec:cl_irl}, this new loss function leads to a better performance compared to the case using the open-loop loss function.
\end{rem}

On the other hand, if the active-set method was used as the trajectory solver (either the off-the-shelf commercial solver or the active-set DDP-based approach mentioned in Remark \ref{rem:acset}), it is equivalent to solving the equality-constrained ($\eqc\acset = \vecbf{0}$) optimal control problem.
In order to compute the gradient, we consider the following augmented system:
\begin{equation}
\label{eq:prob_aug_acset}
\begin{aligned}
\min_{\myset{U}} &~~ \sum_{k=0}^{N-1} \bar{\pathcost}(\parastat_{k}, \ctrl_{k}) + \bar{\termcost}(\parastat_{N}) \\
\mathrm{s.t.} &~~ \parastat_{k+1} = \bar{\dyn}(\parastat_{k}, \ctrl_{k}), \parastat_{0}~\mathrm{is~given}, \\
&~~ \bar{\eqc}\acset = \vecbf{0},
\end{aligned}
\end{equation}
where $\bar{f}$, $f \in \{\pathcost, \termcost, \dyn\}$ shares the same definitions of those in \eqref{eq:prob_aug} while the active equality constraint is defined as $\bar{\eqc}\acset(\parastat_{k}, \ctrl_{k}) := \eqc\acset(\stat, \ctrl; \para)$.

Define $\bar{\qfunc}\acset := \bar{\pathcost} + (\bar{\vfunc}\acset)^{+}$, by which one can obtain its partial derivatives $\bar{\qfunc}\acset_{\ctrl}, \bar{\qfunc}\acset_{\ctrl\parastat}, \bar{\qfunc}\acset_{\ctrl\ctrl}$ by definition:
\begin{equation}
\label{eq:hat_Q_for_gradient_acset}
    \begin{aligned}
        \bar{\qfunc}\acset_{\ctrl} & = \bar{\pathcost}_{\parastat} + \bar{\dyn}_{\parastat}^{\top}(\bar{\vfunc}_{\parastat}\acset)^{+} , \\
        \bar{\qfunc}\acset_{\ctrl\parastat} & = \bar{\pathcost}_{\ctrl\parastat} + \bar{\dyn}_{\ctrl}^{\top}(\bar{\vfunc}_{\parastat\parastat}\acset)^{+}\dyn_{\parastat} + (\bar{\vfunc}_{\parastat}\acset)^{+} \odot \bar{\dyn}_{\ctrl\parastat}, \\
        \bar{\qfunc}\acset_{\ctrl\ctrl} & = \bar{\pathcost}_{\ctrl\ctrl} +  \bar{\dyn}_{\ctrl}^{\top}(\bar{\vfunc}_{\parastat\parastat}\acset)^{+}\dyn_{\ctrl} + (\bar{\vfunc}_{\parastat}\acset)^{+} \odot \dyn_{\ctrl\ctrl}.
    \end{aligned}
\end{equation}
Additionally, define 
\begin{equation}
    \matbf{A} :=  [\bar{\eqc}\acset_{\ctrl}(\bar{\qfunc}\acset_{\ctrl\ctrl})^{-1} (\bar{\eqc}\acset_{\ctrl})^{\top} ]^{-1}\bar{\eqc}\acset_{\ctrl}(\bar{\qfunc}\acset_{\ctrl\ctrl})^{-1}.
\end{equation}
The following result establishes the relationship between the one-time backward-forward recursion on the augmented system and the gradient of optimal trajectory w.r.t. learning parameter. 

\begin{thm}
\label{thm:ddp_acset}
    Suppose $\trajset$ is the optimal solution to the optimal control problem \eqref{eq:prob_ineq_eq}, $\bar{\qfunc}\acset_{\ctrl\ctrl}$ is invertiable and $\bar{\eqc}\acset_{\ctrl}$ is full row-rank for $k=0, \dots, N-1$. 
    The derivative of solved trajectory w.r.t. the learning parameter $\dv{\trajset}{\para}$ can be obtained by iteratively updating 
    \begin{equation}
\label{eq:dctrldpara_acset}
    \frac{\dctrl}{\dpara} = \begin{bmatrix}(\bar{\qfunc}\acset_{\ctrl\ctrl})^{-1}[\matbf{I} - (\bar{\eqc}\acset_{\ctrl})^{\top} \matbf{A}] &  \matbf{A}^{\top}
    \end{bmatrix}
    \begin{bmatrix}
     \bar{\qfunc}\acset_{\ctrl\parastat} \\ \bar{\eqc}\acset_{\parastat}
    \end{bmatrix} 
    \begin{bmatrix}
    \matbf{I} \\ \frac{\dstat}{\dpara}
    \end{bmatrix}
\end{equation}
    and \eqref{eq:dstatdpara}
 for $k=0, \dots, N-1$, with $\frac{\dstat_{0}}{\dpara} = \vecbf{0}$, where $\qfunc\acset_{(\cdot)}$ is defined in \eqref{eq:hat_Q_for_gradient_acset}.
\end{thm}
\begin{proof}
The result can be established by following a similar process as that of Theorem \ref{thm:ddp}. 
Suppose we use the active-set method to find the optimal solution of the augmented system \eqref{eq:prob_aug_acset}, note that here due to the absence of inequality constraint, the set identification step can be omitted. 
Then from the KKT condition 
\begin{equation}
\begin{bmatrix}
    \bar{\qfunc}\acset_{\ctrl\ctrl}  & (\bar{\eqc}\acset_{\ctrl})^{\top} \\
    \bar{\eqc}\acset_{\ctrl} & \matbf{0}
\end{bmatrix}
    \begin{bmatrix}
     \dctrl \\ \lmEq\acset
    \end{bmatrix} = - \begin{bmatrix}
     \bar{\qfunc}\acset_{\ctrl\parastat} \\ \bar{\eqc}\acset_{\parastat}
    \end{bmatrix} \dparastat - \begin{bmatrix}
    \bar{\qfunc}\acset_{\ctrl} \\ \vecbf{0}
    \end{bmatrix},
\end{equation}
one can obtain the backward recursion
\begin{equation*}
\begin{aligned}
    \dctrl & = \begin{bmatrix}(\bar{\qfunc}\acset_{\ctrl\ctrl})^{-1}[\matbf{I} - (\bar{\eqc}\acset_{\ctrl})^{\top} \matbf{A}] &  \matbf{A}^{\top}
    \end{bmatrix}
    (\begin{bmatrix}
     \bar{\qfunc}\acset_{\ctrl\parastat} \\ \bar{\eqc}\acset_{\parastat}
    \end{bmatrix} 
    \dparastat + \begin{bmatrix}
    \bar{\qfunc}\acset_{\ctrl} \\ \vecbf{0}
    \end{bmatrix}), \\
    & =: \bar{\gaink}\acset + \bar{\gainK}\acset \dparastat.  
\end{aligned}
\end{equation*}
where $\bar{\gaink}\acset$ and $\bar{\gainK}\acset$ are used to update $(\bar{\vfunc}_{\parastat\parastat}\acset)^{+}$ and $(\bar{\vfunc}_{\parastat}\acset)^{+}$ similar to what was done in \eqref{eq:qfunc_to_vfunc}. The forward recursion can be obtained exactly in the same way as in the proof of Theorem \ref{thm:ddp}.
\end{proof}

\begin{rem}
    Similar to Remark \ref{rem:equiv_ddp_pdp}, one can also show that the computation of the gradient term for the optimal control problem from active-set DDP-based method is equivalent to \cite[Theorem 1]{jin2021safe}.
\end{rem}

Note that Theorems \ref{thm:ddp} and \ref{thm:ddp_acset} consider the most general multi-stage constrained optimal control problem and they can be reduced to the unconstrained case by ignoring all the terms related to the constraints (i.e., $\inc, \eqc, \lmIn, \lmEq$ for Theorem \ref{thm:ddp}, and $\eqc\acset, \lmEq\acset$ for Theorem \ref{thm:ddp_acset}), which is detailed as follows.
\begin{cor}
    Suppose $\trajset$ is the optimal solution to the unconstrained optimal control problem \eqref{eq:prob_ineq_eq} with $\inc, \eqc := \vecbf{0}$, and $\bar{\qfunc}_{\ctrl\ctrl}$ is invertiable for $k=0, \dots, N-1$. 
    The derivative of solved trajectory w.r.t. the learning parameter $\dv{\trajset}{\para}$ can be obtained by iteratively updating  \eqref{eq:dctrldpara} and \eqref{eq:dstatdpara}
 for $k=0, \dots, N-1$, with $\frac{\dstat_{0}}{\dpara} = \vecbf{0}$, where $\hat{\qfunc}_{(\cdot)}$ is defined in \eqref{eq:hat_Q_for_gradient} with $\bar{\inc}, \bar{\eqc} := \vecbf{0}$.
\end{cor}

\begin{rem}
    The above result implies that the gradient is computed by performing a single backward-forward traditional DDP recursion on an augmented unconstrained system, i.e., \eqref{eq:prob_aug} with constraints being removed.  
    Similar to Remark \ref{rem:equiv_ddp_pdp}, one can also show that this recursion is equivalent to \cite[Lemma 5.2]{jin2020pontryagin} by viewing the Hamiltonian function $H$ and the dual variable $\gbf{\lambda}$ for the dynamics constraint defined in \cite{jin2021safe} as the cost-to-go $\qfunc := \pathcost + \vfunc^{+}$ and $\vfunc_{\stat}$, respectively.  
\end{rem}

\begin{rem}
\label{rem:barrierDDP}  
    Note that DDP-based gradient solver for constrained problems can be reduced to the one for unconstrained problems, then again the solver for constrained problems involves more terms (i.e., $\bar{\inc}, \bar{\eqc}$-related terms in \eqref{eq:hat_Q_for_gradient}) to deal with these constraints.
    To compute these terms, more symbolic evaluations are performed, which result in longer computational time than that for unconstrained problems. 
    However, it can be easily shown that \eqref{eq:hat_Q_for_gradient} is indeed the intermediate matrices for the unconstrained system with modified stage cost (i.e., $\bar{\pathcost}(\parastat_{k}, \ctrl_{k}) - \pert \vecbf{1}^{\top}\log(-\bar{\inc}) + 1/(2\pert)\|\bar{\eqc}\|^{2}$).
    Therefore, one can also solve the gradient for constrained problems by resorting to the solver for unconstrained problems with a modified objective function. 
    This is consistent with the idea of barrier method in optimization literature and we call this BarrierDDP-based gradient solver.
    In practice, as will be shown in numerical simulations in Section \ref{subsec:simu_pdp_ddp_comp}, the modification in the stage cost does not introduce much overhead for the symbolic evaluation of stage cost while saving significant overhead for that of constraints-related terms. 
\end{rem}

\begin{algorithm}
	\caption{DDP-based gradient solver}
	\label{alg:ipddp_grad_solver}
	\begin{algorithmic}[1]
		\Require system \eqref{eq:prob_ineq_eq}, optimal trajectory $\trajset$, parameter $\theta$, perturbation $\pert$ 
		\Ensure $\dv{\trajset}{\para}$	
        \State construct the augmented system \eqref{eq:prob_aug} \label{alg:irl.ddpgrad1}
        \State set $\bar{\vfunc}_{\parastat, N} = \bar{\termcost}_{\parastat}$, $\bar{\vfunc}_{\parastat\parastat, N} = \bar{\termcost}_{\parastat\parastat}$
        \For{$k = N-1, \dots, 0$}   \label{alg:irl.ddpgradbwd1}
        \State evaluate $\bar{\hat{\qfunc}}_{(\cdot)}$ using \eqref{eq:hat_Q_for_gradient}
        \State compute control gains in \eqref{eq:dctrldpara}
        \State update $\bar{\vfunc}_{\parastat}$, $\bar{\vfunc}_{\parastat\parastat}$ similarly as in \eqref{eq:qfunc_to_vfunc}
        \EndFor    \label{alg:irl.ddpgradbwd2}
        \State set $\frac{\dstat_{0}}{\dpara} = \vecbf{0}$ 
        \For{$k = 0, \dots, N-1$}   \label{alg:irl.ddpgradfwd1}
            \State update $\frac{\dctrl}{\dpara}$  according to \eqref{eq:dctrldpara} 
            \State update $\frac{\dstat^{+}}{\dpara}$ according to \eqref{eq:dstatdpara}
        \EndFor \label{alg:irl.ddpgradfwd2}
        \State collect $\dv{\trajset}{\para}$ \label{alg:irl.ddpgrad2}
	\end{algorithmic}
\end{algorithm}

We summarize our proposed DDP-based gradient solver in Algorithm \ref{alg:ipddp_grad_solver}, where the backward and forward recursions are detailed in lines \ref{alg:irl.ddpgradbwd1} to \ref{alg:irl.ddpgradbwd2}, lines \ref{alg:irl.ddpgradfwd1} to \ref{alg:irl.ddpgradfwd2}, respectively. 

\subsection{Open-loop IRL algorithm}

Equipped with the introduced trajectory solvers and gradient solvers, it is now ready to summarize the entire IRL algorithm with the open-loop loss, as seen in Algorithm \ref{alg:ol_irl}.
Note that Algorithm \ref{alg:ol_irl} only shows the case where interior-point DDP-based gradient solver is adopted, for the case of active-set DDP-based gradient solver, one can replace the involved (optimal) cost-to-go accordingly. 
Furthermore, if the trajectory was solved by any (interior-point, active-set, or traditional) DDP-based trajectory solver, the (optimal) cost-to-go computed in the last iteration of the trajectory solver can be saved and then can be reused in the backward recursion of the gradient solver.

\begin{algorithm}
	\caption{Open-loop IRL Algorithm}
	\label{alg:ol_irl}
	\begin{algorithmic}[1]
		\Require demonstrative trajectories $\myset{D}$, system \eqref{eq:prob_ineq_eq}, loss function $L^{\mathrm{ol}}$, initial parameter $\theta_{0}$, maximum iteration $t_{\max}$
		\Ensure $\theta$	
		\For{$t = 0, \dots, t_{\max}$}	
        \State call external solver, Algorithm \ref{alg:ipddp_traj_solver}, or active-set DDP-based trajectory solver to solve \eqref{eq:prob_ineq_eq} with $\para = \para_{t}$ (perhaps save some intermediate matrices related to the cost-to-go) 
        \State collect $\trajset$
        \State evaluate $\pdv{L^{\mathrm{ol}}}{\trajset}$ and $\pdv{L^{\mathrm{ol}}}{\para}$
        \State call Algorithm \ref{alg:ipddp_grad_solver} to obtain $\dv{\trajset}{\para}$
        \State update $\para_{t}$ according to \eqref{eq:para_upd}
        \State $t \gets t+1$
        \EndFor
	\end{algorithmic}
\end{algorithm}

\begin{rem}
    As shown in Algorithm \ref{alg:ol_irl}, the open-loop IRL algorithm is essentially a first-order gradient-descent algorithm to solve a generic bi-level optimization problem. 
    By following \cite[Theorem 2.1]{ghadimi2018approximation}, one can establish the global convergence result and iteration complexity of the full problem with assumptions on strong convexity and smoothness conditions for the functions in the lower-level optimization problem. 
    However, these conditions are too restrictive for commonly considered unconstrained infinite-horizon linear-quadratic regulator problems, let alone the multi-stage optimal control problem.
    Nevertheless, in practice, if one assumes that for each $\para \in \Theta$, the solution to the lower-level optimization problem always exists and is unique, along with smoothness conditions, the result of local convergence to a stationary point can be easily obtained.
\end{rem}

\section{Closed-loop IRL}
\label{sec:cl_irl}

\begin{figure}[!t]
	\centering
		\includegraphics[scale=0.5]{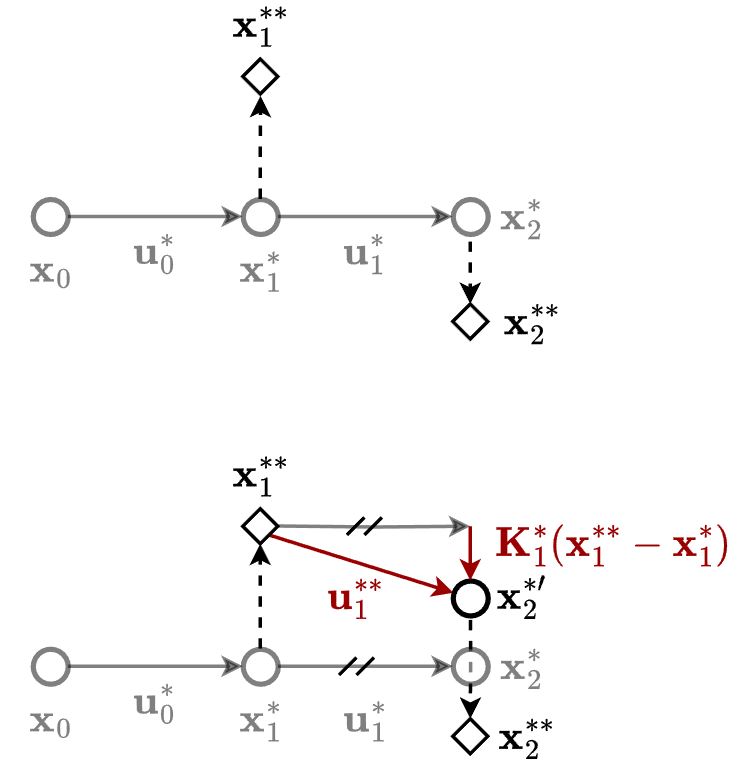}
	\caption{Illustration of collection of open-loop and closed-loop trajectories. 
                    The gray part denotes the nominal optimal trajectory under ideal environments. 
                    For the collection of the open-loop trajectory (top), it is implicitly assumed that the noise process (denoted by the dashed arrow) only affects the measurement afterward. 
                    On the contrary, for the collection of closed-loop trajectory, the next control input will take into consideration this noise and make a correction (denoted by the red arrow).}
	\label{fig:ol_cl}
\end{figure}

\begin{table*}[ht]\centering
\caption{Entire trajectories for the simple example \eqref{eq:prob_example}.}
\resizebox{\linewidth}{!}{\begin{tabular}{||c|| c c c c c||} 
 \hline
 & state at $k=0$ & control at $k=0$ & state at $k=1$ & control at $k=1$ & state at $k=2$ \\ [0.5ex] 
 \hline
   reproduced traj. $\trajset$  & $\stat_{0}$ & $-\frac{2\paraex+1}{2\paraex+3}\stat_{0}$ & $\frac{2}{2\paraex+3}\stat_{0}$ & $-\frac{1}{2\paraex+3}\stat_{0}$ & $\frac{1}{2\paraex+3}\stat_{0}$ \\ 
 \hline
  noise-free demo. $\trajset^{\ast}$  & $\stat_{0}$ & $-\frac{2\paraex^{\ast}+1}{2\paraex^{\ast}+3}\stat_{0}$ & $\frac{2}{2\paraex^{\ast}+3}\stat_{0}$ & $-\frac{1}{2\paraex^{\ast}+3}\stat_{0}$ & $\frac{1}{2\paraex^{\ast}+3}\stat_{0}$ \\ 
 \hline
  noisy demo.  $\trajset^{\ast\ast}$ & $\stat_{0}$ & $-\frac{2\paraex^{\ast}+1}{2\paraex^{\ast}+3}\stat_{0}$ & $\frac{2}{2\paraex^{\ast}+3}\stat_{0} +\noise_{1}$ & $-\frac{1}{2\paraex^{\ast}+3}\stat_{0}-\frac{1}{2}\noise_{1}$ & $\frac{1}{2\paraex^{\ast}+3}\stat_{0}+\frac{1}{2}\noise_{1}+\noise_{2}$ \\ 
 \hline
\end{tabular}}
\label{tab:example}
\end{table*}

In the previous section, we tackle the IRL problem with the open-loop loss $L^{\mathrm{ol}}$, and we can expect that $\para_{t} \to \para^{\ast}$ as $t \to \infty$ if $\para_{0}$ is at the vicinity of $\para^{\ast}$ for noise-free demonstrations. 
It is also expected that this type of least-square formulation can tolerate some noise in the collected demonstration signal. 
However, this formulation implicitly assumes that the noise only appears after the optimal trajectory is solved and executed precisely, or mathematically speaking, it adds some perturbations on the optimal demonstrations $\trajset_{\sampleset_{i}}(\para^{\ast})$ afterward, see Fig. \ref{fig:ol_cl}(a).
However, this is often not the case in the real data collection process, where the action is performed in a feedback manner to counter the uncertainty.

Let us first take a detour to consider the following simple example of an optimal control problem
\begin{equation}
\label{eq:prob_example}
\begin{aligned}
\min_{\myset{U}} &~~ \sum_{k = 0,1} 
\frac{1}{2}(\paraex \stat_{k}^{\top}\stat_{k} + \ctrl_{k}^{\top}\ctrl_{k}) + \frac{1}{2}\stat_{2}^{\top}\stat_{2} \\
\mathrm{s.t.} &~~ \stat_{k+1} = \stat_{k} + \ctrl_{k}, \stat_{0}~\mathrm{is~given},
\end{aligned}
\end{equation}
where $\paraex$ is the parameter to be learned. Solving the above problem, one can find that the optimal feedback policy is given by $\ctrl_{0} = -\frac{2\paraex+1}{2\paraex+3}\stat_{0}, \ctrl_{1} = -\frac{1}{2}\stat_{1}$. 
Therefore, given an estimated parameter $\paraex$ (resp. the true parameter $\paraex^{\ast}$), the reproduced trajectory $\trajset$ (resp. noise-free demonstration $\trajset^{\ast}$) can be explicitly expressed as the second (resp. third) row in Table \ref{tab:example}. 
However, if there is some process and/or measurement noise $\noise_{k}$ (see the fourth and sixth column of the last row in Table \ref{tab:example}), the control input at $k=1$ will change correspondingly (see the fifth column of the last row in Table \ref{tab:example}) and the noisy demonstration $\trajset^{\ast\ast}$ can be obtained. In this case, the open-loop loss $L^{\mathrm{ol}}$ defined in \eqref{eq:loss_ol} can be written as 
\begin{equation*}
\begin{aligned}
   L^{\mathrm{ol}} & = \underbrace{\vecbf{0} + \|\tilde{\paraex}\stat_{0} + \noise_{1}\|^{2}  + \|\frac{1}{2}(\tilde{\paraex}\stat_{0} + \noise_{1}\|^{2}}_{\sum_{k=0,1,2} \|\stat_{k} - \stat_{k}^{\ast\ast}\|_{2}^{2} } \\
   & \quad \quad \quad \quad  + \underbrace{\|\tilde{\paraex}\stat_{0}\|^{2} + \|\frac{1}{2}(\tilde{\paraex}\stat_{0} + \noise_{1}) + \noise_{2}\|^{2}}_{\sum_{k=0,1} \|\ctrl_{k} - \ctrl_{k}^{\ast\ast}\|_{2}^{2} },
\end{aligned}
\end{equation*}
where $\tilde{\paraex} := \frac{2\paraex+1}{2\paraex+3} - \frac{2\paraex^{\ast}+1}{2\paraex^{\ast}+3}$.
By some mathematical operations, one can find that the optimal solution for $L^{\mathrm{ol}}$ is given by $-\frac{3\noise_{1} + \noise_{2}}{5\stat_{0}}$, which means that $\tilde{\paraex} \to 0$ only if the noise is of zero-mean and state-independent and one has collected a sufficiently large amount of data. In other words, nonzero-mean or state-dependent noise, or limited size of data will lead to a biased estimation of $\paraex^{\ast}$.
Furthermore, for either a longer horizon $N>2$ or more general linear system dynamics, $L^{\mathrm{ol}}$ involves higher-order terms of $\paraex$ and one cannot solve the stationary point analytically from $\dv{L^{\mathrm{ol}}}{\paraex} = 0$. However, one has that $\dv{L^{\mathrm{ol}}}{\paraex}|_{\paraex = \paraex^{\ast}}$ is again a linear combination of noise $\{\noise_{k}\}_{k=1}^{N}$, which implies that zero-mean and state-independent noise and a sufficiently large amount of data are necessary conditions for unbiased estimation of $\paraex^{\ast}$.

For a general nonlinear optimal control problem, in addition to the numerically computed nominal optimal open-loop input, an additional feedback term should be implemented to correct the deviation $(\stat_{1}^{\ast\ast} - \stat_{1}^{\ast})$, where $\stat_{k}^{\ast\ast}$ is the observed current state and is not necessarily equal to the ideal current state $\stat_{k}^{\ast}$ due to process noise in $\dyn$, see Fig. \ref{fig:ol_cl}(b) for illustration. 
In the following, we assume that the demonstrations are collected from a closed-loop controller solved by a DDP-based trajectory solver, i.e., instead of having $\ctrlset$ as the output, it additionally records the feedback gain $\gainK_{k}^{\ast}$. 
During the physical roll-out, the control input is recomputed as\footnote{Specifically, for the infinite-horizon LQR problem, this means the optimal gain is used for generating the demonstrations in a feedback manner.} 
\begin{equation}
\label{eq:cl_ctrl}
    \ctrl_{k}^{\ast\ast} = \ctrl_{k}^{\ast} + \gainK_{k}^{\ast} (\stat_{k}^{\ast\ast} - \stat_{k}^{\ast}).
\end{equation}
Denote the collected demonstration as $\trajset_{\sampleset}^{\ast\ast}$.
In this case, if we use the open-loop loss $L^{\mathrm{ol}}$ for this type of noisy demonstration, $\para_{t}$ does not converge to $\para^{\ast}$ as $t \to \infty$ since $\dv{L^{\mathrm{ol}}}{\para}|_{\para = \para^{\ast}}$ is a nonlinear function of noise $\{\noise_{k}\}_{k=1}^{N}$ and is not equal to $\vecbf{0}$ almost surely (This will also be validated by numerical simulations in Section \ref{subsec:simu_adv}).
To tackle this, we propose a new IRL problem:
\begin{equation}
\label{eq:bilevel_prob_cl}
\begin{aligned}
\min_{\para \in \Theta} &~~ L^{\mathrm{cl}} := \sum_{k \in \sampleset} \| \underbrace{\hat{\qfunc}_{\ctrl} + \hat{\qfunc}_{\ctrl\stat}(\stat^{\ast\ast} - \stat) + \hat{\qfunc}_{\ctrl\ctrl}(\ctrl^{\ast\ast} - \ctrl)}_{=:\rescl}\|^{2}  \\
\mathrm{s.t.} &~~ \hat{\qfunc}_{(\cdot)}, \trajset~\mathrm{with}~\ctrlset~\mathrm{being~solved~from}~\eqref{eq:prob_ineq_eq}.
\end{aligned}
\end{equation}
We refer to $L^{\mathrm{cl}}$ as the closed-loop loss since it is motivated by the closed-loop controller \eqref{eq:cl_ctrl} and captures the feedback nature.
In particular, firstly, notice that $\hat{\qfunc}_{\ctrl} \equiv \vecbf{0}$, since $\trajset$ is the optimal trajectory of system \eqref{eq:prob_ineq_eq}\footnote{It is still kept in \eqref{eq:bilevel_prob_cl} for the subsequent content of specialization.}. 
Secondly, $\rescl$ recovers \eqref{eq:cl_ctrl} if $\hat{\qfunc}_{\ctrl\ctrl}^{-1}$ is multiplied in each term and all the quantities related to $\trajset$ are replaced by the optimal trajectory $\trajset(\para^{\ast})$\footnote{An alternate form of the residual $\rescl' := (\ctrl_{k}^{\ast\ast} - \ctrl_{k}) + \gainK_{k} (\stat_{k}^{\ast\ast} - \stat_{k})$ may be more obvious to understand the design, while it breaks the tie with the subsequent content of specialization and we do not present here. Nevertheless, it is still applicable for closed-loop IRL of general nonlinear problems and shares nearly the same subsequent algorithmic computation of gradients.}.
By the second point, it can be seen that $\para^{\ast}$ is a global minimum for \eqref{eq:bilevel_prob_cl}. 
Note that currently $\rescl$ only relates the residuals of the current input to the current state, while it is still possible to consider the opposite direction, i.e., including the dynamics residual $\| \stat^{\ast\ast +} - \dyn(\stat^{\ast\ast}, \ctrl^{\ast\ast}; \para)\|_{2}^{2}$, which relates the next state to current input. 
However, due to its least-square form, this residual only works well for additive process noise but not for other types of noise. 
Nevertheless, the addition only brings a marginal overhead in terms of computation (as its required gradient term has already been computed by Algorithm \ref{alg:ipddp_traj_solver}). Alternatively, one can use this residual to initialize the parameter to be estimated.
On the other hand, if the collected demonstrations are generated in a closed-loop manner other than \eqref{eq:cl_ctrl}, e.g., model predictive control, the proposed loss can be interpreted as finding an affine time-varying feedback controller which matches the closed-loop demonstrations. 

To solve the new IRL problem \eqref{eq:bilevel_prob_cl}, one can resort to the gradient descent method similar to \eqref{eq:para_upd}: 
{\small \begin{equation*}
\begin{aligned}
        \dv{L^{\mathrm{cl}}}{\para} & = \sum_{k \in \sampleset} (\dv{\rescl}{\para})^{\top} \rescl \\
        & = \sum_{k \in \sampleset}  (\bar{\hat{\qfunc}}_{\ctrl\stat} \frac{\dstat_{k}}{\dpara}   + [(\stat^{\ast\ast} - \stat_{k})^{\top} \otimes \matbf{I}_{m}]\dvexd{\bar{\hat{\qfunc}}_{\ctrl\stat}}{\para} \\
        & \quad \quad + \bar{\hat{\qfunc}}_{\ctrl\ctrl} \frac{\dctrl_{k}}{\dpara} + [(\ctrl^{\ast\ast} - \ctrl_{k})^{\top} \otimes \matbf{I}_{m}]\dvexd{\bar{\hat{\qfunc}}_{\ctrl\ctrl}}{\para} )^{\top}\rescl \\
        & = \sum_{k \in \sampleset}  (\bar{\hat{\qfunc}}_{\ctrl\para} + [(\stat^{\ast\ast} - \stat_{k})^{\top} \otimes \matbf{I}_{m}]\dvexd{\bar{\hat{\qfunc}}_{\ctrl\stat}}{\para} \\
        & \quad \quad  \quad \quad \quad + [(\ctrl^{\ast\ast} - \ctrl_{k})^{\top} \otimes \matbf{I}_{m}]\dvexd{\bar{\hat{\qfunc}}_{\ctrl\ctrl}}{\para} )^{\top}\rescl. 
\end{aligned}
\end{equation*}}
In the above, the first equality is from the fact that both the trajectory $\trajset$ itself and the intermediate matrices $\bar{\hat{\qfunc}}_{(\cdot)}$ (which are evaluated at the current trajectory $\trajset$) are functions of learning parameter $\para$.
The second equality results from \eqref{eq:dctrldpara} and implies that it is not necessary to compute $\frac{\dctrl_{k}}{\dpara}$ and $\frac{\dstat_{k}}{\dpara}$ explicitly since the required term $\bar{\hat{\qfunc}}_{\ctrl\para}$ has been computed as part of $\bar{\hat{\qfunc}}_{\ctrl\parastat}$ in \eqref{eq:hat_Q_for_gradient}. 
Nonetheless, one can find that this term is tightly related to the first-order derivative of the trajectory w.r.t. parameter. 
However, notice that the above gradient also involves $\dvexd{\bar{\hat{\qfunc}}_{(\cdot)}}{\para}$, which is the gradient of the intermediate matrices w.r.t. the learning parameter and has not been obtained in Section \ref{subsec:backward}. 
Intuitively speaking, this relates to the second-order derivatives of the trajectory w.r.t. parameter.
This is as expected since in the open-loop loss formulation, one tries to find a parameter to match the solved trajectory with the demonstrations, while in the closed-loop one, one aims to find a parameter to match their variations in the differential sense. 

In order to compute $\dvexd{\bar{\hat{\qfunc}}_{(\cdot)}}{\para}$, we differentiate $\hat{\qfunc}_{\ctrl\ctrl}$ in \eqref{eq:hat_Q_for_gradient}\footnote{Note that the subsequent derivation is based on the interior-point DDP-based gradient solver, while a similar derivation can be easily performed on the active-set DDP-based gradient solver.} w.r.t. $\para$, i.e.,
\begin{equation}
\label{eq:hat_Quu_grad}
    \begin{aligned}
        \dvexd{\bar{\hat{\qfunc}}_{\ctrl\ctrl}}{\para} &= \dvexd{ \{\bar{\pathcost}_{\ctrl\ctrl}  + \bar{\dyn}_{\ctrl}^{\top}\bar{\vfunc}_{\parastat\parastat}^{+}\bar{\dyn}_{\ctrl} + \bar{\vfunc}_{\parastat}^{+} \odot \bar{\dyn}_{\ctrl\ctrl} \}}{\para} \\
        &\quad + \dvexd{\{\pert \bar{\inc}_{\ctrl}^{\top} [\Diag(\bar{\inc})]^{-2} \bar{\inc}_{\ctrl} - \pert ([\Diag(\bar{\inc})]^{-1} \vecbf{1}) \odot \bar{\inc}_{\ctrl\ctrl} \}}{\para} \\
        &\quad + \dvexd{ \{\pert^{-1} \bar{\eqc}_{\ctrl}^{\top}\bar{\eqc}_{\ctrl} + \pert^{-1}\bar{\eqc} \odot \bar{\eqc}_{\ctrl\ctrl} \}}{\para},\\ 
    \end{aligned}
\end{equation}
where the second and third rows denote the terms related to inequality and equality constraints. 
For the sake of clarity, we only show the derivation of the first row. The first term reads 
\begin{equation*}
\begin{aligned}
        \dvexd{\bar{\pathcost}_{\ctrl\ctrl}}{\para} & = \pvexd{\bar{\pathcost}_{\ctrl\ctrl}}{\para} + \pvexd{\bar{\pathcost}_{\ctrl\ctrl}}{\stat} \frac{\dstat}{\dpara} + \pvexd{\bar{\pathcost}_{\ctrl\ctrl}}{\ctrl} \frac{\dctrl}{\dpara},
\end{aligned}
\end{equation*}
which comes from the fact that $\bar{\pathcost}_{\ctrl\ctrl}$ is the function of $(\para, \stat, \ctrl)$. 
The second term can be obtained by using the matrix calculus \cite{magnus2019matrix}:
\begin{equation*}
\begin{aligned}
        \dvexd{\{\bar{\dyn}_{\ctrl}^{\top}\bar{\vfunc}_{\parastat\parastat}^{+}\bar{\dyn}_{\ctrl} \}}{\para} & = (\cmMat^{m,m} + \matbf{I}_{m^2})(\matbf{I}_{m} \otimes \bar{\dyn}_{\ctrl}^{\top}\bar{\vfunc}_{\parastat\parastat}^{+}) \dvexd{\bar{\dyn}_{\ctrl}}{\para} \\
        & \quad + (\bar{\dyn}_{\ctrl}^{\top} \otimes \bar{\dyn}_{\ctrl}^{\top}) \dvexd{\bar{\vfunc}_{\parastat\parastat}^{+}}{\para},
\end{aligned}
\end{equation*}
where the term $\dvexd{(\cdot)}{\para}$ involved can be obtained similarly as in the above equation.
For the third term, by the definition of tensor contraction, one has
\begin{equation*}
    \begin{aligned}
        \dvexd{\{ \bar{\vfunc}_{\parastat}^{+} \odot \bar{\dyn}_{\ctrl\ctrl} \}}{\para} & = \sum_{i=1,\dots,n}\dvexd{\{ [\bar{\vfunc}_{\parastat}^{+}]_{i} [\bar{\dyn}_{\ctrl\ctrl}]_{i} \}}{\para} \\
        & = \sum_{i=1,\dots,n}\dvexd{ [\bar{\vfunc}_{\parastat}^{+}]_{i} }{\para} [\bar{\dyn}_{\ctrl\ctrl}]_{i} + [\bar{\vfunc}_{\parastat}^{+}]_{i} \dvexd{ [\bar{\dyn}_{\ctrl\ctrl}]_{i}}{\para}.
    \end{aligned}
\end{equation*}
The second and third rows of \eqref{eq:hat_Quu_grad} and $\dvexd{\bar{\hat{\qfunc}}_{\ctrl\stat}}{\para}$ can be obtained similarly by following the above derivations.

Note that in order to accelerate the learning process, we use the Levenberg–Marquardt algorithm, i.e., updating the parameter using the following rule:
\begin{equation}
\label{eq:para_upd_lm_opt}
    [\recMat^{\top} \recMat +  \eta' \matbf{I} ] \dpara = \recMat^{\top} \rescl^{\sampleset}
\end{equation}
where $\eta'$ is a (non-negative) damping factor adjusted at each iteration, $\recMat := \vex(\{ \dv{\rescl}{\para} \}_{k \in \sampleset})$ and $\rescl^{\sampleset} := \vex( \{ \rescl \}_{k \in \sampleset} )$ are the concatenated gradient and residual terms for the closed-loop loss. 

We summarize the closed-loop IRL algorithm in Algorithm \ref{alg:cl_irl}. In line \ref{alg:irl.ddpgradcall}, Algorithm \ref{alg:ipddp_grad_solver} is called to obtain the intermediate matrices as well as the first-order gradient for both computing the loss and preparing for calculating $\dvexd{\bar{\hat{\qfunc}}_{(\cdot)}}{\para}$. Lines \ref{alg:irl.ddpgrad2bwd1} to \ref{alg:irl.ddpgrad2bwd2} detail the backward iteration for computing $\dvexd{\bar{\hat{\qfunc}}_{(\cdot)}}{\para}$.  
\begin{algorithm}
	\caption{Closed-loop IRL Algorithm}
	\label{alg:cl_irl}
	\begin{algorithmic}[1]
		\Require demonstrative trajectories $\myset{D}$, system \eqref{eq:prob_ineq_eq}, loss function $L^{\mathrm{cl}}$, initial parameter $\theta_{0}$, maximum iteration $t_{\max}$
		\Ensure $\theta$	
		\For{$t = 0, \dots, t_{\max}$}	
        \State call external solver, Algorithm \ref{alg:ipddp_traj_solver}, or active-set DDP-based trajectory solver to solve \eqref{eq:prob_ineq_eq} with $\para = \para_{t}$ (perhaps save some intermediate matrices related to the cost-to-go) 
        \State collect $\trajset$
        \State call Algorithm \ref{alg:ipddp_grad_solver} to obtain $\dv{\trajset}{\para}$ and save $\bar{\hat{\qfunc}}_{(\cdot)}$ \label{alg:irl.ddpgradcall}
        \State evaluate $\rescl$
        \State set $\dvexd{\bar{\vfunc}_{\parastat, N}}{\para} = \dvexd{\bar{\termcost}_{\parastat}}{\para}$, $\dvexd{\bar{\vfunc}_{\parastat\parastat, N}}{\para} = \dvexd{\bar{\termcost}_{\parastat\parastat}}{\para}$
        \For{$k = N-1, \dots, 0$}   \label{alg:irl.ddpgrad2bwd1} 
        \State evaluate $\dvexd{\bar{\hat{\qfunc}}_{(\cdot)}}{\para}$ using \eqref{eq:hat_Quu_grad}
        \State update $\dvexd{\bar{\vfunc}_{\parastat}}{\para}$, $\dvexd{\bar{\vfunc}_{\parastat\parastat}}{\para}$ similarly as in \eqref{eq:qfunc_to_vfunc}
        \EndFor    \label{alg:irl.ddpgrad2bwd2}
        \State collect $\dvexd{\bar{\hat{\qfunc}}_{(\cdot)}}{\para}$ to compute $\recMat$
        \State update $\para_{t}$ according to \eqref{eq:para_upd_lm_opt}
        \State $t \gets t+1$
        \EndFor
	\end{algorithmic}
\end{algorithm}

The above content only details the computational aspects of our proposed algorithm. Next, we aim to provide some theoretical characterization of the condition for recoverability, i.e., under which conditions the algorithm can find $\para^{\ast}$. 
\begin{thm}
\label{thm:rank_cond}
Suppose that the level set $\{\para \mid L^{\mathrm{cl}}(\para) \leq L^{\mathrm{cl}}(\para^{\ast})\}$ is bounded and that the residual function $\rescl$ is Lipschitz
continuously differentiable in a neighborhood of $L^{\mathrm{cl}}$. Assume that for each $t$, the approximate
solution $\dpara$ of \eqref{eq:para_upd_lm_opt} satisfies the inequality 
$$L^{\mathrm{cl}}(\para_{t}) - L^{\mathrm{cl}}(\para_{t}+\dpara) \geq c_{1} \|\recMat^{\top} \rescl^{\sampleset}\| \min(\Delta_{t}, \frac{\|\recMat^{\top} \rescl^{\sampleset}\|}{\|\recMat^{\top} \recMat\|}),$$ for some positive constant $c_{1}$, and in addition $\|\dpara\| \leq c_{2} \Delta_{t}$ for some constant $c_{2} \geq 1$, where $ \Delta_{t}$ is the trust-region radius in its counterpart trust-region method such that $\eta'(\dpara - \Delta_{t}) = 0$, then Algorithm \ref{alg:cl_irl} converges to the stationary point, i.e., $\lim_{t \to \infty} \dv{L^{\mathrm{cl}}}{\para} = \lim_{t \to \infty} \recMat^{\top} \rescl^{\sampleset} = \vecbf{0}$. Furthermore,  the learning parameter $\para^{\ast}$ can be fully recovered only if $\Rank(\recMat) = \size{\para}$. 
\end{thm}
\begin{proof}
    The first statement follows from \cite[Theorem 10.3]{nocedal1999numerical}. 
    The second statement follows from the fact that $\para_{t}+\vecbf{c}$ with $\vecbf{c}$ being a null vector of  $\recMat$ still satisfies \eqref{eq:para_upd_lm_opt} if $\Rank(\recMat) < \size{\para}$. 
\end{proof}

If LQR problem is considered, following the definition of the residual term in \eqref{eq:bilevel_prob_cl}, 
    \begin{equation}
    \begin{aligned}
        \rescl_{\mathrm{lqr}} & =  \hat{\qfunc}_{\ctrl} + \hat{\qfunc}_{\ctrl\stat}(\stat^{\ast\ast} - \stat) + \hat{\qfunc}_{\ctrl\ctrl}(\ctrl^{\ast\ast} - \ctrl) \\
        & = \hat{\qfunc}_{\ctrl\stat}\stat^{\ast\ast} + \hat{\qfunc}_{\ctrl\ctrl}\ctrl^{\ast\ast},
    \end{aligned}
    \end{equation}
where the second equality follows from the optimality condition of LQR. 
Defining $$\recMat_{\mathrm{lqr}} := [(\stat^{\ast\ast})^{\top} \otimes \matbf{I}_{m}]\dvexd{\hat{\qfunc}_{\ctrl\stat}}{\para}  + [(\ctrl^{\ast\ast})^{\top} \otimes \matbf{I}_{m}]\dvexd{\hat{\qfunc}_{\ctrl\ctrl}}{\para},$$ one can have the following result.

\begin{cor}
    The learning parameter $\para^{\ast}$ for LQR can be fully recovered only if $\Rank(\recMat_{\mathrm{lqr}}) = \size{\para}$.
\end{cor}

\begin{rem}
    Note that a two-step strategy has been proposed in \cite{xue2021inverse}, where a gain matrix $\gainK^{\ast\ast}$ is firstly solved from a least square problem \cite[Eq. (15)]{xue2021inverse} and a bi-level problem with a cost function $\Tr{(\gainK - \gainK^{\ast\ast})^{\top}(\gainK - \gainK^{\ast\ast})}$ is then iteratively solved. 
    It should be noted that Algorithm \ref{alg:cl_irl} can also be applied to this scheme if we use 
    \begin{equation*}
        \begin{aligned}
            \dvexd{\gainK}{\para} & = - \dvexd{\hat{\qfunc}_{\ctrl\ctrl}^{-1} \hat{\qfunc}_{\ctrl\stat}}{\para} \\
            & = - (\matbf{I}_{n} \otimes \hat{\qfunc}_{\ctrl\ctrl}^{-1})[(\gainK^{\top} \otimes \matbf{I}_{m}) \dvexd{\hat{\qfunc}_{\ctrl\ctrl}}{\para} + \dvexd{\hat{\qfunc}_{\ctrl\stat}}{\para}],
        \end{aligned}
    \end{equation*}
    which requires the same gradient terms derived in \eqref{eq:hat_Quu_grad}. Moreover, our presented algorithm is applicable to IRL of general nonlinear systems subject to constraints.
\end{rem}

Note that due to the nature of non-linearity, the above rank condition depends on collected demonstrations, the solved trajectory, and the current parameter. 
In the following, we shall show that under specific assumptions, $\recMat$ is linear in $\para$ and each element only depends on collected demonstrations.
Before that, we present an assumption and some definitions which will be used.
\begin{ass}
\label{ass:ioc}
    \begin{enumerate}\setlength{\parindent}{0pt}
        \item the termination condition for solving \eqref{eq:prob_ineq_eq} is set as $\trajset = \trajset_{\sampleset}^{\ast\ast}$; \label{cor:recov.ass1}
        \item the demonstrations satisfy the interior-point min-max Bellman equation \eqref{eq:minmax_bellman_ineq_eq} with perturbation $\pert$; \label{cor:recov.ass2} 
        \item the stage cost is linearly parameterized by $\para$, i.e., $\pathcost = \para^{\top}\gbf{\phi}(\stat,\ctrl)$; \label{cor:recov.ass3} 
        \item the terminal cost $\termcost$, dynamics $\dyn$, and constraints $\inc, \eqc$ are independent of $\para$ and known. \label{cor:recov.ass4}  
    \end{enumerate}
\end{ass}

Define 
\begin{equation*}
    \begin{aligned}
        \vecbf{c}_{(\cdot)} &:= \pert \inc_{(\cdot)}^{\top} [\Diag(\inc)]^{-1}\vecbf{1} + \pert^{-1} \eqc_{(\cdot)}^{\top} \eqc, \\
        \vfunc_{\stat,1:m} &:= \col( \{ \vfunc_{\stat}\}_{k=1}^{m} ), \\
        \gbf{\phi}_{(\cdot),1:m}^{\top} &:= \col( \{ \gbf{\phi}_{(\cdot)}^{\top}\}_{k=1}^{m} ), \\
        \vecbf{c}_{(\cdot), 1:m} &:= \col( \{ \vecbf{c}_{(\cdot)}\}_{k=1}^{m} ) , (\cdot) \in \{\stat, \ctrl\},\\
        \matbf{B}_{0:m} & := \Diag( \{ \dyn_{\ctrl}^{\top} \}_{k=0}^{m} ), \\
        \matbf{E}_{m+1} & := [\matbf{0}, \dots, \matbf{I}]^{\top} \in \mathbb{R}^{(m+1)n_{\stat} \times n_{\stat}},
    \end{aligned}
\end{equation*}
    \begin{equation*}
           \matbf{A}_{1:m} := \begin{bmatrix}
    \matbf{I} & - \dyn_{\stat,1}^{\top} & & \\
     & \matbf{I} & \ddots & \\
    & & \ddots & - \dyn_{\stat,m}^{\top} \\
    \matbf{0} & &  & \matbf{I}
    \end{bmatrix}. 
    \end{equation*}
Furthermore, define 
\begin{equation}
\label{eq:recMat}
    \begin{aligned}
        \recMat_{\mathrm{lin},1} &:= \gbf{\phi}_{\ctrl,1:m}^{\top} - \matbf{B}_{0:m} \matbf{A}_{1:m}^{-1} \gbf{\phi}_{\stat,1:m}^{\top}, \\
        \recMat_{\mathrm{lin},2} &:= \matbf{B}_{0:m} \matbf{A}_{1:m}^{-1}\matbf{E}_{m+1}, \\
        \recMat_{\mathrm{lin},3} &:= - \matbf{B}_{0:m} \matbf{A}_{1:m}^{-1}\vecbf{c}_{\stat, 1:m} + \vecbf{c}_{\ctrl, 1:m}, \\
        \recMat_{\mathrm{lin},1:2} &:= [\recMat_{\mathrm{lin},1}, \recMat_{\mathrm{lin},2}]. 
    \end{aligned}
\end{equation}

\begin{cor}
\label{cor:recov}
Under Assumption \ref{ass:ioc} and $\recMat_{\mathrm{lin},3} \neq \vecbf{0}$, if $\lim_{\mu \to 0} \Rank(\recMat_{\mathrm{lin},1:2}) = \size{\para}+\size{\stat}$, then the learning parameter $\para^{\ast}$ can be recovered from the demonstration as
\begin{equation*}
    \begin{aligned}
        \para^{\ast} & = [\lim_{\mu \to 0} \arg\min_{[\para^{\top}, \vfunc_{\stat, m+2}^{\top}]^{\top}} \|\rescl^{\sampleset}\|^{2}]_{1:\size{\para}} \\
        & =  [\lim_{\mu \to 0} - (\recMat_{\mathrm{lin},1:2}^{\top}\recMat_{\mathrm{lin},1:2})^{-1}\recMat_{\mathrm{lin},1:2}^{\top}\recMat_{\mathrm{lin},3}]_{1:\size{\para}}.
    \end{aligned}
\end{equation*}
\end{cor}
\begin{proof}
    By Assumption \ref{ass:ioc}-1), it follows from \eqref{eq:minmax_bellman_ineq_eq} that 
    \begin{equation}
    \label{eq:vfunc_lin}
    \begin{aligned}
        \vfunc_{\stat}(\stat^{\ast\ast}) & = \hat{\qfunc}_{\stat}(\stat^{\ast\ast}, \ctrl^{\ast\ast}) \\
        & = \pathcost_{\stat} + \dyn_{\stat}^{\top}\vfunc_{\stat}^{+} + \vecbf{c}_{\stat}\\
        & = \gbf{\phi}_{\stat}^{\top} \para + \dyn_{\stat}^{\top}\vfunc_{\stat}^{+} + \vecbf{c}_{\stat},
    \end{aligned}
    \end{equation}
    where the last equality results from Assumption \ref{ass:ioc}-3).
    Furthermore, it can be easily found that $\vfunc_{\stat}$ is always linear in $\para$.

    Stacking \eqref{eq:vfunc_lin} for $k = 0, \dots, m$, one can obtain
    \begin{equation*}
        \vfunc_{\stat, 1:m}
        = \gbf{\phi}_{\stat,1:m}^{\top} \para + \Diag( \{ \dyn_{\stat}^{\top} \}_{k=1}^{m} ) \vfunc_{\stat,2:m+1} + \vecbf{c}_{\stat,1:m}.
    \end{equation*}
    By some mathematical manipulations, one has
    \begin{equation*}
        \matbf{A}_{1:m} \vfunc_{\stat,1:m+1} + \gbf{\phi}_{\stat,1:m}^{\top} \para - \matbf{E}_{m+1} \vfunc_{\stat, m+2} + \vecbf{c}_{\stat, 1:m} = \vecbf{0}.
    \end{equation*}
    It follows from Assumption \ref{ass:ioc}-1) that
    \begin{equation}
    \label{eq:qfunc_lin}
        \begin{aligned}
            \rescl^{\sampleset} 
            & = \col(\{ \hat{\qfunc}_{\ctrl}(\stat^{\ast\ast}, \ctrl^{\ast\ast} ) \}_{k=0}^{m} ) \\
            & = \col(\{\gbf{\phi}_{\ctrl}^{\top} \para + \dyn_{\ctrl}^{\top}\vfunc_{\stat}^{+} + \vecbf{c}_{\ctrl} ) \}_{k=0}^{m} ) \\
            & = \gbf{\phi}_{\ctrl,1:m}^{\top}  \para + \matbf{B}_{0:m} \vfunc_{\stat, 1:m+1} + \vecbf{c}_{\ctrl, 1:m} \\
            & =  \recMat_{\mathrm{lin},1}  \para +
            \recMat_{\mathrm{lin},2} \vfunc_{\stat, m+2} + \recMat_{\mathrm{lin},3}.
             & \quad 
        \end{aligned}
    \end{equation}
    If $\rescl^{\sampleset} = \vecbf{0}$ and $\Rank(\recMat_{\mathrm{lin},1:2}) = \size{\para}+\size{\stat}$, one has $[\para^{\top}, \vfunc_{\stat, m+2}^{\top}]^{\top} =  - (\recMat_{\mathrm{lin},1:2}^{\top}\recMat_{\mathrm{lin},1:2})^{-1}\recMat_{\mathrm{lin},1:2}^{\top}\recMat_{\mathrm{lin},3}$.
    As $\pert \to 0$, \eqref{eq:vfunc_lin} and \eqref{eq:qfunc_lin} reduce to the non-perturbed version of Bellman principle of optimality differentiated w.r.t. the state and control, respectively, which are the equilibrium conditions for the constrained IOC problem. 
\end{proof}

It has been shown that the above rank condition only depends on the collected demonstrations and this property resembles that in \cite{molloy2018finite,jin2021inverse}. 
However, due to the introduction of constraints, the rank condition is quite different. 
Moreover, unlike \cite{molloy2020online} where only control constraints can be considered, our method can deal with general nonlinear constraints. 

\section{Numerical Experiments}
\label{sec:simu}
In this section, we first present several examples to validate the equivalence between our proposed DDP-based methods and PDP-based methods.
Then, we apply our proposed closed-loop IRL algorithm on these examples to show its advantage over open-loop IRL. 
Also, we provide an example to demonstrate the proposed recoverability conditions on both the general IRL problem and the specialized constrained inverse optimal control problem.

\subsection{System settings}
For simulations, we consider the following four systems of different dimensions (complexities), which have been commonly used in the literature \cite{amos2018differentiable, pinosky2023hybrid, lutter2023combining,plancher2017constrained,xie2017differential,jin2020pontryagin}:

\noindent
\paragraph{Cartpole} the system dynamics is given by
\begin{equation*}
\label{eq:cartpole_dyn}
\begin{aligned}
\ddot{x} & = (f+m_{\mathrm{p}} \sin(q)(l\dot{q}^{2}+g\cos(q)))/b, \\
\ddot{q} & = (-f \cos(q) - m_{\mathrm{p}} l \dot{q}^{2} \cos(q) \sin(q) \\
		& \quad \quad \quad  - (m_{\mathrm{c}} + m_{\mathrm{p}}) g \sin(q) )/(lb), \\
\end{aligned}
\end{equation*}
where $m_{\mathrm{c}}$, $m_{\mathrm{p}}$ are the masses of cart and pole, $l$ is the pole length, and $b = m_{\mathrm{c}} + m_{\mathrm{p}} \sin[2](q)$. 
The state vector $\stat$ is defined as $\stat := [x, \dot{x}, q, \dot{q}]^{\top}$, where
$x, \dot{x}$ denote the horizontal position and velocity of the cart, and $q, \dot{q}$ denote the angle and angular velocity of the pole.
The control input $\ctrl$ is the force $f$ applied to the cart. 
The control task is to drive the system to a prescribed desired state at $\stat_{\mathrm{d}} = [0, 0, \pi, 0]^{\top}$ and hence we consider the following stage and terminal costs:
\begin{equation*}
\label{eq:cartpole_cost}
\begin{aligned}
\pathcost & :=  (\stat - \stat_{\mathrm{d}})^{\top} \Diag(\para_{\stat}) (\stat - \stat_{\mathrm{d}}) + \para_{\ctrl} \|\ctrl\|^{2}, \\
\termcost & :=  (\stat - \stat_{\mathrm{d}})^{\top} \Diag(\para_{\stat}) (\stat - \stat_{\mathrm{d}}),
\end{aligned}
\end{equation*}
where $\para_{\stat}, \para_{\ctrl}$ denote the weights for state and control, and we set $\para_{\ctrl} = 0.1$ to avoid ambiguity.
In addition, we set norm-bounded constraints for both state and control vectors, i.e., $|x| \leq x_{\mathrm{ub}}$ and $|f| \leq f_{\mathrm{ub}}$, where $x_{\mathrm{ub}}$ and $f_{\mathrm{ub}}$ are the upper bounds for the cart position and the applied force, respectively.
We set $\para = \{ m_{\mathrm{c}}, m_{\mathrm{p}}, l, \para_{\stat}, x_{\mathrm{ub}}, f_{\mathrm{ub}}\}$ as the parameter to be learned. 

\noindent
\paragraph{Quadrotor} the system dynamics is given by
\begin{equation*}
\label{eq:quadrotor_dyn}
\begin{aligned}
\dot{\vecbf{p}}_{\mathrm{w}} &= \vecbf{v}_{\mathrm{w}}, & \dot{\vecbf{v}}_{\mathrm{w}} &= T_{\mathrm{b}} \matbf{R} \vecbf{e}_{z}/m - g \vecbf{e}_{z}, \\
\dot{\vecbf{q}}_{\mathrm{b}} &= \vecbf{q}_{\mathrm{b}} \oplus [0, \gbf{\omega}_{\mathrm{b}}^{\top}]^{\top}/2, & \dot{\gbf{\omega}}_{\mathrm{b}} &= \matbf{J}_{\mathrm{b}}^{-1} (\gbf{\tau}_{\mathrm{b}} -[\gbf{\omega}_{\mathrm{b}}]_{\times} \matbf{J}_{\mathrm{b}} \gbf{\omega}_{\mathrm{b}}),
\end{aligned}
\end{equation*}
where $g = 10~\mathrm{m \cdot s^{-2}}$ is the gravitational acceleration, $\vecbf{e}_{z} = [0,0,1]^{\top}$.
The state vector $\stat$ is defined as $\stat := [\vecbf{p}_{\mathrm{w}}^{\top}, \vecbf{v}_{\mathrm{w}}^{\top}, \vecbf{q}_{\mathrm{b}}^{\top}, \gbf{\omega}_{\mathrm{b}}^{\top} ]^{\top} \in \mathbb{R}^{13}$, where $\vecbf{p}_{\mathrm{w}} \in \mathbb{R}^{3}$, $\vecbf{v}_{\mathrm{w}} \in \mathbb{R}^{3}$ respectively denote the position and velocity in the world frame and $\vecbf{q}_{\mathrm{b}} \in \mathbb{R}^{4}$ (equivalent rotation representation $\matbf{R} \in SO(3)$), $\gbf{\omega}_{\mathrm{b}} \in \mathbb{R}^{3}$ respectively denote the orientation and angular velocity in the body frame. 
 $m \in \mathbb{R}$ is the mass and $\matbf{J}_{\mathrm{b}} \in \mathbb{R}^{3 \times 3}$ is the moment of inertia.
 $T_{\mathrm{b}} \in \mathbb{R}$ and $\gbf{\tau}_{\mathrm{b}} \in \mathbb{R}^{3}$ denote the overall thrust and torque in the body frame, which are generated by: 
\begin{equation*}
\label{eq:mixer_eq}
\left[\begin{smallmatrix}
T_{\mathrm{b}} \\
\gbf{\tau}_{\mathrm{b}}
\end{smallmatrix}\right]
=\left[ \begin{smallmatrix}
1 & 1 & 1 & 1 \\
0 & -l / 2 & 0 & l / 2 \\
-l / 2 & 0 & l / 2 & 0 \\
c & -c & c & -c
\end{smallmatrix} \right] \ctrl,
\end{equation*}
where $l$ is the wing length, $c$ is the thrust-torque ratio, and $\ctrl \in \mathbb{R}^{4}$ denotes the thrust generated by four propellers.
The control task is to drive the system to the desired state at $[\vecbf{0}_{3}, \vecbf{0}_{3}, \vecbf{q}_{\mathrm{d}}, \mathbf{0}_{3}]^{\top}$.
Similar stage and terminal costs can be considered for this example except that the cost term for orientation should be $\para_{\vecbf{q}_{\mathrm{b}}}\tr(\matbf{I}_{3} -  \matbf{R}_{\mathrm{d}}^{\top}\matbf{R})/2$.
In addition, we set norm-bounded constraints for both state and control vectors, i.e., $\|\vecbf{p}_{\mathrm{w}}\| \leq r$ and $\|\ctrl\|_{\infty} \leq u_{\mathrm{ub}}$, where $r$ and $u_{\mathrm{ub}}$ are the radius of safe area and the upper bound of thrust, respectively.
We set $\para = \{ m, \matbf{J}_{\mathrm{b}}, l, \para_{\stat}, r, u_{\mathrm{ub}}\}$ as the parameter to be learned.

\noindent
\paragraph{Two-link robot arm} the system dynamics is given by
{\small \begin{equation*}
\begin{aligned}
\begin{bmatrix}\ddot{q}_{1} \\ \ddot{q}_{2}
\end{bmatrix} & = M^{-1}\left(\begin{bmatrix}\tau_{1} \\ \tau_{2}
\end{bmatrix} - m_{2}l_{1}l_{2} \begin{bmatrix} -\dot{q}_{2}^{2} - 2\dot{q}_{1}\dot{q}_{2} \\ \dot{q}_{1}^{2}
\end{bmatrix} /2 - \right .\\
& \left . \begin{bmatrix} m_{1}l_{1}g \cos(q_{1})/2 + m_{2}g(l_{2}\cos(q_{1}+q_{2})/2 + l_{1}\cos(q_{1})) \\ m_{2}gl_{2}\cos(q_{1}+q_{2})/2\end{bmatrix}\right), \\
\end{aligned}
\end{equation*}}
where $m_{i}, l_{i}, I_{i} = m_{i}l_{i}^{2}/12, i \in \mathcal{I}_{2}$ denote the link mass, link length, and angular momentum, respectively;
{\small
\begin{equation*}
\begin{aligned}
M & = \left[\begin{matrix} m_{1}l_{1}^{2}/4 + I_{1} + m_{2} (l_{1}^{2} + l_{2}^{2}/4 + 2l_{1} + l_{2}\cos(q_{2}/2) ) + I_{2} \\ m_{2}(l_{2}^{2}/4+l_{1}l_{2}\cos(q_{2})/2 + I_{2})
\end{matrix}\right. \\
&  \phantom{xxxxxxxxxxxxxxxxxxxx} \left. \begin{matrix}  m_{2}(l_{2}^{2}/4+l_{1}l_{2}\cos(q_{2})/2 + I_{2})\\  m_{2}l_{2}^{2}/4+I_{2}
\end{matrix} \right]. 
\end{aligned}
\end{equation*}}
The state vector $\stat$ is defined as $\stat := [q_{1},q_{2}, \dot{q}_{1}, \dot{q}_{2}]^{\top}$, which is the concatenation of the angles and angular velocities of both links, and the control input $\ctrl := [\tau_{1}, \tau_{2}]^{\top}$ is the concatenation of torques.
The control task is to drive the system to the desired state at $\stat_{\mathrm{d}} = [\pi/2, 0, 0, 0]^{\top}$.
In addition, we set norm-bounded constraints for both state and control vectors, i.e., $|q_{i}| \leq q_{\mathrm{ub}}$ and $\|\ctrl\|_{\infty} \leq u_{\mathrm{ub}}$, where $q_{\mathrm{ub}}$ and $u_{\mathrm{ub}}$ are the upper bounds of the joint angle and the torque, respectively.
We set $\para = \{l_{1}, l_{2}, \para_{\stat}, q_{\mathrm{ub}}, u_{\mathrm{ub}}\}$ as the parameter to be learned.

\noindent
\paragraph{Rocket}
the system dynamics is given by
\begin{equation*}
\begin{aligned}
\dot{\vecbf{p}}_{\mathrm{w}} &= \vecbf{v}_{\mathrm{w}}, & \dot{\vecbf{v}}_{\mathrm{w}} &=  \matbf{R} \gbf{\tau}/m - g \vecbf{e}_{z}, \\
\dot{\vecbf{q}}_{\mathrm{b}} &= \vecbf{q}_{\mathrm{b}} \oplus [0, \gbf{\omega}_{\mathrm{b}}^{\top}]^{\top}/2, & \dot{\gbf{\omega}}_{\mathrm{b}} &= \matbf{J}_{\mathrm{b}}^{-1} ([\vecbf{r}_{\mathrm{gp}}]_{\times} \ctrl -[\gbf{\omega}_{\mathrm{b}}]_{\times} \matbf{J}_{\mathrm{b}} \gbf{\omega}_{\mathrm{b}}),
\end{aligned}
\end{equation*}
where $\vecbf{r}_{\mathrm{gp}} \in \mathbb{R}^{3}$ is the gimbal-point position vector and $\ctrl \in \mathbb{R}^{3}$ is the vectored thrust.
The control task is to drive the system to the desired state at $[\mathbf{0}_{3}, \mathbf{0}_{3}, \vecbf{q}_{\mathrm{d}}, \mathbf{0}_{3}]^{\top}$.	
In addition, we set norm-bounded constraints for both state and control vectors, i.e., $\tr(\matbf{I}_{3} -  \matbf{R}_{\mathrm{d}}^{\top}\matbf{R})/2 \leq \alpha_{\mathrm{ub}}$ and $\|\ctrl\|_{2} \leq u_{\mathrm{ub}}$, where $\alpha_{\mathrm{ub}}$ and $u_{\mathrm{ub}}$ are the upper bounds of the tilt angle and the vectored thrust, respectively.
We set $\para = \{ m, \matbf{J}_{\mathrm{b}}, \para_{\stat}, \alpha_{\mathrm{ub}}, u_{\mathrm{ub}}\}$ as the parameter to be learned.

\subsection{PDP-based vs. DDP-based methods for gradient computation}
\label{subsec:simu_pdp_ddp_comp}
We first consider the unconstrained IRL problem with open-loop loss. 
For the above-mentioned four examples, we temporarily exclude the norm-bounded constraints and their involved upper bounds from the optimal control problem and the learning parameters, respectively.
We use both the PDP-based \cite{jin2020pontryagin} and our proposed DDP-based algorithms to compute the required gradient.
For the sake of clarity, we only run the gradient descent for $20$ steps for this comparison.
Figure \ref{fig:ddp_pdp_comp_grad_diff} shows the difference between gradients computed by two algorithms. 
Figure \ref{fig:ddp_pdp_comp_comp_time} shows the computational time for the gradient computation in each gradient descent step adopting both algorithms.
Next, we present the comparison of SafePDP \cite{jin2021safe} and our proposed IPDDP-based algorithm, which is used for the IRL problem with constraints.  
Similarly, the gradient difference and computational time are recorded in Figs. \ref{fig:ipddp_safepdp_comp_grad_diff}, \ref{fig:ipddp_safepdp_comp_comp_time}, respectively. Additionally, we implement the BarrierDDP-based method mentioned in Remark \ref{rem:barrierDDP}, which incorporates the constraints into stage cost via barrier functions. 
Based on the above results, we have the following comments.
\begin{figure}[h]
	\centering
		\includegraphics[width=0.9\linewidth]{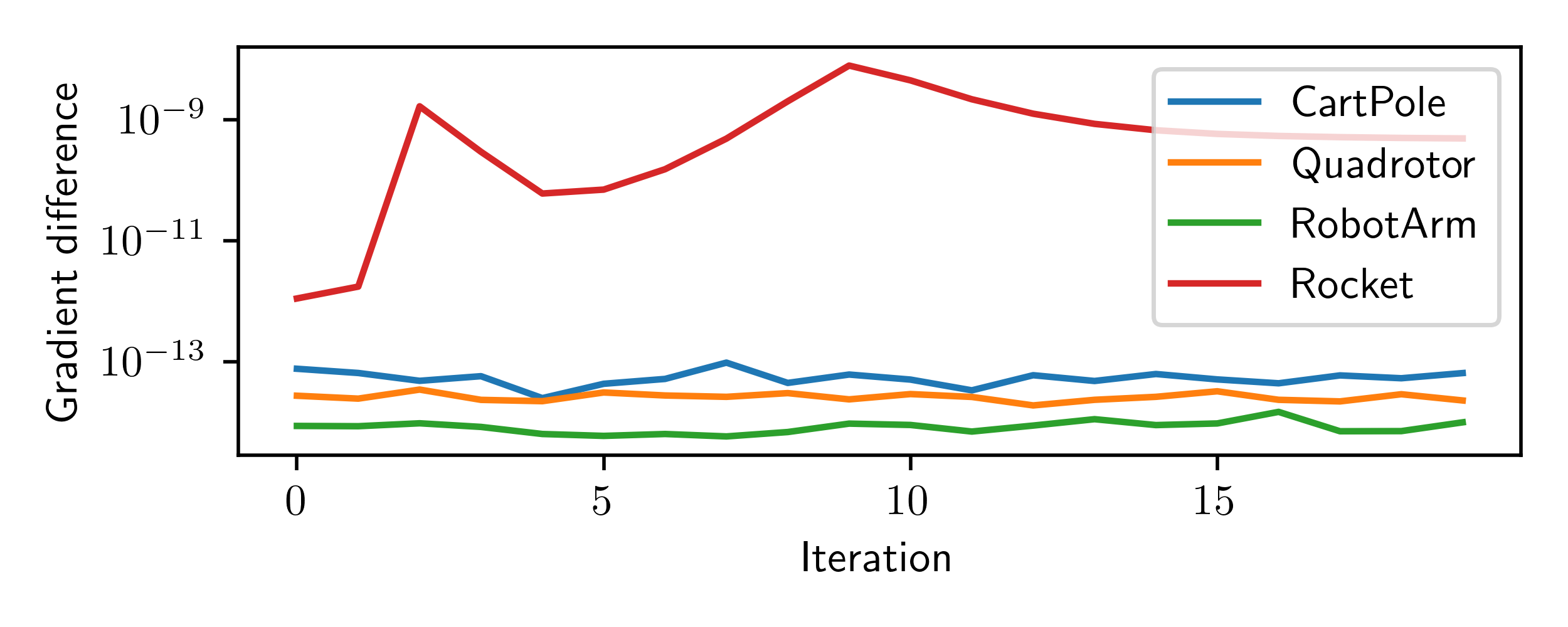}
	\caption{The difference between the gradients computed by PDP-based and proposed DDP-based algorithms on unconstrained problems.}
	\label{fig:ddp_pdp_comp_grad_diff}
\end{figure}

\begin{figure}[h]
	\centering
		\includegraphics[width=0.9\linewidth]{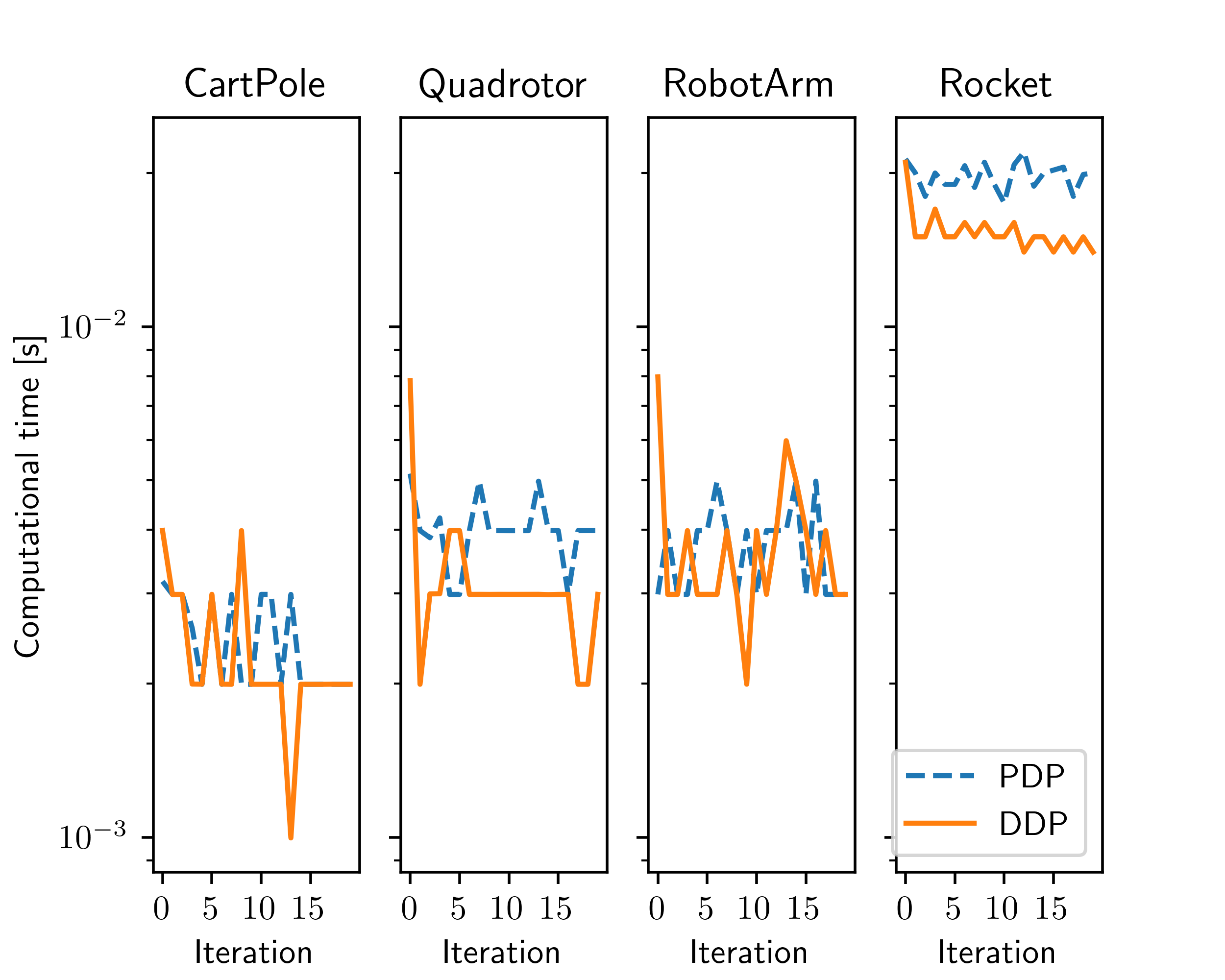}
	\caption{The computational time for each call of PDP-based and proposed DDP-based algorithms on unconstrained problems.}
	\label{fig:ddp_pdp_comp_comp_time}
\end{figure}

\begin{figure}[h]
	\centering
		\includegraphics[width=0.9\linewidth]{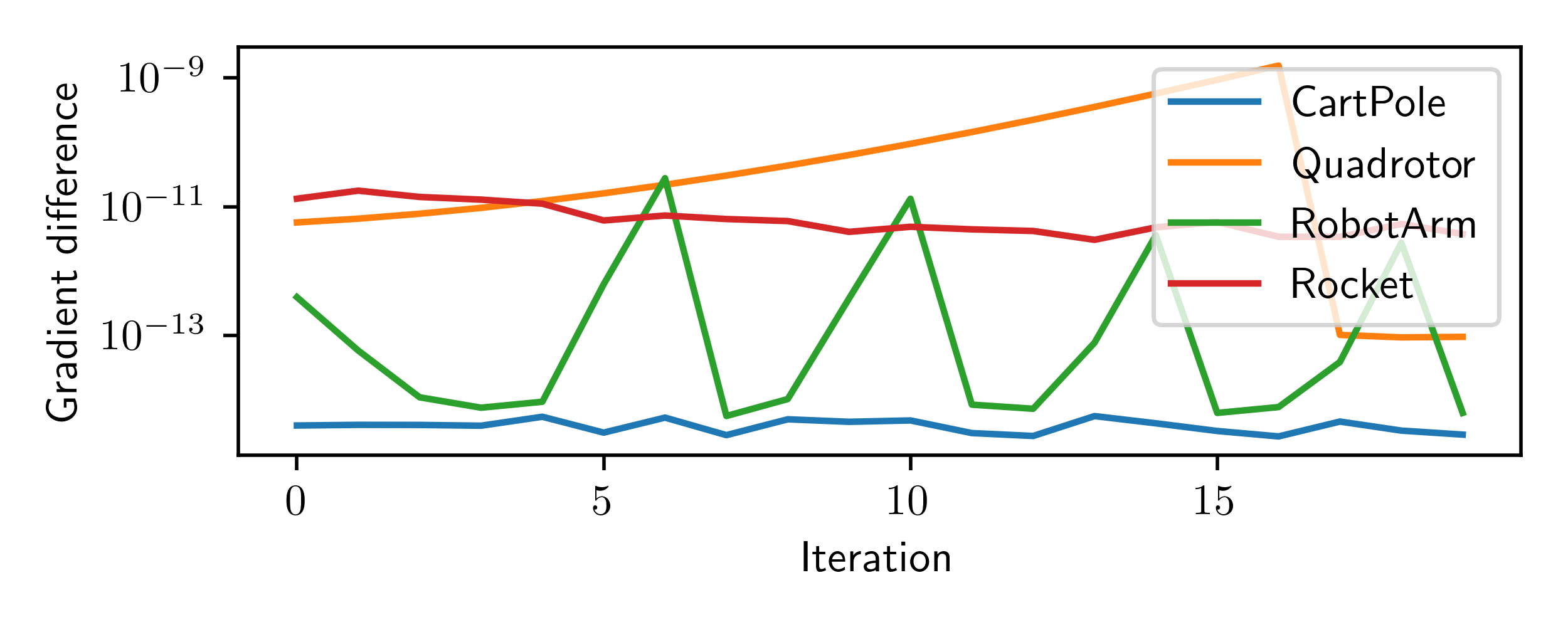}
	\caption{The difference between the gradients computed by PDP-based and proposed DDP-based algorithms on constrained problems.}
	\label{fig:ipddp_safepdp_comp_grad_diff}
\end{figure}

\begin{figure}[h]
	\centering
		\includegraphics[width=0.9\linewidth]{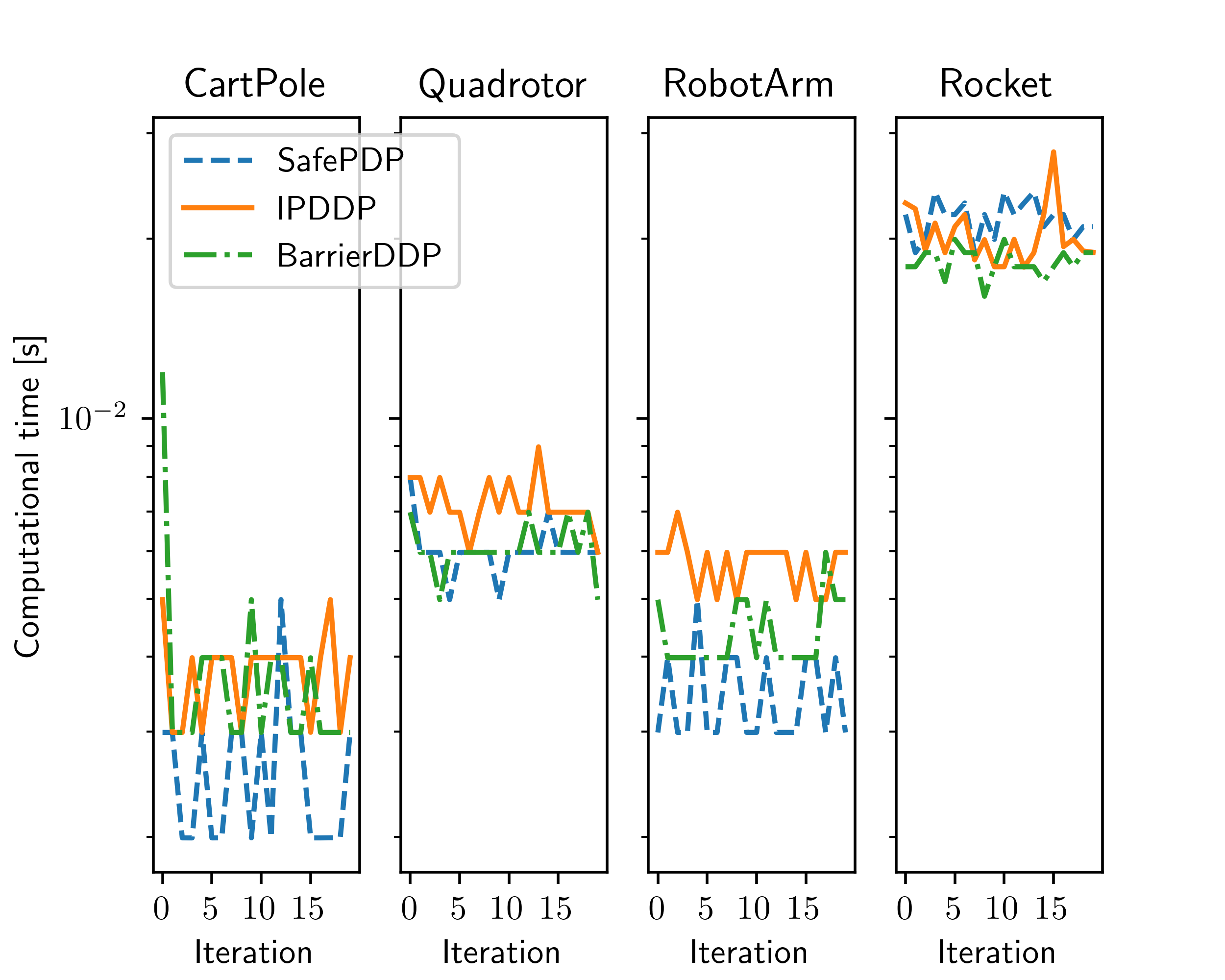}
	\caption{The computational time for each call of PDP-based and proposed DDP-based algorithms on constrained problems.}
	\label{fig:ipddp_safepdp_comp_comp_time}
\end{figure}

1) In terms of gradient difference, it can be seen from Figs. \ref{fig:ddp_pdp_comp_grad_diff} and \ref{fig:ipddp_safepdp_comp_grad_diff} that the residual is negligible for the tested examples, which verifies that our theoretical result of the equivalence of the gradient computation from the algorithms. 

2) It can be found from Fig. \ref{fig:ddp_pdp_comp_comp_time} that for the system with a lower dimension (cartpole and robot arm), the computational time is marginally the same, while for those with a higher dimension (quadrotor and rocket), DDP-based algorithm is faster since our derivation is more compact in the sense that it uses a vectorized form of many small terms which are also used in PDP-based algorithms.

3) As seen from Fig. \ref{fig:ipddp_safepdp_comp_comp_time}, compared to SafePDP, IPDDP-based algorithm is marginally worse for the first three examples while marginally better in the fourth example. 
    The reason is that although the compact derivation saves the computational time (as explained in Fig. \ref{fig:ddp_pdp_comp_comp_time}), IPDDP-based algorithm introduces the dual variables as the control variable, which increases the problem size and leads to a bit longer computational overhead for symbolic evaluation of \eqref{eq:hat_Q_for_gradient}.
    However, this is not the case for BarrierDDP since its implementation does not increase the problem size as in IPDDP-based algorithm while inheriting the advantage of DDP over PDP on problems with higher dimensions, which can be seen from Fig. \ref{fig:ipddp_safepdp_comp_comp_time}.

\subsection{Advantages of closed-loop IRL over open-loop IRL}
\label{subsec:simu_adv}
We define the following metrics to evaluate the performance of our proposed algorithms.
\begin{itemize}
    \item \textbf{Parameter residual}, which  measures the error between the learned parameter $\para$ and ground truth $\para^{\ast}$, i.e.,
        $$\resi_{\mathrm{para}}(\para) := \|\para - \para^{\ast}\|^{2},$$
    $\resi_{\mathrm{para}} = 0$ means exact recovery of the true parameter.
    \item \textbf{Trajectory residual}, measuring the distance between the demonstration trajectories $\trajset(\para^{\ast})$ and the rollout trajectories $\trajset_{\mathrm{rollout}}(
    \para)$, i.e., 
    $$\resi_{\mathrm{traj}}(\para) := \|\trajset(\para^{\ast}) - \trajset_{\mathrm{rollout}}(
    \para
    )\|_{2}^{2},$$
    This metric resembles the open-loop loss $L^{\mathrm{ol}}$ while differs in that the rollout trajectories $\trajset_{\mathrm{rollout}}(
    \para)$ are not obtained by directly solving \eqref{eq:prob_ineq_eq} but by performing the feedback policy $\{\gaink, \gainK\}$ on the system with the true dynamics, i.e., $\dyn(\cdot;\para^{\ast})$, which is possibly contaminated by a process noise.

    \item \textbf{Suboptimality gap}, which measures the performance gap between the testing demonstrations $\trajset(\para^{\ast})$ and the rollout trajectories $\trajset_{\mathrm{rollout}}(
    \para)$ evaluated at the performance index under parameter $\para^{\mathrm{e}}$, i.e.,
    $$\resi_{\mathrm{sub}}(\para; \para^{\mathrm{e}}) := \mpccost(\trajset_{\mathrm{rollout}}(
    \para); \para^{\mathrm{e}}) -\mpccost(\trajset(\para^{\ast}); \para^{\mathrm{e}}),$$  
     Specifically, $\para^{\mathrm{e}}$ can be chosen among the true  $\para^{\ast}$ and the final value of the learned parameters. 
    Note that this suboptimality gap can be negative even if $\para = \para^{\ast}$ due to different noise realizations. 

\end{itemize}

\begin{figure}[!t]
	\centering
		\includegraphics[width=0.9\linewidth]{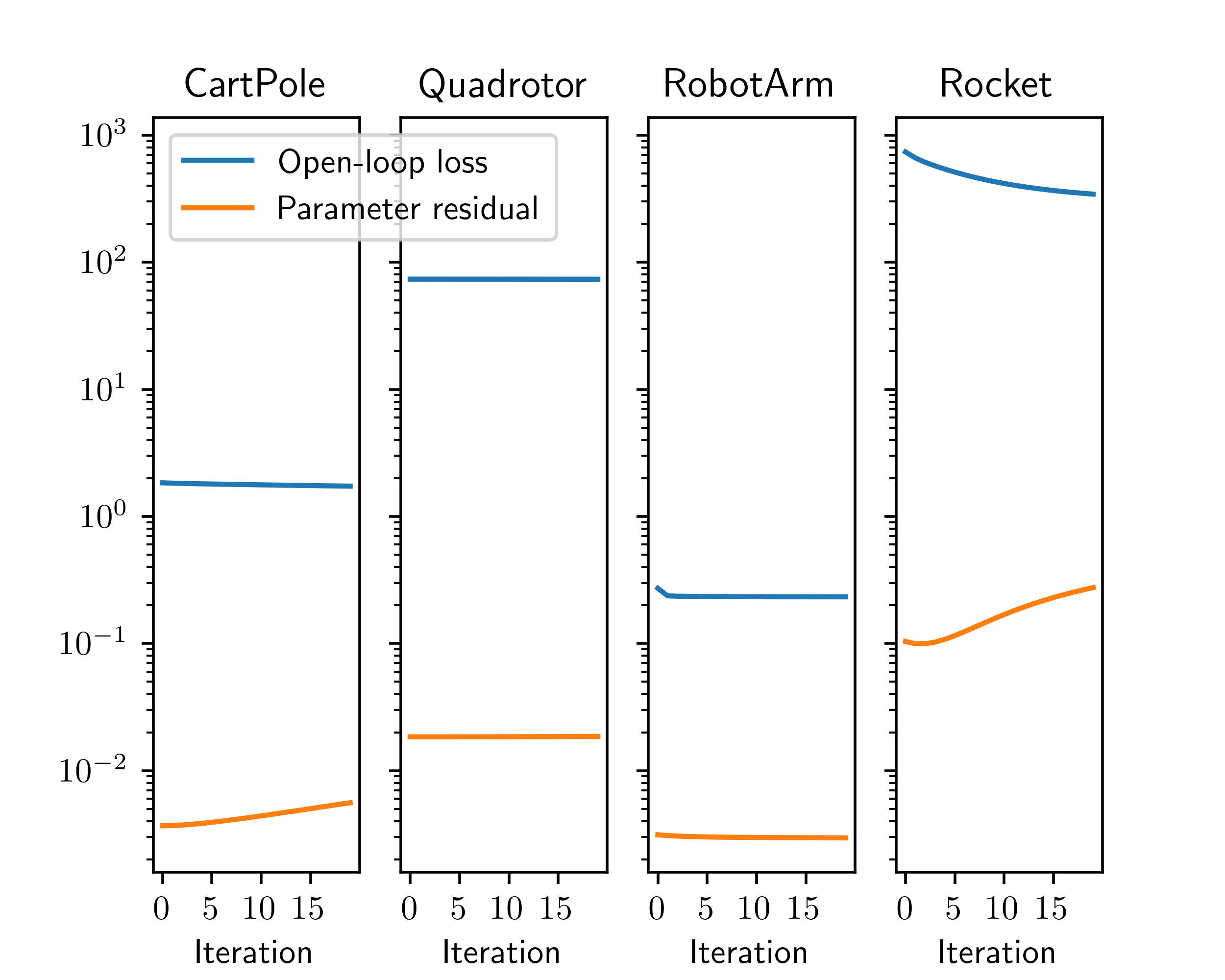}
	\caption{Traces of loss and parameter estimation error by adopting PDP-based and proposed DDP-based algorithms on unconstrained problems. The stepsizes are set as $10^{-3}, 10^{-4}, 10^{-2}, 10^{-4}$, and the horizons are set as $N=12,10,10,40$.}
	\label{fig:ddp_pdp_comp_loss}
\end{figure}

\begin{figure}[!t]
	\centering
		\includegraphics[width=0.9\linewidth]{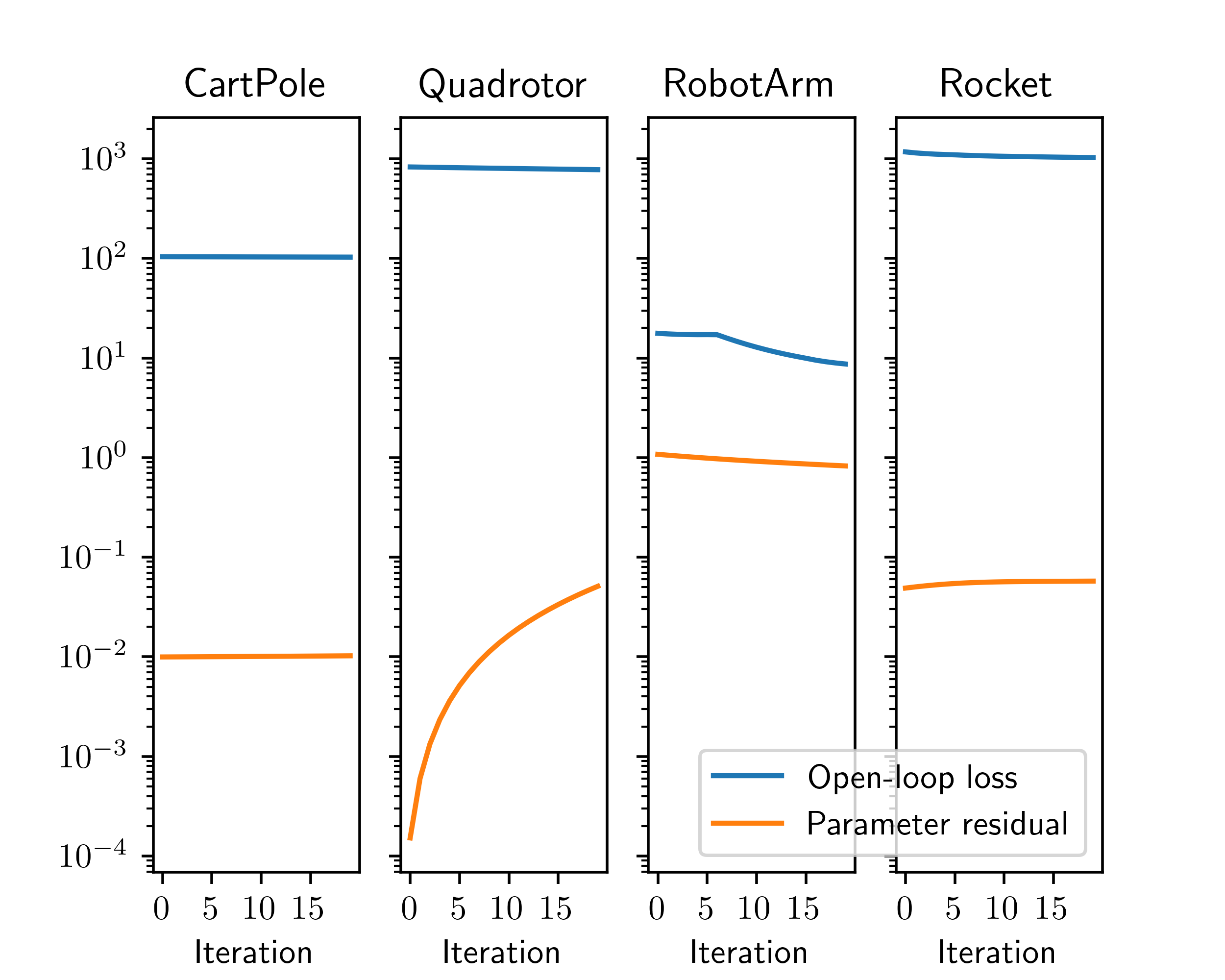}
	\caption{Traces of loss and parameter estimation error by adopting PDP-based and proposed DDP-based algorithms on constrained problems.}
	\label{fig:ipddp_safepdp_comp_loss}
\end{figure}

\begin{figure}[!t]
	\centering
		\includegraphics[width=0.9\linewidth]{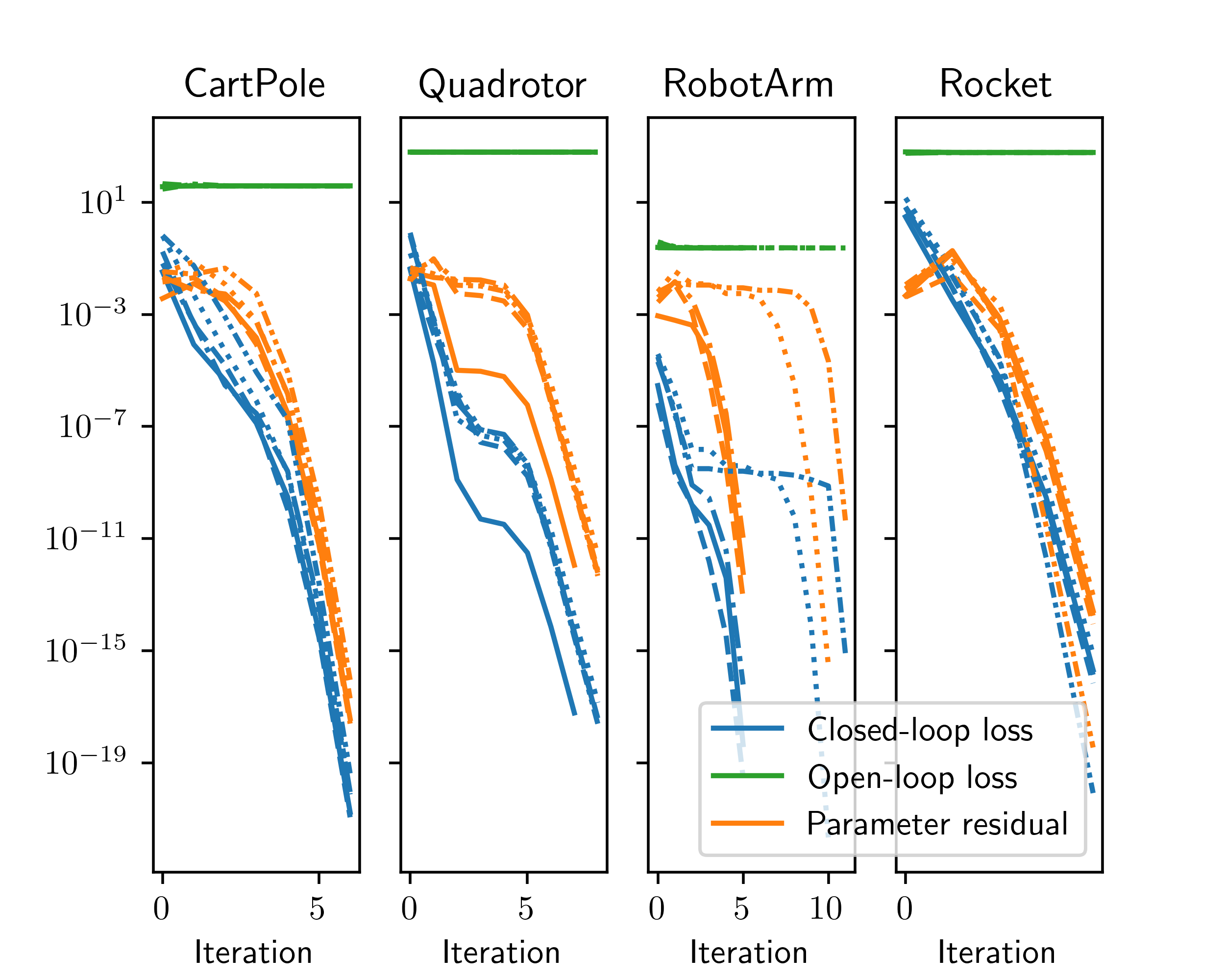}
	\caption{Traces of loss and parameter estimation error by adopting Algorithm \ref{alg:cl_irl}.}
	\label{fig:ddp_LM_new_loss}
\end{figure}

\begin{figure*}[!t]
	\centering
		\includegraphics[width=0.9\linewidth]{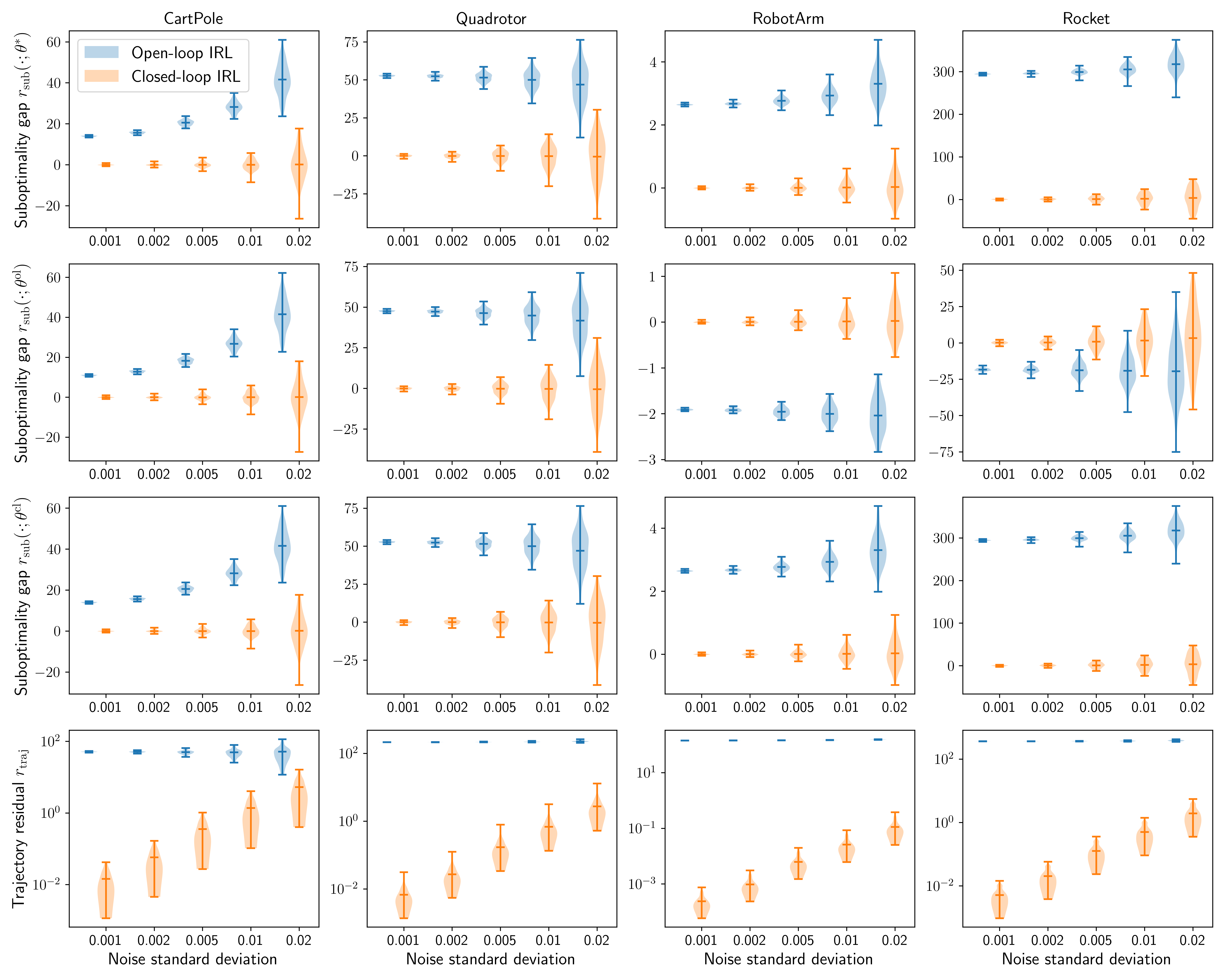}
	\caption{Performance evaluation on different metrics with parameters learned from open-loop and closed-loop IRL algorithms. The lower and upper bars denote the range and the middle bar denotes the mean. The shaded area shows the probability density of the data at different values.}
	\label{fig:perfeval}
\end{figure*}

We first present a qualitative comparison between the open-loop and closed-loop IRL algorithms. 
For open-loop IRL, we use the gradients calculated in generating Figs. \ref{fig:ddp_pdp_comp_grad_diff} and \ref{fig:ipddp_safepdp_comp_grad_diff} to update the parameter according to Algorithm \ref{alg:ol_irl} and record the trace of open-loop loss $L^{\mathrm{ol}}$ and the parameter residual $\resi_{\mathrm{para}}$ in Figs. \ref{fig:ddp_pdp_comp_loss} and  \ref{fig:ipddp_safepdp_comp_loss}. 
For closed-loop IRL, we implement Algorithm \ref{alg:cl_irl} for $5$ trials in each simulation example and record the trace of closed-loop loss $L^{\mathrm{cl}}$ and parameter residual $\resi_{\mathrm{para}}$ in Fig. \ref{fig:ddp_LM_new_loss}, where each type of line denotes a different trial. 
For the first trial denoted by solid lines, it uses the same initial condition as that in Fig. \ref{fig:ddp_pdp_comp_loss}.
In the meantime, we record its open-loop loss $L^{\mathrm{ol}}$ during the learning process, denoted by the green lines.
In order to quantitatively demonstrate the advantages of our proposed closed-loop IRL, we test the learned parameter $\para^{\mathrm{ol}}$ and $\para^{\mathrm{cl}}$ from the algorithms in the following new setting which is different for training.
Specifically, we set the horizon as $20$, randomly choose a new initial condition, and use $\para^{\mathrm{ol}}$ and $\para^{\mathrm{cl}}$ to compute its corresponding feedback policy.
During rollout, we randomly add multiplicative process noise to the system dynamics and record the entire trajectory.
We repeat the simulation 100 times for each algorithm under noise of different standard deviations.
Additionally, we go through the same process with the true parameter $\para^{\ast}$ to generate the test dataset.
Then, we evaluate these trajectories with the above-defined sub-optimality gaps $\resi_{\mathrm{sub}}(\para; \para^{\mathrm{e}})$ with $\para^{\mathrm{e}} \in \{\para^{\ast}, \para^{\mathrm{ol}}, \para^{\mathrm{cl}}\}$ and trajectory residual $\resi_{\mathrm{traj}}$, as shown in each row in Fig. \ref{fig:perfeval}. Based on the above results, we have the following comments.

 1) As seen from Fig. \ref{fig:ddp_pdp_comp_loss}, the loss is decreasing slowly as expected for a gradient descent algorithm. 
    Further, as explained in Sec. \ref{sec:cl_irl}, due to the different nature of demonstration (closed-loop) and loss function (open-loop), the optimizing direction for $L^{\mathrm{ol}}$ does not necessarily coincide with the optimizing direction for the parameter residual $\resi_{\mathrm{para}}$. 
    This can be seen from the rocket example (the last column of Fig. \ref{fig:ddp_pdp_comp_loss}), where the parameter residual is indeed increasing.
    A similar phenomenon can be observed from Fig. \ref{fig:ipddp_safepdp_comp_loss}, i.e., the parameter residuals for cartpole, quadrotor, and rocket systems increase even the open-loop losses decrease. 
    
 2)  It can be found from Fig. \ref{fig:ddp_LM_new_loss} that different from open-loop IRL, the parameter residual of closed-loop IRL decreases as the closed-loop loss decreases.
    In the meantime, the open-loop loss is recorded (not used for iteration), from which one can find that it remains a large value even if the parameter residual is negligible.
    This is expected since our closed-loop design has incorporated the closed-loop nature of demonstrations while not seeking to minimize the discrepancy between demonstrated and reproduced trajectories.
    Additionally, owing to the usage of the LM algorithm, it only takes tens of iterations to converge to a very small residual, which is significantly faster than the gradient-descent-based closed-loop IRL.
    
3) For the first three rows of Fig. \ref{fig:perfeval}, the range and variance of suboptimality gaps for both algorithms increase as the standard deviation of noise increases, while the mean of those for closed-loop IRL is approximately zero, indicating that it achieves a similar level of performance (in the sense of cost function) as the policy induced from the true parameter.
    As seen from the first row of Fig. \ref{fig:perfeval}, closed-loop IRL significantly outperforms open-loop IRL in terms of suboptimality gap evaluated at true parameter, i.e. $\resi_{\mathrm{sub}}(\para; \para^{\ast})$. 
    The third row which corresponds to the suboptimality gap evaluated at closed-loop IRL learned parameter $\para^{\mathrm{cl}}$ resembles the first row since the parameter residual $\resi_{\mathrm{para}}(\para^{\mathrm{cl}})$ is negligible. 
    We cannot guarantee the advantage of closed-loop IRL over open-loop IRL in terms of suboptimality gap evaluated at open-loop IRL learned parameter $\para^{\mathrm{ol}}$ (the second row of Fig. \ref{fig:perfeval}), since in this case the latter is exactly optimized under $\para^{\mathrm{ol}}$ and is expected to outperform the former.  
    Nevertheless, we observe that the former still outperforms the latter in the cartpole and quadrotor example and they are close in the rocket example, since in these cases open-loop IRL wrongly estimates the parameter in system dynamics, while the rollout is performed on the dynamic system with the true parameter $\para^{\ast}$.  
    
 4) As seen from the last row of Fig. \ref{fig:perfeval}, closed-loop IRL outperforms open-loop IRL by at least one order in terms of trajectory residual $\resi_{\mathrm{traj}}$, which means that the rollout trajectory generated from learned parameter is much closer to the one generated from the true parameter. 
    Different from the previous three rows where the mean for closed-loop IRL is always approximately zero, the mean in this row increases as the standard deviation of noise increases, this is because different noise realizations lead to distinct rollout trajectories and hence a strictly positive trajectory residual $\resi_{\mathrm{traj}}$, and the difference between two trajectories increases. 
    This can also be understood with a simplified case where the rollout trajectory is assumed to be a linear function of parameter with additive noise, then under negligible parameter residual, i.e., $\para = \para^{\ast}$, the mean of trajectory residual $\resi_{\mathrm{traj}}$ is exactly two times of the variance of the noise. 
    
 5) As can be observed from all of the subplots in Fig. \ref{fig:perfeval}, the ranges of data for two algorithms overlap (or are going to be overlapping) with each other as the noise gets larger, this is because the rollout trajectory deviate too much from the nominal trajectory which is used for computing the feedback policy, and hence the policy cannot be guaranteed to perform well in this case.

\subsection{Properties of closed-loop IRL} 
In the previous section, we have demonstrated the advantages of our proposed closed-loop IRL over the open-loop one by implicitly assuming that both algorithms are training with sufficient data. 
In this section, we aim to provide an in-depth analysis on how much data is required.
As before, we first present a set of qualitative examples, where we set $|\sampleset| = 2$, i.e., $2$ sampling instants within horizon $N$, and apply Algorithm \ref{alg:cl_irl} subsequently. 
We use the same initial conditions ($5$ trials) as in Fig. \ref{fig:ddp_LM_new_loss} and record the trace of loss $L^{\mathrm{cl}}$ and parameter residual $\resi_{\mathrm{para}}$ in Fig. \ref{fig:ddp_LM_new_loss_shorterHorizon}. 
Next, we vary the length of demonstration $|\sampleset|$ from $1$ to $10$, and only record the final closed-loop loss $L^{\mathrm{cl}}$ and parameter residual $\resi_{\mathrm{para}}$, as shown in Fig. \ref{fig:ddp_LM_new_loss_shorterHorizon_by_length}.
\begin{figure}[h]
	\centering
		\includegraphics[width=0.9\linewidth]{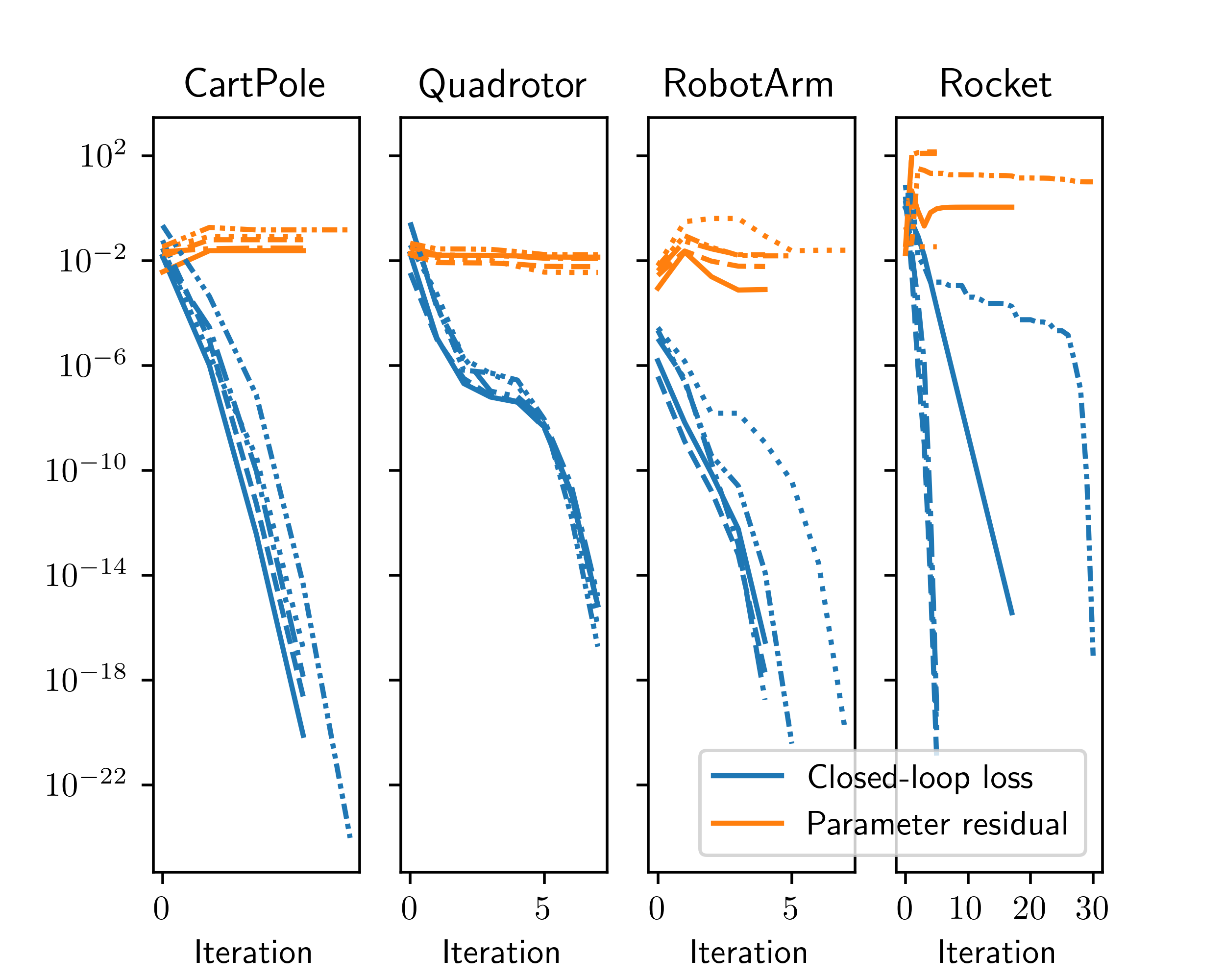}
	\caption{Traces of loss and parameter estimation error by adopting Algorithm \ref{alg:cl_irl} with a short length of demonstrations ($|\sampleset|=2$).}
 \label{fig:ddp_LM_new_loss_shorterHorizon}
\end{figure}

\begin{figure}[h]
	\centering
		\includegraphics[width=0.9\linewidth]{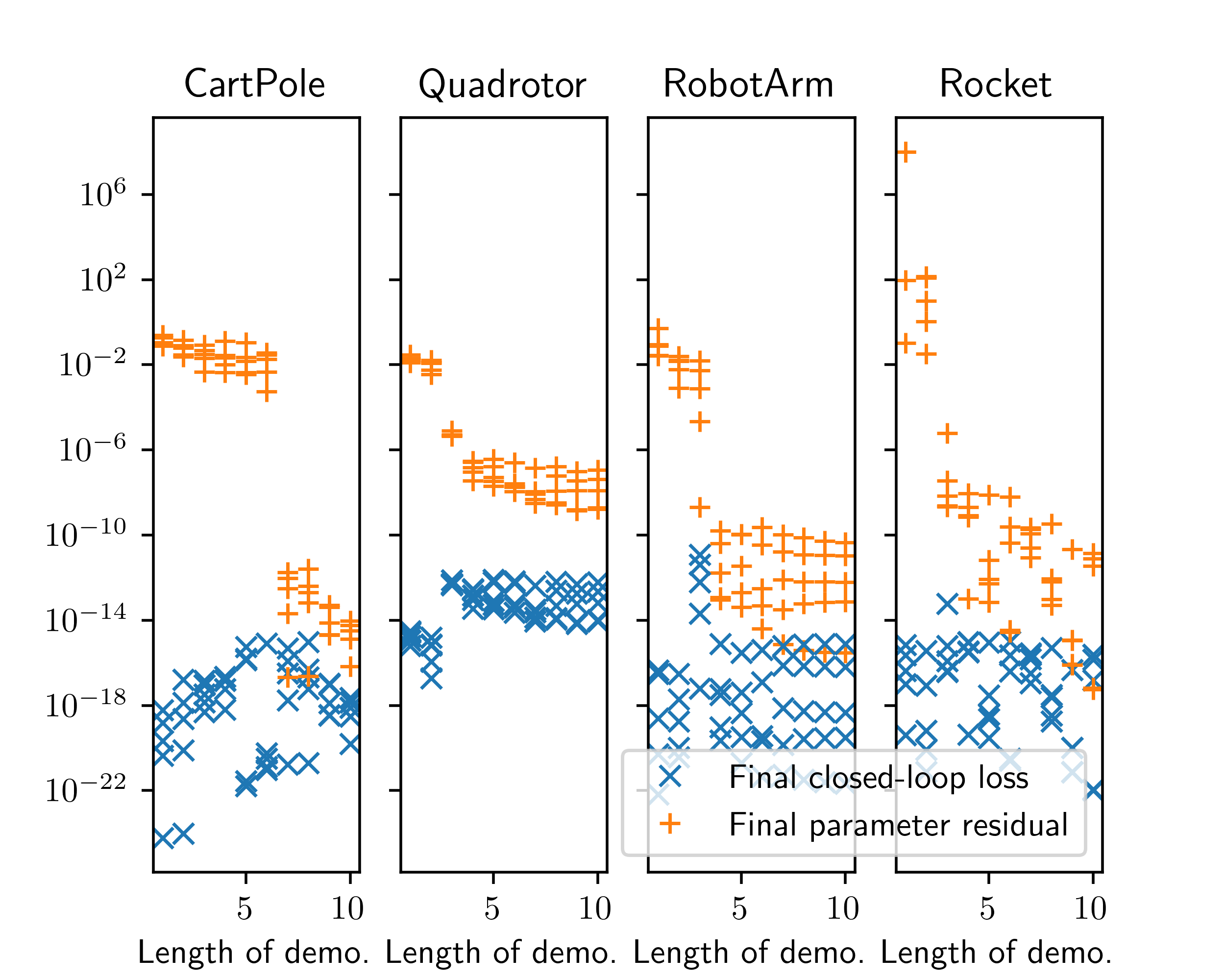}
	\caption{Final loss and parameter estimation error by adopting Algorithm \ref{alg:cl_irl} with different lengths of demonstrations.}
 \label{fig:ddp_LM_new_loss_shorterHorizon_by_length}
\end{figure}
It can be found from Fig. \ref{fig:ddp_LM_new_loss_shorterHorizon} that although the closed-loop loss $L^{\mathrm{cl}}$ decreases rapidly, the parameter residual stops updating and remains a non-negligible value. 
The reason is that for $\Rank(\recMat) < \size{\para}$, there exists another set of parameters except for $\para^{\ast}$ such that the closed-loop loss $L^{\mathrm{cl}}$ is zero. 
This result can be more easily seen in Fig. \ref{fig:ddp_LM_new_loss_shorterHorizon_by_length}. 
One can find an obvious parameter residual $\resi_{\mathrm{para}}$ drop when $|\sampleset|$ is near to $\lceil \size{\para}/\size{\ctrl} \rceil$ since in this case $\Rank(\recMat) = \size{\para}$ in general, e.g. for the quadrotor example, $\lceil \size{\para}/\size{\ctrl} \rceil = \lceil 9/4 \rceil = 3$.
For a longer length of demonstration, i.e., $|\sampleset| > \lceil \size{\para}/\size{\ctrl} \rceil$, both the closed-loop loss $L^{\mathrm{cl}}$ and the parameter residual $\resi_{\mathrm{para}}$ remain negligible values since $\Rank(\recMat)$ is non-decreasing w.r.t. the increase of $|\sampleset|$.

\subsection{Constrained inverse optimal control} 

In this section, we present an example of an LQR problem to validate Corollary \ref{cor:recov}. 
We consider the linear system $\stat^{+} = [\begin{smallmatrix}
    -1&1 \\
    0&1 
\end{smallmatrix}] \stat + [\begin{smallmatrix}
    1 \\
    3
\end{smallmatrix}] \ctrl$ with the stage cost 
$\pathcost := \stat^{\top}\Diag(\para_{\stat})\stat + \para_{\ctrl}\ctrl^{\top}\ctrl$
and the terminal cost $\termcost := 0$, where the true parameter $\para^{\ast} = [\para_{\stat}^{\ast\top}, \para_{\ctrl}^{\ast\top}]^{\top} = [0.1, 0.3, 0.6]^{\top}$.
Alternatively, one can rewrite $\pathcost = \gbf{\phi}^{\top}\para$ with $\gbf{\phi} = [\stat^{\top}\otimes\stat^{\top}, \ctrl^{\top}\otimes\ctrl^{\top}]^{\top}$, which satisfies Assumption \ref{cor:recov}-3).
We use the inequality constraint $\|[\stat]_{1}\ctrl\| \leq 0.1,$
which is a more general nonlinear constraint than the control-only constraint considered in \cite{molloy2020online}.
We generate the initial state $\stat_{0}$ randomly and produce a trajectory with horizon $N=50$.
For this trajectory, we use different lengths (from $1$ to $100$) of observation and perturbation to construct the matrix $\recMat_{\mathrm{lin},i}, i = 1,2,3$ as defined in \eqref{eq:recMat}. 
The rank of $\recMat_{\mathrm{lin},1:2}$ and the parameter residual are recorded in Fig. \ref{fig:ioc}.
\begin{figure}[h]
	\centering
		\includegraphics[width=\linewidth]{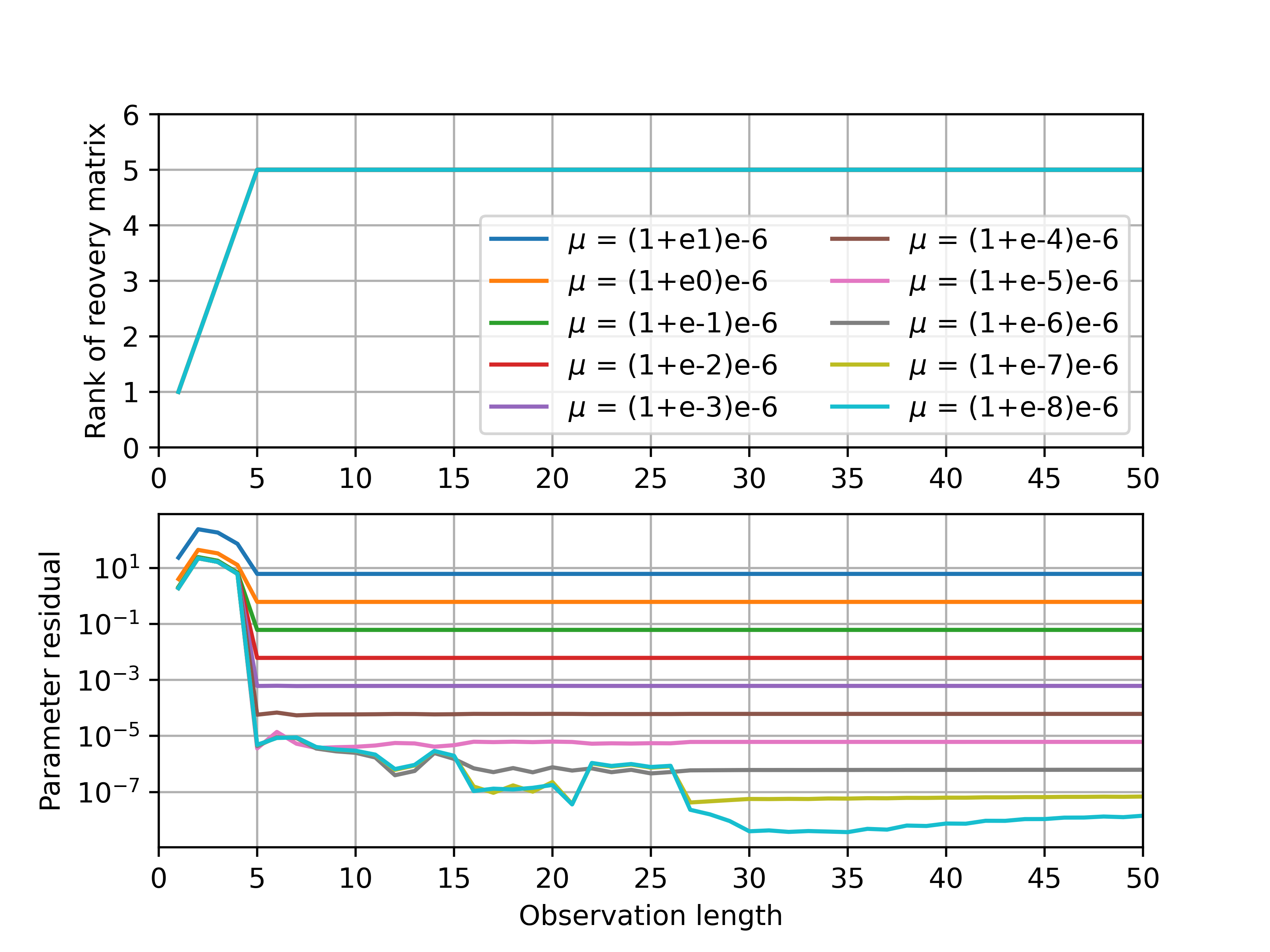}
	\caption{Rank of $\recMat_{\mathrm{lin},1:2}$ and parameter residual w.r.t. different observation length and perturbation.}
	\label{fig:ioc}
\end{figure}
It can be found that when the observation length is not long enough, i.e., $|\sampleset| < \lceil (\size{\para}+\size{\stat})/\size{\ctrl} \rceil = 5$, $\recMat_{\mathrm{lin},1:2}$ is rank-deficient, and it will be rank $5$ when it is sufficiently long, i.e., $|\sampleset| \geq 5$.
On the other hand, if $\recMat_{\mathrm{lin},1:2}$ is rank-deficient, the solution is not unique and may be of no physical meaning. If $\recMat_{\mathrm{lin},1:2}$ is full column rank, the parameter residual denotes the error between the true parameter and estimated parameter under assumed perturbation $\pert$, and it decreases as $\pert$ decreases.

\section{Real-world Experiments}
\label{sec:exp}
In this section, we aim to demonstrate the advantages of our proposed closed-loop IRL over open-loop IRL via a real-world task, quadrotor navigation in partially unknown environments.
\begin{figure}[!t]
	\centering
		\includegraphics[width=0.9\linewidth]{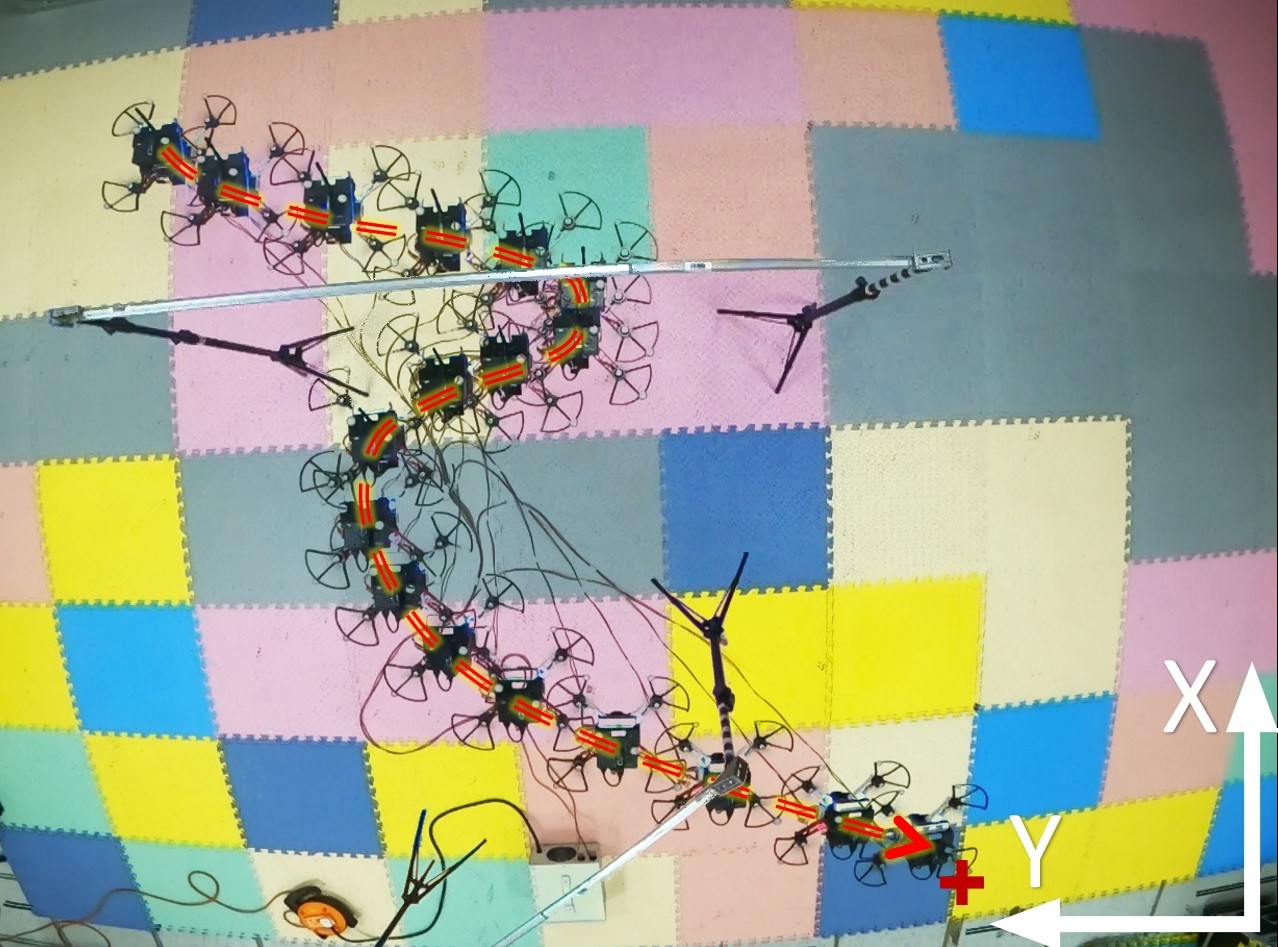}
	\caption{Quadrotor robot experiment system setup (top view). A self-made quadrotor is powered via a cable connected to a ground power supply, it relies on the motion capture system (not shown) for localization, and uses an onboard computer for high-level trajectory planning and a flight controller for low-level tracking. The task is to navigate the quadrotor from the starting position (up-left) to the goal position (bottom-right red plus sign) while flying through two gates (formed by the vertical pole and two tripods) sequentially.}
	\label{fig:uav_setup}
\end{figure}
\subsubsection{Experiment setup}
We verify the advantage of our proposed approach using a self-made tethered quadrotor in a $5 \mathrm{m} \times 5 \mathrm{m} \times 2 \mathrm{m}$ indoor area, which is equipped with a motion capture system. 
Specifically, as seen in Fig. \ref{fig:uav_setup}, the quadrotor is connected to a ground power supply using a cable to support long-duration operation. 
It uses a motion capture system for localization in the environment and is equipped with an i7 computer for onboard computation. 
The quadrotor performs trajectory planning onboard by solving an optimal control, with a linear dynamics model as commonly used in drone control \cite{allen2019real}. 
Denote the position, velocity, and acceleration by $\vecbf{p}$, $\vecbf{v}$, and $\ctrl$, respectively. 
Due to the physical limitations and safety considerations, we set $\|\vecbf{v}\|_{\infty} \leq 1$ and $\|\ctrl\|_{\infty} \leq 0.5$ to limit both the velocity and acceleration. 
With the partially unknown information, the task of trajectory planning is to minimize the following stage cost and terminal cost
\begin{equation*}
\label{eq:uav_exp_cost}
\begin{aligned}
\pathcost & :=  \para_{1} \exp{ -0.01(k - k_{\mathrm{g,1}})^{2}} \| \vecbf{p} - \vecbf{c}_{\mathrm{g,1}} \|^{2} \\
            & \quad + 
                \para_{2} \exp{-0.1(k - k_{\mathrm{g,2}})^{2}} \| \vecbf{p} - \vecbf{c}_{\mathrm{g,2}} \|^{2}
                + \para_{\ctrl} \| \ctrl \|^{2}, \\
\termcost & :=  \|\vecbf{p} - \vecbf{p}_{\mathrm{d}}\|^{2}.
\end{aligned}
\end{equation*}
respectively.
This type of formulation has been used in \cite{jin2022learning, sharma2022correcting}. 
Here, the cost function only encodes the approximate locations of each gate, $\vecbf{c}_{\mathrm{g},1}$ and $\vecbf{c}_{\mathrm{g},2} $, which can be represented by the position of any point on the gates. 
By partially unknown environment, we mean the accurate size (geometry) of the gate is unknown, which is typically required for navigation.  
Therefore, we aim to learn the cost function weights, which encode how the quadrotor safely flies through the gate. 
In the testing and generalization scenarios, we will vary the location of the two gates.
Note that in this case, the stage cost is time-dependent, as mentioned in Sec. \ref{sec:prob}, all of our presented methods still apply.
We assume that $\para_{\ctrl} = 0.01$ to avoid ambiguity and set $\para :=  [\para_{1}, \para_{2}]^{\top}$ as the learning parameter, where $\para_{i} \geq 0, i=1,2$.
We add an additional constraint $\vecbf{1}^{\top}\para = 1$ on the learning parameter.  
The planned high-level trajectory is tracked by a low-level cascaded PID controller. 
Note that both the physical setup (disturbance brought by power cable during motion) and software stack (hierarchical control architecture) necessitate the use of closed-loop control.

\paragraph{Training, test and generalization settings}
As seen in Fig. \ref{fig:uav_setup}, we collect the demonstration trajectory by recording the real-time position obtained by the motion capture system and the high-level control command sent to the low-level controller. 
We set the initial position as $[1.5,1,1]^{\top}$ and the initial velocity as $\vecbf{0}$.
The desired position is set as $[-0.5,-1,1]^{\top}$.
We set $\vecbf{c}_{\mathrm{g},i}$ as the center of two gates with $\vecbf{c}_{\mathrm{g},1} = [1,0,1]^{\top}$ and $\vecbf{c}_{\mathrm{g},2} = [0.5, 1,1]^{\top}$, the height of both gates as $2 \mathrm{m}$, and the width of two gates as $1\mathrm{m}$ and $1.5 \mathrm{m}$.
The planning horizon is set as $N=30$.
In the sequel, we shall refer to this setting as both the training and test setting.
 
Different from the environment for training and testing, we will set new ones by varying the following settings (The other setting is kept the same as the training setting.):
\begin{enumerate}
\renewcommand{\labelenumi}{\arabic{enumi})}
    \item longer planning horizon with $N=40$;
    \item new initial conditions i) $[1.5, 1.5, 1]^{\top}$ and ii) $[1.5,0.5,1]^{\top}$;
    \item new desired position $[0,-1.5,0]^{\top}$; 
\item new gate position  $\vecbf{c}_{\mathrm{g},2} = [0,1,1]^{\top}$, which is further away from gate $1$. 
\end{enumerate}

With these new settings, we use the learned parameters to compute the feedback policy for the high-level trajectory of the quadrotor. 
We check if the trajectories executed in the real-world experiments can successfully complete the task goal: flying through two gates sequentially and arriving at the vicinity of the desired position.  
We use the following metrics 
\begin{itemize}
    \item \textbf{minimum distance to each gate center}, i.e., $\min_{k \in \myset{N}} \|\vecbf{p}_{k} - \vecbf{c}_{\mathrm{g}, i} \|, i = 1,2,$
    \item \textbf{final distance to the goal}, i.e., $\|\vecbf{p}_{N} - \vecbf{p}_{\mathrm{d}} \|$,
\end{itemize}
to quantitatively evaluate the generalization performance.

\begin{figure*}[ht!]
    \centering
    % First row of figures
    \begin{subfigure}[b]{0.32\textwidth}
        \includegraphics[width=\textwidth]{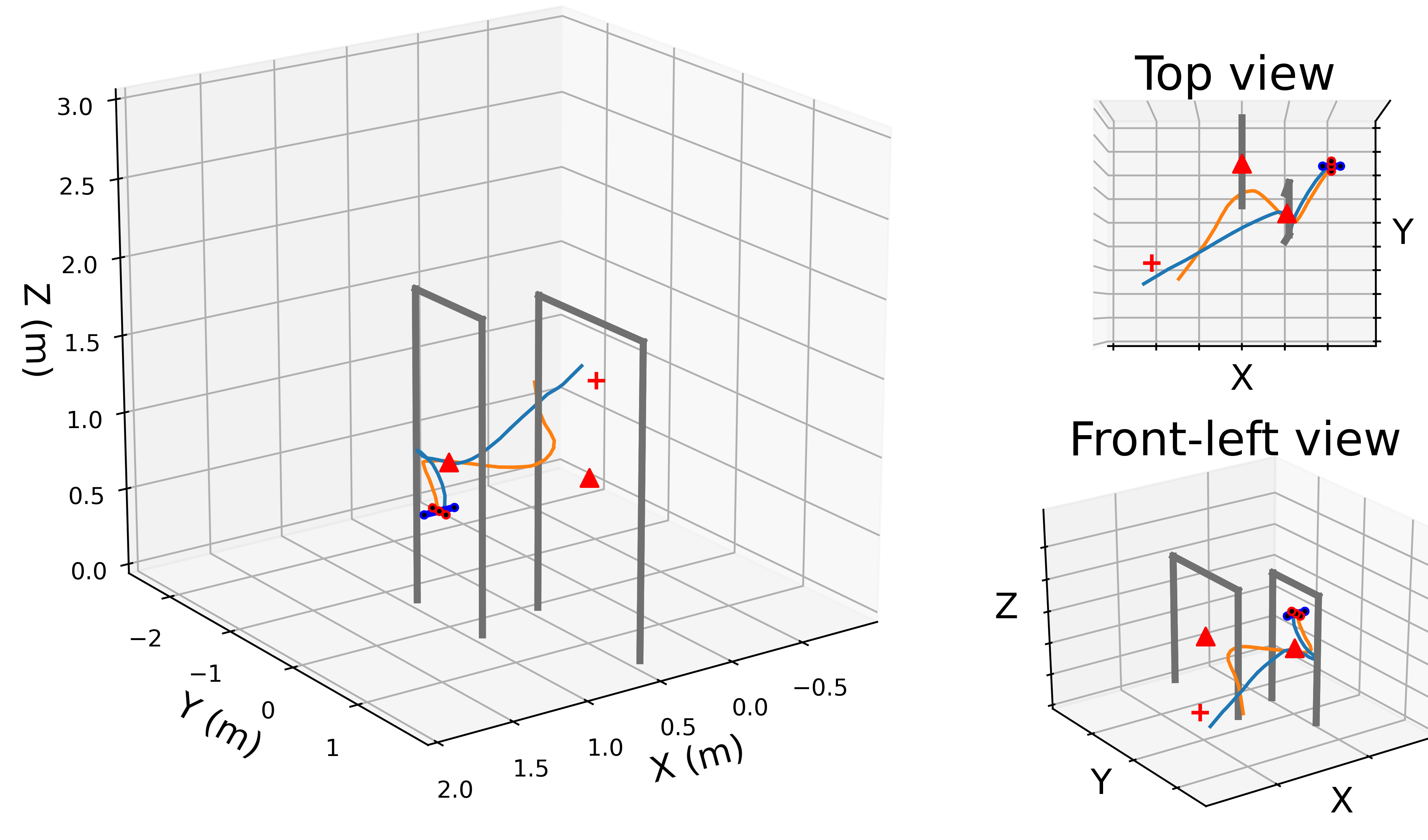}
        \caption{Trajectory planning of the learned cost function test under the training setting. \\}
        \label{fig:uav_gen_none}
    \end{subfigure}
    \hfill % adds horizontal space between figures
    \begin{subfigure}[b]{0.32\textwidth}
        \includegraphics[width=\textwidth]{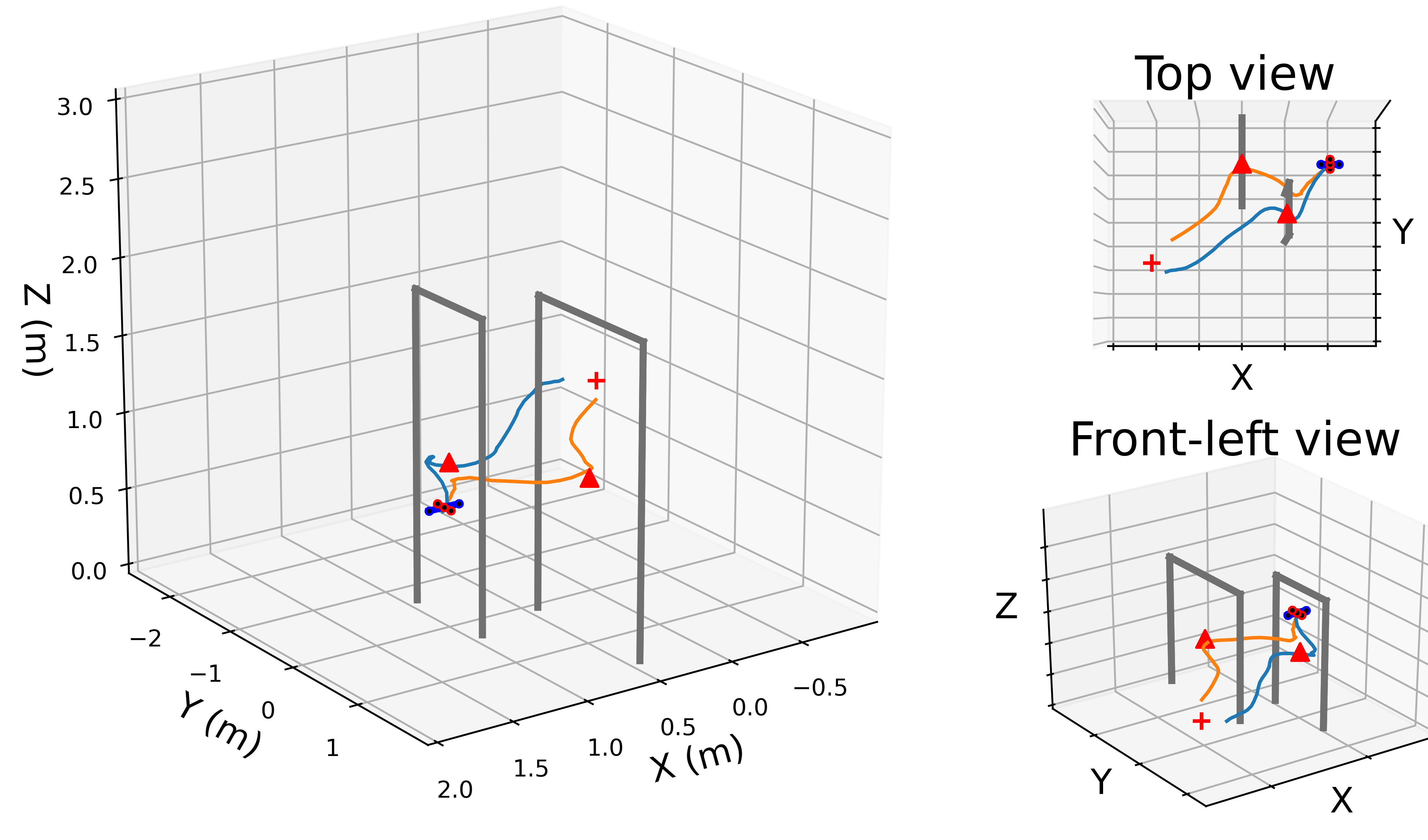}
        \caption{Generalization of the learned cost function on trajectory planning with a longer horizon.}
      \label{fig:uav_gen_hori}
   \end{subfigure}
   \hfill
    \begin{subfigure}[b]{0.32\textwidth}
        \includegraphics[width=\textwidth]{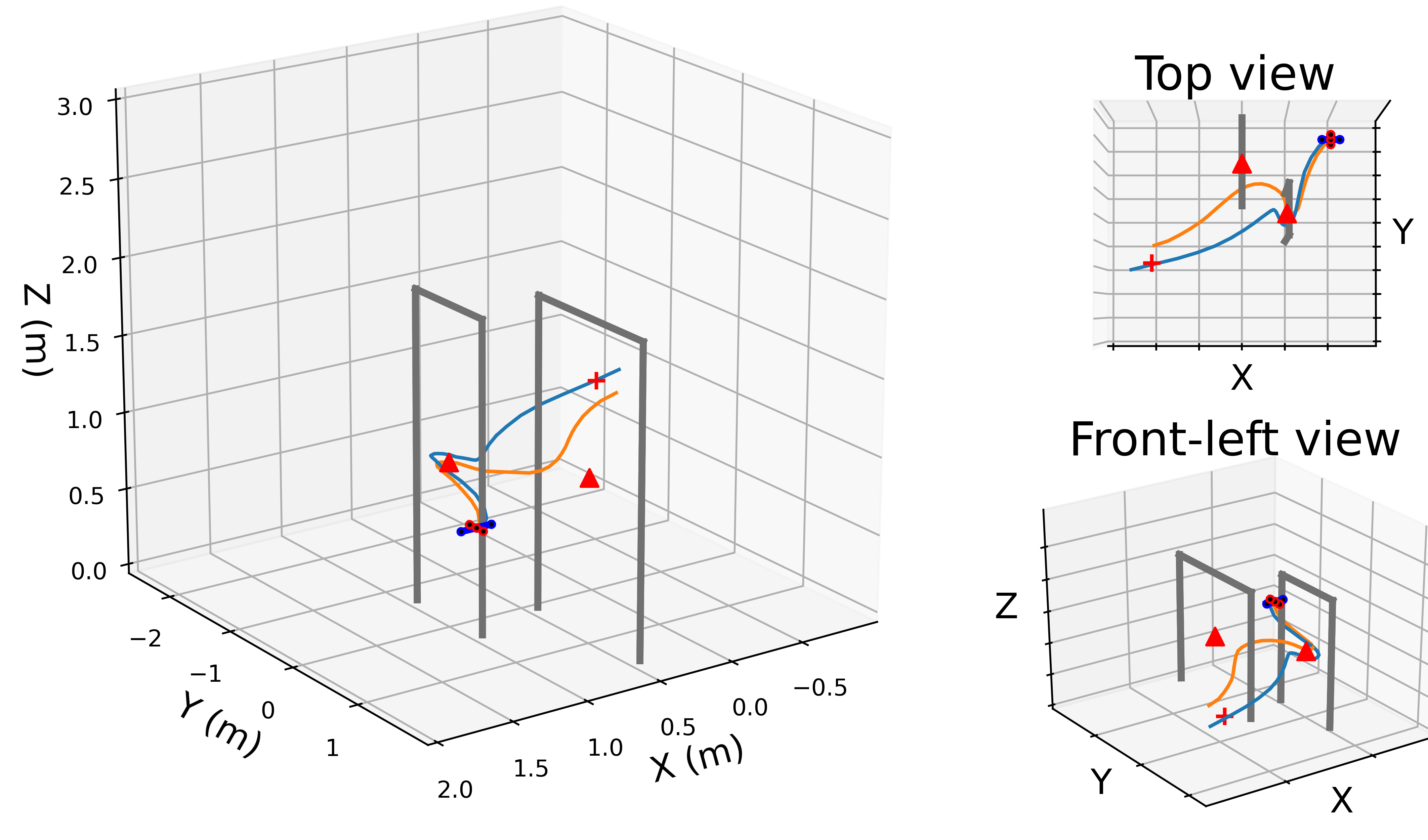}
        \caption{Generalization of the learned cost function on trajectory planning with new initial condition i).}
        \label{fig:uav_gen_ini1}
    \end{subfigure}
    
    \begin{subfigure}[b]{0.32\textwidth}
        \includegraphics[width=\textwidth]{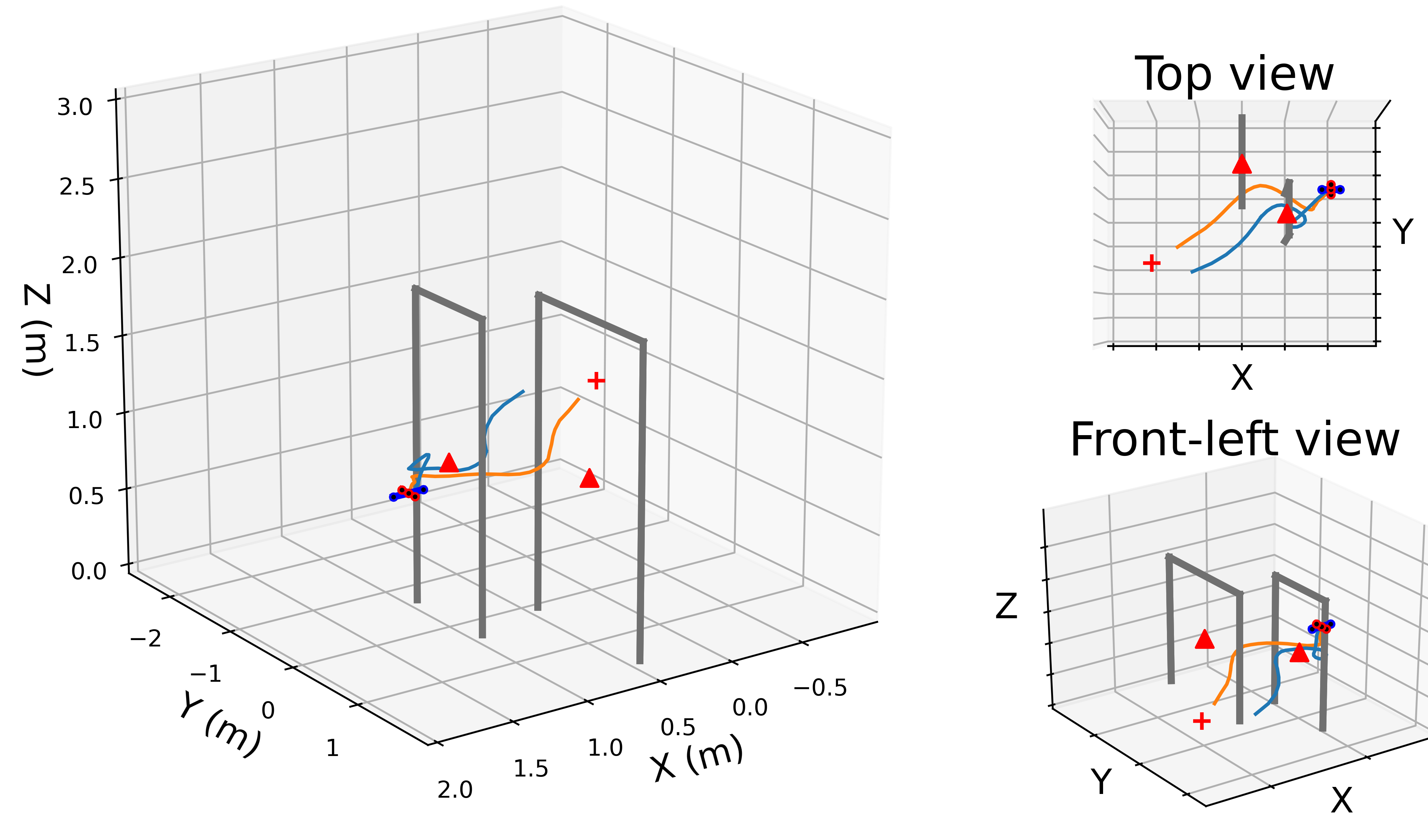}
        \caption{Generalization of the learned cost function on trajectory planning with new initial condition ii).}
        \label{fig:uav_gen_ini2}
    \end{subfigure}
    \hfill
    \begin{subfigure}[b]{0.32\textwidth}
        \includegraphics[width=\textwidth]{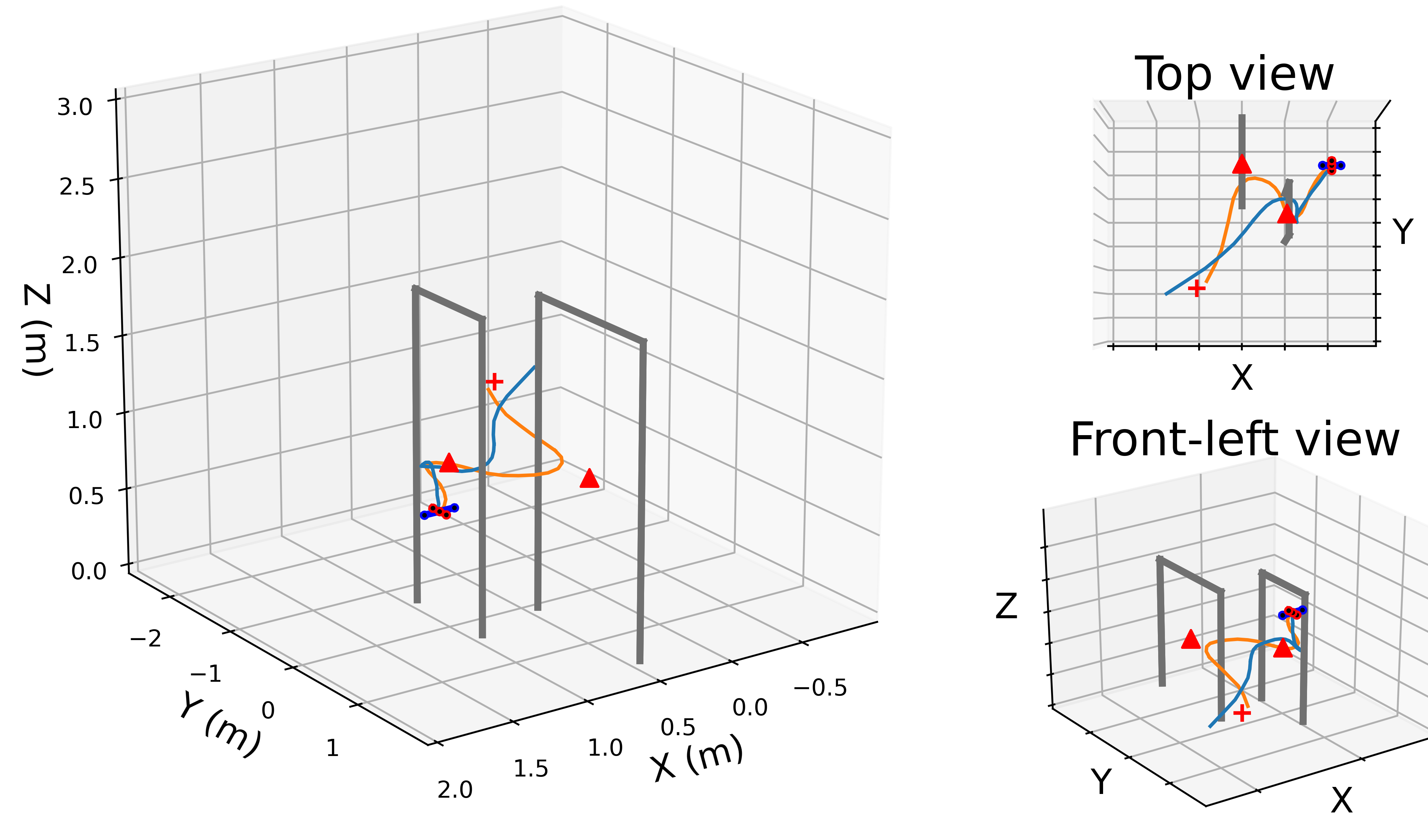}
        \caption{Generalization of the learned cost function on trajectory planning with a new desired position.}
        \label{fig:uav_gen_end}
    \end{subfigure}
    \hfill
    \begin{subfigure}[b]{0.32\textwidth}
        \includegraphics[width=\textwidth]{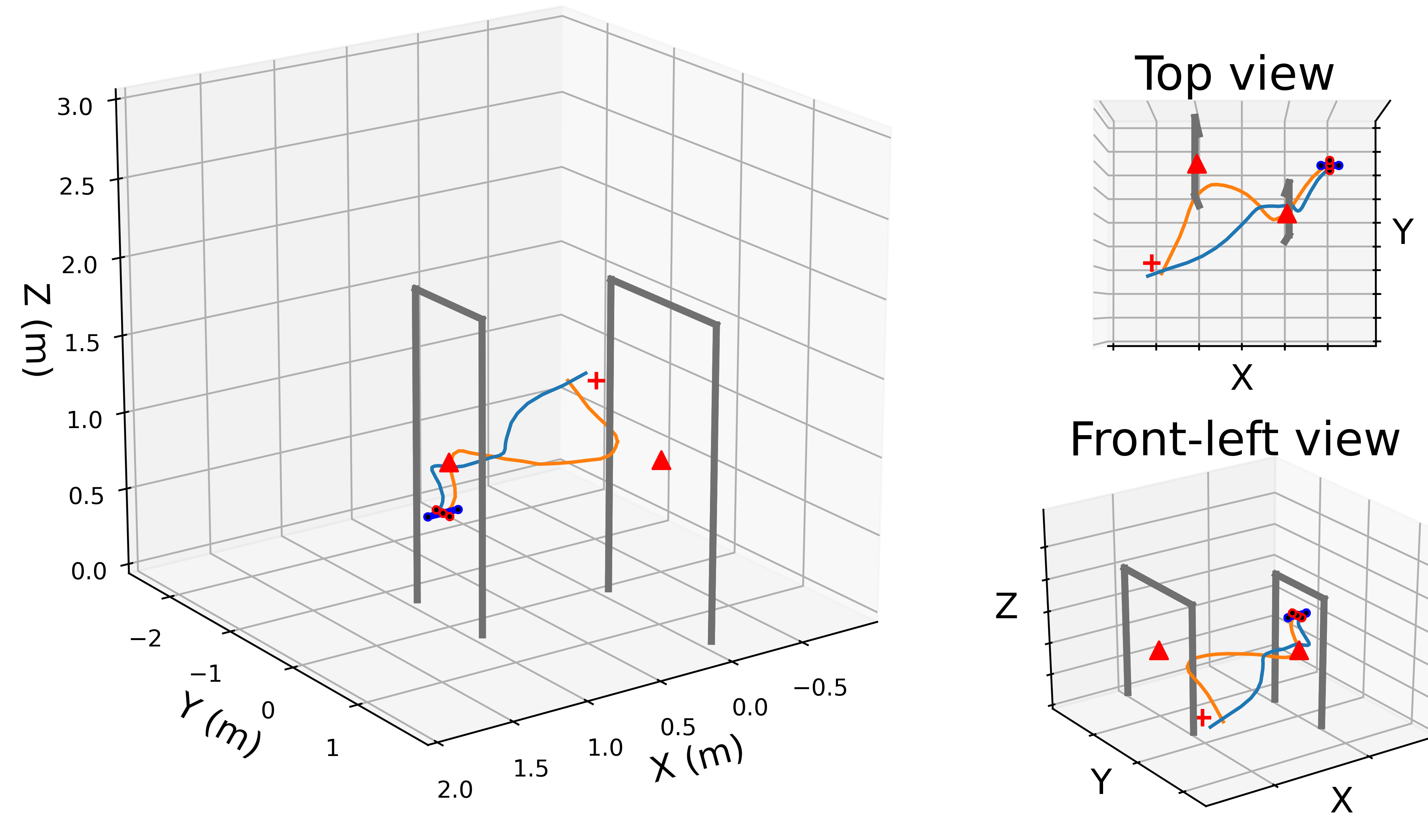}
        \caption{Generalization of the learned cost function on trajectory planning with a new gate position.}
        \label{fig:uav_gen_gate}
    \end{subfigure}
    \caption{Trajectory planning of the learned cost function (a) test under the training setting and its generalization to new settings, i.e., (b) longer horizon, (c)(d) new initial positions, (e) new desired position, (f) new gate position. 
    \textit{All the experiments are performed in the area shown in Fig. \ref{fig:uav_setup} and trajectories are recorded by the motion capture system.}
    OL and CL trajectories are denoted by blue and orange solid lines, respectively.
    The initial state of the quadrotor is denoted by a red-blue icon. 
    Red plus sign is the desired position. 
    We use gray bars and red triangles to denote the gate and its center, respectively. 
    All the quantitative measures are presented in Table \ref{tab:uav_measure}.
    }
    \label{fig:test_and_gen}
\end{figure*}

\subsubsection{Results and analysis}

We run both open-loop and closed-loop algorithms with the above-collected demonstration. 
The final learned parameters given by these algorithms are $\para^{\mathrm{ol}} = [0.74, 0.26]^{\top}$ and $\para^{\mathrm{cl}} = [0.45, 0.55]^{\top}$, respectively.
In the sequel, we shall refer to the trajectories generated by parameters learned from open-loop and closed-loop IRL as OL and CL trajectories, respectively.
By checking the value of these parameters, one can expect that OL trajectory will put more weight on gate $1$ and less weight on gate $2$ than CL trajectory.

We first test the performance in the test setting.  
A set of trajectories (one OL trajectory and one CL trajectory) is recorded in Fig. \ref{fig:uav_gen_none}. 
We further test the generalization of the learned parameter, or equivalently, cost function in the generalization settings. 
The generalization of the learned cost function to a longer planning horizon of $N=40$ is shown in Fig. \ref{fig:uav_gen_hori}.
 Figures \ref{fig:uav_gen_ini1} and \ref{fig:uav_gen_ini2} show the generalization of the learned cost function to new initial positions, which can be easily seen from the top view.
Figure \ref{fig:uav_gen_end} and \ref{fig:uav_gen_gate} present the generalization of the learned cost function to a new desired position and new gate positions, as seen from the top view.
\textit{Note that all the experiments are performed in the area shown in Fig. \ref{fig:uav_setup} and trajectories are recorded by the motion capture system and visualized in Fig. \ref{fig:test_and_gen}.}
Furthermore, for each case, we have repeated $5$ times and computed quantitative measures for the recorded trajectories, as shown in Table \ref{tab:uav_measure}.
Based on these results, we have the following comments.
\begin{table*}[h]\centering
\caption{Measure of the trajectory planning test and its generalization results. All the values are averaged for $5$ trials. }

\begin{tabular}{c|cc|cc|cc}
\hline 
\multirow{2}*{Scenario}  & \multicolumn{2}{c|}{Minimum distance to gate $1$ center} & \multicolumn{2}{c|}{Minimum distance to gate $2$ center}  & \multicolumn{2}{c}{Final distance to the goal} \\ \cline{2-7}
                          & \multicolumn{1}{p{2.5cm}<{\centering}}{CL}        & \multicolumn{1}{p{2.5cm}<{\centering}|}{OL}  & \multicolumn{1}{p{2.5cm}<{\centering}}{CL}        & \multicolumn{1}{p{2.5cm}<{\centering}|}{OL}& \multicolumn{1}{p{1.5cm}<{\centering}}{CL}        & \multicolumn{1}{p{1.5cm}<{\centering}}{OL}           \\
\hline \hline

Fig. \ref{fig:uav_gen_none}  & 0.20 & 0.10 & 0.58 & 0.87  & 0.41 & 0.46\\
Fig. \ref{fig:uav_gen_hori}  & 0.09 & 0.05 & 0.28 & 0.89 & 0.40 & 0.34 \\
Fig. \ref{fig:uav_gen_ini1} & 0.06 & 0.08 & 0.65 & 0.78  & 0.39 & 0.56\\
Fig. \ref{fig:uav_gen_ini2} & 0.10 & 0.09 & 0.60 & 0.94 & 0.43 & 0.55 \\
Fig. \ref{fig:uav_gen_end}  & 0.10 & 0.07 & 0.48 & 0.86  & 0.28 & 0.43 \\
Fig. \ref{fig:uav_gen_gate} & 0.08 & 0.07 & 0.63 & 1.14 & 0.34 & 0.29 \\
\hline
\end{tabular}
\label{tab:uav_measure}
\end{table*}
From the test of learned weights, as seen from Fig. \ref{fig:uav_gen_none} and Table \ref{tab:uav_measure},  CL trajectory takes a larger detour on flying through gate $2$ than OL trajectory. 
The average final distance to the goal is slightly smaller.
Under the longer horizon, new initial positions, new desired position, and new gate position, Fig. \ref{fig:uav_gen_hori}-\ref{fig:uav_gen_gate} show that the generalized CL trajectories still fly through two gates sequentially, and arrive at the vicinity of the desired position. 
However, generalized OL trajectories fail to fly through gate $2$. 
Specifically, as seen in Fig. \ref{fig:uav_gen_hori} and Table \ref{tab:uav_measure}, with a longer planning horizon, both OL and CL trajectories will be closer to gate $1$ center. 
Then, both of them take a larger detour towards the center of gate $2$. 
In this case, the apex of the CL trajectory gets closer to the center of gate $2$ and results in a large drop in terms of minimum distance to gate $2$ center.
However, this is not the case for OL trajectory, since the increase of horizon only reshapes its segment near gate $1$.
Nevertheless, this detour changes the velocity profile of OL trajectory and results in a smaller terminal velocity and overshoot.

Note that for the generalization of gate position in Fig. \ref{fig:uav_gen_gate}, the movement of gate $2$ enlarges the width of the curve (see the top view) due to the attraction force from its center.
However, this (discrete) change of environment is too significant to be followed by a continuous adaptation of the trajectory, which results in an increase of minimum distance to gate $2$ center, especially for OL trajectory as it does not reshape in $Y$-direction but the gate moves further away in $X$-direction.
We also report a failure case where we further move gate $2$ away from the initial position, i.e., $\vecbf{c}_{\mathrm{g},2} = [-0.2,1,1]^{\top}$, as visualized in Fig. \ref{fig:uav_gen_gate_fail}. 
It can be seen that the CL trajectory fails to reshape itself to fly through gate $2$. 
Moreover, the change of gate position completely alters the landscape of the cost function, enlarging the minimum distance to gate $1$ and failing to arrive at the vicinity of the desired position.
This result clearly shows the bound of the generalizability, i.e., the learned cost function can only be applied to some unseen scenarios that are close to the training setting.

\begin{figure}[!t]
	\centering
		\includegraphics[width=0.45\textwidth]{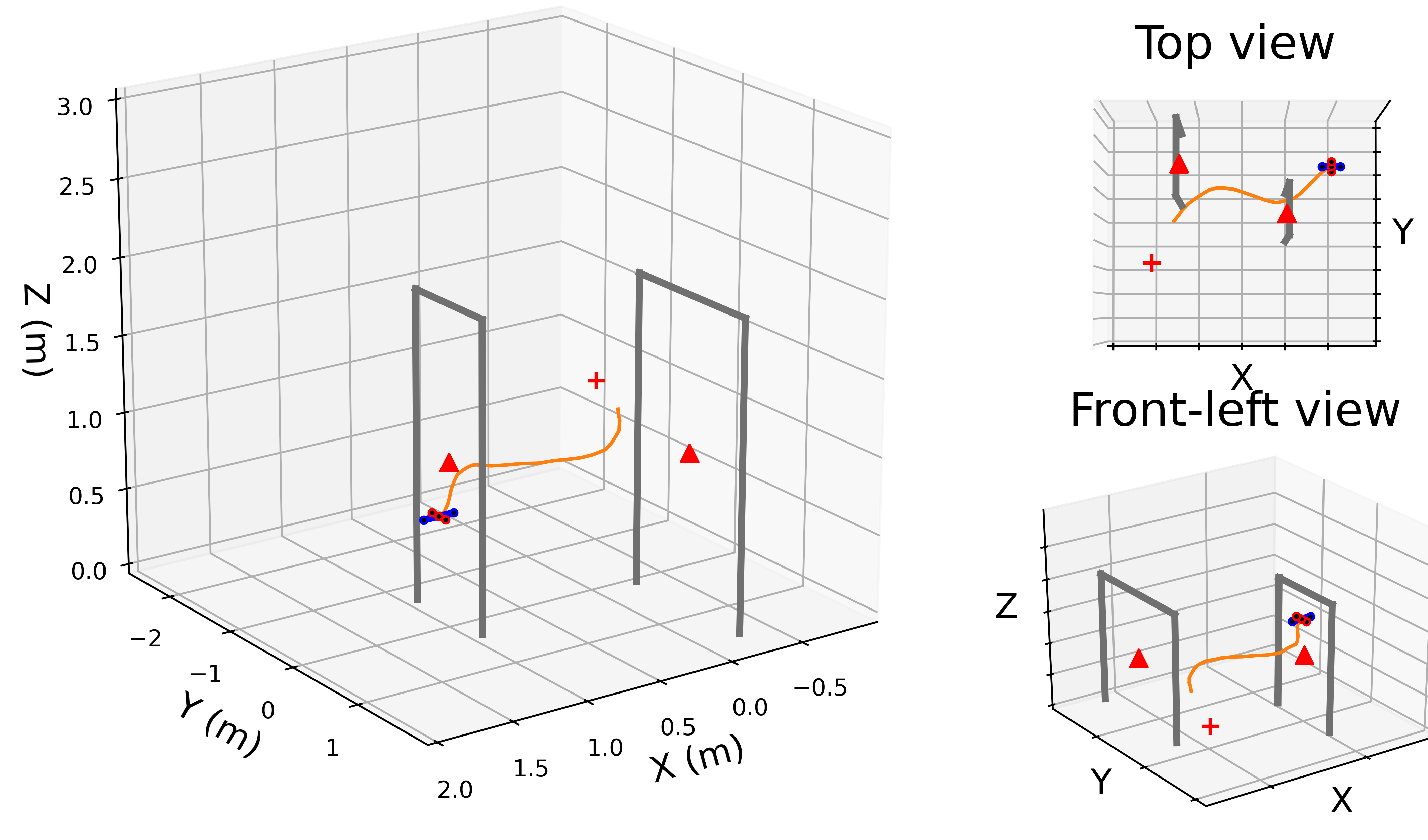}
	\caption{Failure case of generalization of the CL learned cost function on trajectory planning with a new gate position. The legend is the same as that in Fig. \ref{fig:test_and_gen}.}
	\label{fig:uav_gen_gate_fail}
\end{figure}

\section{Conclusion}
\label{sec:conclu}
In this work, we have proposed a DDP-based framework for IRL with general constraints, where the DDP was exploited to compute the gradient required in the outer loop.
We have established the equivalence between DDP-based and PDP-based methods in terms of computation.
In addition, inspired by the DDP condition, we have proposed the closed-loop IRL with the closed-loop loss function to capture the nature of collected demonstrations. 
Moreover, we have shown that this new formulation can be reduced to a general constrained IOC problem under certain conditions, which leads to a generalized recoverability condition. 
Simulations and experiments demonstrated the superiority of the closed-loop algorithm.
Future work can be on the extension of this framework to the multi-agent systems and stochastic systems.

\bibliographystyle{IEEEtran}
\bibliography{ref.bib}

\end{document}